\documentclass{article}

\usepackage{microtype}
\usepackage{graphicx}
 \usepackage{multirow} 
\usepackage{subcaption}
\usepackage{booktabs} 

\usepackage{hyperref}

\usepackage[accepted]{icml2026}

\usepackage{amsmath}
\usepackage{amssymb,enumitem}
\usepackage{mathtools}
\usepackage{amsthm}
\usepackage[ruled,vlined]{algorithm2e}
\usepackage{tikz}
\usetikzlibrary{calc,positioning,fit,matrix}

\newlength{\pairgap}   
\newlength{\groupgap}  

\newlength{\colwAll}
\newlength{\colwGrp}
\usepackage{tikz}
\usetikzlibrary{calc} 

\definecolor{plum}{rgb}{0.56, 0.27, 0.52}
\definecolor{paradise}{RGB}{242, 129, 0}
\definecolor{navyblue}{RGB}{0, 84, 181}
\definecolor{shiraz}{RGB}{186, 25, 63}

\usepackage[capitalize,noabbrev]{cleveref}

\theoremstyle{plain}
\newtheorem{theorem}{Theorem}[section]

\newtheorem{corollary}[theorem]{Corollary}
\theoremstyle{definition}

\theoremstyle{remark}
\newtheorem{remark}[theorem]{Remark}

\usepackage[textsize=tiny]{todonotes}

\icmltitlerunning{MMD Guidance: Training-Free Distribution Adaptation for Diffusion Models via Maximum Mean Discrepancy Guidance}

\begin{document}

\twocolumn[
  \icmltitle{MMD Guidance: Training-Free Distribution Adaptation for Diffusion Models via Maximum Mean Discrepancy Guidance}



  \icmlsetsymbol{equal}{*}

  \begin{icmlauthorlist}
    \icmlauthor{Matina Mahdizadeh Sani}{equal,waterloo}
    \icmlauthor{Nima Jamali}{equal,waterloo}
    \icmlauthor{Mohammad Jalali}{equal,cuhk}
    \icmlauthor{Farzan Farnia}{cuhk}
  \end{icmlauthorlist}

  \icmlaffiliation{waterloo}{School of Computer Science, University of Waterloo}
  \icmlaffiliation{cuhk}{Department of Computer Science and Engineering, The Chinese University of Hong Kong}

  \icmlcorrespondingauthor{Matina Mahdizadeh Sani}{m3mahdizadehsani@uwaterloo.ca}
  \icmlcorrespondingauthor{Nima Jamali}{nima.jamali@uwaterloo.ca}
  \icmlcorrespondingauthor{Mohammad Jalali}{mjalali24@cse.cuhk.edu.hk}
  \icmlcorrespondingauthor{Farzan Farnia}{farnia@cse.cuhk.edu.hk}

  \icmlkeywords{Machine Learning, ICML}

  \vskip 0.3in
]



\printAffiliationsAndNotice{\icmlEqualContribution}

\begin{abstract}
Pre-trained diffusion models have emerged as powerful generative priors for both unconditional and conditional sample generation, yet their outputs often deviate from the characteristics of user-specific target data. Such mismatches are especially problematic in domain adaptation tasks, where only a few reference examples are available and retraining the diffusion model is infeasible. Existing inference-time guidance methods can adjust sampling trajectories, but they typically optimize surrogate objectives such as classifier likelihoods rather than directly aligning with the target distribution. We propose \emph{MMD Guidance}, a training-free mechanism that augments the reverse diffusion process with gradients of the \textit{Maximum Mean Discrepancy (MMD)} between generated samples and a reference dataset. MMD provides reliable distributional estimates from limited data, exhibits low variance in practice, and is efficiently differentiable, which makes it particularly well-suited for the guidance task. Our framework naturally extends to prompt-aware adaptation in conditional generation models via product kernels. Also, it can be applied with computational efficiency in latent diffusion models (LDMs), since guidance is applied in the latent space of the LDM. Experiments on synthetic and real-world benchmarks demonstrate that MMD Guidance can achieve distributional alignment while preserving sample fidelity. The project code is available at \url{github.com/matinamehdizadeh/MMD-Guidance}.
\end{abstract}

\section{Introduction}
The rapid advancement of generative artificial intelligence is largely driven by the paradigm of foundation models: pre-trained neural networks capable of generalization across diverse tasks and modalities \citep{bommasani2021opportunities}. These models, including large language models (LLMs) \citep{brown2020language} and denoising diffusion models \citep{ho2020denoising,song2021denoising,rombach2022high}, offer powerful priors that can be adapted to specific downstream applications with minimal computational overhead. While this adaptation can occur through fine-tuning or prompt engineering, training-free inference-time methods remain relatively underexplored, despite their practical advantages of no memory overhead and immediate deployment.

A fundamental challenge in deploying generative models is \textit{distribution matching}, i.e., aligning model outputs with a user's target distribution that differs from the pre-training corpus. For example, consider a medical imaging scenario where a practitioner needs synthetic X-rays matching their hospital's specific equipment characteristics, or a designer requiring images that follow their brand's specific visual style, where both tasks are specified through a small number (e.g. 50-100) of reference examples. A generic text-conditioned image generation model will generate samples following its learned priors, leading to a distribution mismatch. Existing methods cannot fully address the distribution mismatch without computationally expensive retraining on a sufficiently large sample set from the target distribution, which is often unavailable to  users of generative AI services. 

\begin{figure*}[t]
    \centering
    \includegraphics[width=0.98\linewidth]{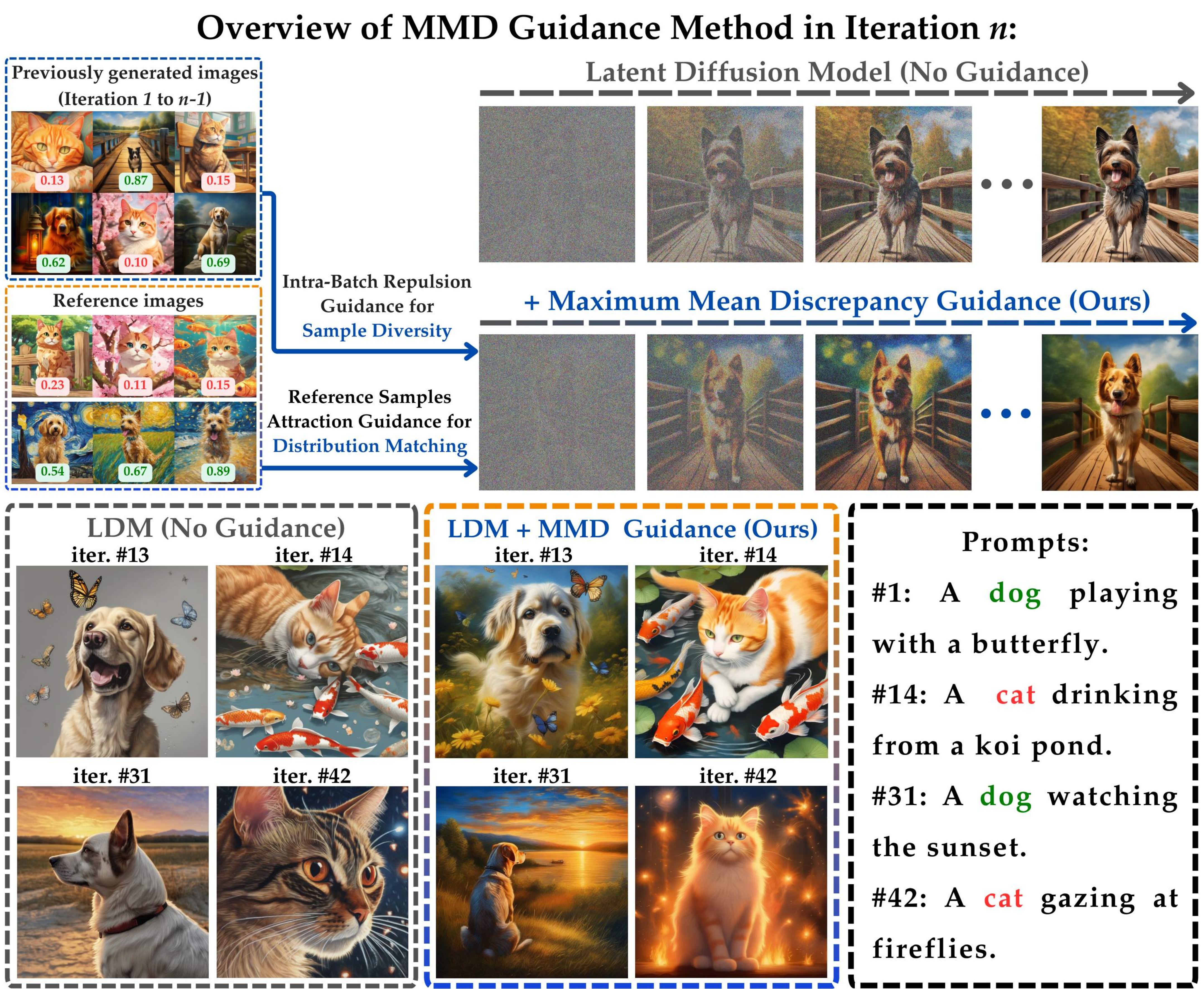}
    \caption{Image generation via a text-conditioned latent diffusion model (LDM) with no guidance vs. our proposed MMD guidance. The LDM (Stable Diffusion-XL) following our proposed MMD guidance over 100 reference samples of "cat" and "dog" images could exhibit the visual format of the target distribution, but the unguided LDM's output samples differ in style from the target model.}
    \label{fig:figure_one}
  
\end{figure*}

The mentioned gap between model capabilities and deployment requirements motivates the following question: Can we guide pre-trained diffusion models to match a user's target distribution at inference time using limited reference samples from their distribution? Note that this distribution matching task can be considered in both the prompt-free (unconditional) and prompt-aware (conditional) data generation, where the distribution alignment targets the user's distribution of interest. 

In particular, diffusion models are suited for such training-free adaptation due to their iterative denoising process. During the sampling phase of diffusion models, external signals can steer the generation trajectory through guidance in sample generation without modifying the parameters of the pre-trained neural net denoiser. Existing guidance techniques successfully bias samples toward specific objectives: classifier guidance \citep{dhariwal2021diffusion} maximizes class probabilities, classifier-free guidance \citep{ho2022classifierfree} amplifies conditioning, and score distillation \citep{poole2022dreamfusion} optimizes external losses. However, none of these methods can explicitly match a target distribution defined by reference samples, since they optimize surrogate objectives rather than directly minimizing distributional discrepancy. The key challenge is to quantify and minimize distribution mismatch in reverse diffusion dynamics.

In this work, we propose the \textit{MMD Guidance} approach, an inference-time mechanism that augments the reverse diffusion process with gradients from the Maximum Mean Discrepancy (MMD) \citep{gretton2012kernel}. This approach aims to address distribution matching by incorporating $\nabla_{x_t} \text{MMD}^2$ into the sampling dynamics at each timestep. MMD is uniquely suited for this task. Unlike the family of $f$-divergences and Wasserstein distances, the MMD distance provides unbiased, low-variance estimates from limited samples without suffering from the curse of dimensionality. Moreover, the kernel-based formulation in MMD is efficiently differentiable, enabling scalable gradient computation with respect to generated samples.

Our proposed MMD guidance method can smoothly extend to \textit{prompt-aware distribution matching} in text-conditioned image generation. Users often need outputs that satisfy both textual semantics ("a portrait of a person") and distributional constraints (matching the style of their reference portraits). We attempt to achieve this property through the product kernel function $k_\text{multimodal}\bigl([p,x], [p',x']\bigr) = k_\text{text}\bigl(p,p'\bigr)\cdot k_\text{image}\bigl(x,x'\bigr) $ that jointly measure similarity in prompt and visual spaces, enabling prompt-aware control without any retraining of the denoiser in conditional diffusion model sample generation. 

For practical deployment, we implement MMD Guidance efficiently in the \emph{latent space} of a \textit{latent diffusion model (LDM)} \citep{rombach2022high}. By computing MMD and its gradients directly on the latent representation $z \in \mathcal{Z}$ rather than pixel space, we achieve a significant speedup while maintaining alignment quality. This enhances the efficiency and scalability of the MMD guidance method for state-of-the-art conditional diffusion generation models.

We validate MMD Guidance in our numerical experiments on synthetic distributions, standard image databases, and domain adaptation benchmarks. On synthetic Gaussian mixture models, the MMD guidance can recover the target mixture from a limited number of samples. In the case of image generation, the MMD Guidance framework reduces distribution discrepancy compared to the baseline training-free guidance methods while maintaining comparable fidelity scores in image generation. Also, the MMD guidance method could preserve sample quality through inference-time guidance in the latent space with proper generalization. Figure~\ref{fig:figure_one} illustrates an example of applying the prompt-aware MMD guidance with only 100 reference samples of "cat" and "dog" categories to the latent space of the SD-XL \citep{podell2023sdxl} latent diffusion model. In summary, our work's main contributions are:\vspace{-2mm}
\begin{itemize}[leftmargin=*,itemsep=0pt]
\item We introduce MMD Guidance, a training-free method for inference-time adaptation of diffusion models to a target distribution specified by a limited reference set.
\item We propose a practical extension for prompt-aware adaptation using product kernels, enabling joint alignment to both text prompts and visual styles.
\item We develop an efficient and scalable implementation that operates directly in the latent space of LDMs.
\item We provide empirical validation demonstrating that MMD guidance can achieve satisfactory distribution matching.
\end{itemize}

\vspace{-3mm}
\section{Related Work}\vspace{-1mm}
\noindent \textbf{Guidance in diffusion models.}
Inference-time guidance has been central to the success of diffusion models. Classifier guidance augments reverse diffusion with classifier gradients, boosting conditional fidelity at the expense of diversity \citep{dhariwal2021diffusion}. Classifier-free guidance (CFG) removes the external classifier by interpolating conditional and unconditional scores \citep{ho2022classifierfree}. This line of work has been extended in several ways. \cite{malarz2025classifier} introduce $\beta$-CFG, which controls the impact of guidance during the denoising process. Similarly, CFG++ \citep{chung2025cfg} aims to mitigate the off-manifold behavior associated with the standard CFG.

Follow-ups extend guidance to multiple conditions via product-of-experts composition \citep{liu2022composablediffusion} and to editing through cross-attention control or partial noising \citep{hertz2023prompttoprompt,meng2022sdedit}. These approaches bias sampling toward predefined conditions, but do not explicitly align generations to an arbitrary reference distribution.

\textbf{Conditional Generation with Guidance.} Controlling generative processes with specific conditions is increasingly vital for practical applications, utilizing inputs such as text guidance~\citep{kim2022diffusionclip, nichol2022glide, liu2023more}, class labels~\citep{dhariwal2021diffusion}, style images~\citep{mou2024t2iadapter, zhang2023adding} and human motions~\citep{tevet2023human}.
Conditional generation methods are divided into training-based and training-free approaches. Training-based methods either learn a time-dependent classifier that guides the noisy sample $\boldsymbol{x}t$ toward the condition $\boldsymbol{y}$~\citep{dhariwal2021diffusion, nichol2022glide, zhao2022egsde, liu2023more} or directly train a conditional denoising model $\boldsymbol{\epsilon}_\theta(\boldsymbol{x}_t,t,\boldsymbol{y})$ through few-shot adaptation~\citep{mou2024t2iadapter, rombach2022high, ruiz2023dreambooth}.
In contrast, training-free guidance enables zero-shot conditional generation using a pre-trained differentiable predictor, such as a classifier or energy function, to assess the alignment of generated samples with target conditions~\citep{he2024manifold, bansal2023universal, yu2023freedom, ye2024tfg}. The diversity guidance frameworks in \citep{jalali2025sparke,jalali2024conditional,zhang2026exploring} leverage the RKE \citep{jalali2023information,ospanov2025towards} and Vendi/Conditional-Vendi scores \citep{friedman2023vendiscorediversityevaluation,ospanov2024statvendi} to improve diversity in the data generated by diffusion models. Our work is a training-free approach that utilizes the MMD distance to improve alignment while maintaining diversity.

\noindent \textbf{Adaptation of Diffusion Models.}
A complementary line adapts pretrained text-to-image models to user-specific concepts and styles. Textual Inversion learns novel token embeddings for new subjects \citep{gal2022textualinversion}; DreamBooth fine-tunes for subject fidelity with a handful of images \citep{ruiz2023dreambooth}; and LoRA enables parameter-efficient adaptation via low-rank updates \citep{hu2022lora}. Structural controllers such as ControlNet and T2I-Adapter add spatial/structural conditions with lightweight modules \citep{zhang2023controlnet,mou2024t2iadapter}. 
In a recent work, Domain Guidance \citep{zhong2025domain} uses the pretrained diffusion models and a fine-tuned model to guide generation toward the target domain.
While effective, these methods require optimization and parameter storage, and can overfit tiny reference sets. Our approach is training-free and operates at inference time, making it complementary in deployment and privacy-constrained settings.

\noindent \textbf{MMD and distribution alignment.}
The maximum mean discrepancy (MMD) is a kernel-based integral probability metric with strong finite-sample properties, widely used in two-sample testing \citep{gretton2012kernel}, domain adaptation \citep{long2015dan}, and generative modeling \citep{li2015gmmn,li2017mmdgan,binkowski2018kid,wang2023distributedKID,wu2026maximum_isit}. Also, the MMD and kernel-based scores have been used for online evaluation and selection of generative models \citep{hu2025multiarmedbandit,rezaei2024be,hu2025online,hu2025promptwise,jafari2026dakucb,nia2026mixturegreedy}, and have been used for embedding alignment \citep{jalali2025spec,gong2025kernel,Ospanov_2025_ICCV,wu2026koda} and generative model comparison \citep{zhang2024interpretable,zhang2025finc,lotfian2026promptsplit}. 
In concurrent work, \cite{galashov2025} replaces the static guidance weight $w$ with a timestep-dependent learnable function. They train a neural network to predict this function using MMD as the optimization objective. In contrast, our training-free method applies MMD directly within the denoising process by injecting $\nabla_{x_t}\mathrm{MMD}^2$ into the diffusion sampler, guiding samples toward the reference distribution, whereas prior work employs MMD as a training loss or evaluation metric.

\section{Preliminaries}
\subsection{Kernel Functions and Maximum Mean Discrepancy}

Let $k: \mathcal{X} \times \mathcal{X} \to \mathbb{R}$ be a positive semi-definite kernel with associated reproducing kernel Hilbert space (RKHS) $\mathcal{H}_k$ and feature map $\varphi: \mathcal{X} \to \mathcal{H}_k$ satisfying $k(x, x') = \langle \varphi(x), \varphi(x') \rangle_{\mathcal{H}_k}$. The \textit{kernel mean embedding} of a distribution $P$ over $\mathcal{X}$ is defined as:
\begin{equation}
\mu_P := \mathbb{E}_{x \sim P}[\varphi(x)] \in \mathcal{H}_k.
\end{equation}

The \textit{Maximum Mean Discrepancy (MMD)} between distributions $P$ and $Q$ measures the RKHS distance between their mean embeddings \citep{gretton2012kernel}:
\begin{align*}
\text{MMD}^2&(P, Q) = \bigl\|\mu_P - \mu_Q\bigr\|_{\mathcal{H}_k}^2 = \mathbb{E}_{x,x' \stackrel{\text{iid}}{\sim} P}\bigl[k(x,x')\bigr] \\
&+ \mathbb{E}_{y,y' \stackrel{\text{iid}}{\sim} Q}\bigl[k(y,y')\bigr] - 2\cdot\mathbb{E}_{x \sim P, y \sim Q}\bigl[k(x,y)\bigr].
\end{align*}

For characteristic kernels (e.g., Gaussian kernel), MMD defines a metric on probability measures: $\text{MMD}(P,Q) = 0$ if and only if $P = Q$, and it  satisfies the triangle inequality. 

\subsection{Latent Diffusion Models (LDMs)}

Latent Diffusion Models (LDMs) \citep{rombach2022high} perform diffusion in the latent space of a pre-trained variational autoencoder (VAE), achieving computational efficiency while maintaining generation quality. Given encoder $\mathcal{E}: \mathcal{X} \to \mathcal{Z}$ and decoder $\mathcal{D}: \mathcal{Z} \to \mathcal{X}$, LDMs operate on latent codes $z = \mathcal{E}(x)$ in the latent space $\mathcal{Z}$. The forward diffusion progressively adds Gaussian noise to latents:
\begin{align*}
q(z_t | z_0) &= \mathcal{N}(z_t; \sqrt{\bar{\alpha}_t} z_0, (1-\bar{\alpha}_t)I), \\ z_t &= \sqrt{\bar{\alpha}_t} z_0 + \sqrt{1-\bar{\alpha}_t} \epsilon_t,
\end{align*}
where $\epsilon_t \sim \mathcal{N}(0,I)$, $\bar{\alpha}_t = \prod_{s=1}^t \alpha_s$, and the parameters $\{\alpha_t\}_{t=1}^T$ follow a proper schedule. A denoising U-Net $\epsilon_\theta(z_t, t, c)$ is trained to predict the noise $\epsilon$ given noisy latent $z_t$, timestep $t$, and optional conditioning $c$ (e.g., text embeddings from CLIP). The training objective is $
\mathcal{L}(\theta) = \mathbb{E}_{z_0, \epsilon, t, c}\left[\|\epsilon - \epsilon_\theta(z_t, t, c)\|^2\right]$. In the \textit{sampling} phase, we proceed via iterative denoising starting from $z_T \sim \mathcal{N}(0,I)$:
\begin{equation}
z_{t-1} = {\frac{1}{\sqrt{\alpha_t}}\left(z_t - \frac{1-\alpha_t}{\sqrt{1-\bar{\alpha}_t}}\epsilon_\theta(z_t, t, c)\right)}  + \sigma_t \eta,
\label{eq:ddpm_sampling}
\end{equation}
where $\eta \sim \mathcal{N}(0,I)$, $\sigma_t^2 = \tilde{\beta}_t = \frac{1-\bar{\alpha}_{t-1}}{1-\bar{\alpha}_t}\beta_t$ for DDPM sampling \citep{ho2020denoising}. The \textit{guidance mechanisms} modify the predicted mean $\mu_\theta$ to steer generation. Notably, the classifier-free guidance \citep{ho2022classifierfree} interpolates conditional and unconditional predictions as follows, in which $w > 1$ amplifies conditioning strength,
\begin{equation}
\tilde{\epsilon}_\theta(z_t, t, c) = (1-w)\cdot\epsilon_\theta(z_t, t, \varnothing) + w \cdot \epsilon_\theta(z_t, t, c)
\end{equation}

\section{MMD Guidance for Diffusion Models}
\subsection{Distribution Matching via Divergence Guidance in Diffusion Models}

Consider the task of adapting a pre-trained diffusion model to generate samples matching a target distribution $Q$, specified only through a small set of reference samples $\mathcal{R} = \{x_j^{(r)}\}_{j=1}^{N_r}$. In our work, we aim to develop a guidance-based framework to address this challenge and perform a training-free adaptation of the pre-trained diffusion model. To guide the sampling process toward this target, we propose augmenting the reverse diffusion with gradients of a divergence measure between the distributions of generated and reference samples.

The choice of divergence measure is critical when estimating from limited data. Standard $f$-divergences such as the KL-divergence and total variation distance suffer from the curse of dimensionality in high-dimensional spaces, requiring sample complexity exponential in dimension for reliable estimation \citep{sriperumbudur2012empirical}. Similarly, Wasserstein distances, while providing a powerful metric in distinguishing probability distributions, exhibit high sample complexity, with minimax estimation rates of $O(n^{-1/d})$ for $d$-dimensional data \citep{weed2019sharp}.

We propose using Maximum Mean Discrepancy (MMD) as the divergence measure in the distribution matching guidance process. The choice of MMD distance offers two key advantages: (i) sample complexity independent of ambient dimension for characteristic kernels, and (ii) closed-form gradient computation enabling efficient optimization. More specifically, considering samples $\{z_t^{(i)}\}_{i=1}^B$ at timestep $t$ with empirical distribution $\widehat{P}_t = \frac{1}{B}\sum_{i=1}^B \delta_{z_t^{(i)}}$ and reference distribution $\widehat{Q} = \frac{1}{N_r}\sum_{j=1}^{N_r} \delta_{z_j^{(r)}}$, the empirical squared MMD is:
\begin{align*}
&\widehat{\text{MMD}}^2(\widehat{P}_t, \widehat{Q}) = \frac{1}{B^2}\sum_{i,i'=1}^B k(z_t^{(i)}, z_t^{(i')})\\ &+ \frac{1}{N_r^2}\sum_{j,j'=1}^{N_r} k(z_j^{(r)}, z_{j'}^{(r)}) - \frac{2}{BN_r}\sum_{i=1}^B\sum_{j=1}^{N_r} k(z_t^{(i)}, z_j^{(r)}).
\end{align*}

\subsection{MMD-Guided Sampling in Diffusion Models}

We incorporate MMD minimization into the diffusion sampling process by modifying the reverse trajectory with its gradients. The MMD-guided sampling update becomes:
\begin{equation}
z_{t-1}^{(i)} = \text{sampler}(z_t^{(i)}, t, \epsilon_\theta) - \lambda_t \nabla_{z_t^{(i)}} \widehat{\text{MMD}}^2(\widehat{P}_t, \widehat{Q}),
\label{eq:mmd_guided_sampling}
\end{equation}
where $\text{sampler}(\cdot)$ denotes any standard sampling scheme (e.g. DDPM or DDIM), and $\lambda_t$ controls guidance strength. Note that we subtract the MMD-squared gradient in the guidance process, as we seek to \textit{minimize} the MMD.

\begin{remark}\label{remark LDM} For computational efficiency, we perform MMD guidance directly in the \emph{latent space} of \textit{latent diffusion models (LDMs)}. Given reference images $\{x_j^{(r)}\}_{j=1}^{N_r}$, we encode them once as $z_j^{(r)} = \mathcal{E}(x_j^{(r)})$ using the common VAE encoder of latent diffusion models (LDMs) \citep{rombach2022high}. This latent-space guidance offers two advantages: 1) \textit{computational efficiency} of operating in a  lower dimension of the standard latent spaces of LDMs compared to the original space, 2) \textit{statistical efficiency} of the compressed latent representation capturing semantic structure while filtering pixel-level noise, improving the signal-to-noise ratio for MMD estimation. Algorithm~\ref{alg:mmd_unconditional} describes the MMD-guided diffusion process in the  LDM's latent space.
\end{remark}

For the gradient computation with respect to sample $z_t^{(i)}$, only terms in the MMD-squared containing $z_t^{(i)}$ contribute non-zero gradients. Hence, assuming a differentiable kernel $k$, this approach yields the following where $\nabla_{z_t^{(i)}}$ denotes the gradient with respect to $z_t^{(i)}$:\vspace{-3mm}
\begin{align}\label{eq:mmd_gradient_batch}
\nabla_{z_t^{(i)}} \widehat{\text{MMD}}^2(\widehat{P}_t, \widehat{Q}) = &\underbrace{\frac{2}{B^2}\sum_{j=1}^B \nabla_{z_t^{(i)}} k(z_t^{(i)}, z_t^{(j)})}_{\text{intra-batch term}}\\
&- \underbrace{\frac{2}{BN_r}\sum_{j=1}^{N_r} \nabla_{z_t^{(i)}} k(z_t^{(i)}, z_j^{(r)})}_{\text{cross term with references}}\nonumber
\vspace{-3mm}
\end{align}
 \vspace{-2mm}

\begin{algorithm}[t]
\caption{MMD-Guided Diffusion Sampling}
\label{alg:mmd_unconditional}
\KwIn{Reference data $\mathcal{R} = \{x_j^{(r)}\}_{j=1}^{N_r}$, batch size $B$, guidance schedule $\{\lambda_t\}_{t=1}^T$, denoiser $\epsilon_\theta$}
\KwOut{Generated samples $\{x^{(i)}\}_{i=1}^B$ matching target distribution}
\textit{Preprocessing:} $z_j^{(r)} \gets \mathcal{E}(x_j^{(r)})$ for all $j \in [N_r]$ \\
\textit{Initialization:} $z_T^{(i)} \sim \mathcal{N}(0, I)$ for all $i \in [B]$\;
\For{$t = T$ \KwTo $1$}{
    \For{$i = 1$ \KwTo $B$ \textit{in parallel}}{
        $\widehat{z}_{t-1}^{(i)} \gets \text{sampler}(z_t^{(i)}, t, \epsilon_\theta)$ \\
        $g^{(i)} \gets \nabla_{z_t^{(i)}} \widehat{\text{MMD}}^2(\widehat{P}_t, \widehat{Q})$ \tcp{Compute via \eqref{eq:mmd_gradient_batch}}
        $z_{t-1}^{(i)} \gets \widehat{z}_{t-1}^{(i)} - \lambda_t g^{(i)}$ \\
    }
}
\Return $\{x^{(i)} = \mathcal{D}(z_0^{(i)})\}_{i=1}^B$ \\
\end{algorithm}

Under the minimization update $z \leftarrow z - \lambda_t \nabla \widehat{\mathrm{MMD}}^2$, the intra-batch term creates \emph{repulsion} among generated samples (promoting diversity), while the cross term creates \emph{attraction} toward the reference samples (encouraging distribution matching). Note that we primarily need the cross term to provide a reliable gradient toward the target distribution, which we show in the following theorem.
\begin{theorem}
\label{thm:gradient_concentration}
Consider sample space $\mathcal{Z} \subseteq \mathbb{R}^d$. Let $k: \mathcal{Z} \times \mathcal{Z} \to \mathbb{R}$ be a normalized kernel with $k(z,z) = 1$ for all $z \in \mathcal{Z}$. Suppose $k$ is differentiable and $L$-Lipschitz w.r.t. either input, i.e., $|k(z, w) - k(z', w)| \leq L\big\|z - z'\big\|_2$ for all $z, z', w \in \mathcal{Z}$. Let $Q$ be the target distribution on $\mathcal{Z}$ and let $\{z_j^{(r)}\}_{j=1}^{N_r} \stackrel{\text{iid}}{\sim} Q$ be reference samples. For every $z_0 \in \mathcal{Z}$, we define the population cross term and the empirical cross term as follows:
\begin{align*}
g^*_{\text{cross}}(z_0) &= -2\,\mathbb{E}_{z' \sim Q}[\nabla_{z_0} k(z_0, z')], \\ 
\widehat{g}_{\text{cross}}(z_0) &= -\frac{2}{N_r}\sum_{j=1}^{N_r} \nabla_{z_0} k(z_0, z_j^{(r)}).
\end{align*}
Then for every $\delta>0$, with probability at least $1-\delta$ over the draw of reference samples we have:
\begin{equation*}
\bigl\|\widehat{g}_{\text{cross}}(z_0) - g^*_{\text{cross}}(z_0)\bigr\|_2 \leq \frac{4L}{\sqrt{N_r}}\Bigl(1 + \sqrt{2\log\bigl({1}/{\delta}\bigr)}\Bigr).
\end{equation*}
\end{theorem}


\begin{corollary}
\label{cor:gaussian_rbf_concise}
For the Gaussian RBF kernel $k(x,y) = \exp\bigl(-\big\|x-y\big\|_2^2/2\sigma^2\bigr)$, the cross term satisfies the following with probability at least $1-\delta$:
\begin{equation*}
\bigl\|\widehat{g}_{\text{cross}}(z_0) - g^*_{\text{cross}}(z_0)\bigr\|_2 \leq \frac{3}{\sigma\sqrt{N_r}}\Bigl(1 + \sqrt{2\log\bigl({1}/{\delta}\bigr)}\Bigr).
\end{equation*}
\end{corollary}

Theorem~\ref{thm:gradient_concentration} provides pointwise concentration guarantees for a fixed latent point $z_0\in\mathcal{Z}$. We extend this result to obtain uniform concentration over the norm ball in the latent space, guaranteeing that the MMD guidance gradient concentrates simultaneously for all latent vectors visited.

\begin{theorem}
\label{thm:uniform_concentration_ball}
Consider the setting of Theorem~\ref{thm:gradient_concentration} and let $\mathcal{Z}=\{z\in\mathbb{R}^d:\|z\|_2\le R\}$. Suppose $\|\nabla_z k(z,w)-\nabla_z k(z',w)\|_2\le L'\|z-z'\|_2$ for all $z,z',w\in\mathcal{Z}$. For every $\delta>0$, the following holds with probability at least $1-\delta$
\begin{align*}
&\sup_{z\in\mathcal{Z}}\;\bigl\|\widehat g_{\mathrm{cross}}(z)-g^*_{\mathrm{cross}}(z)\bigr\|_2 \\
\le\;
&\frac{4L'}{\sqrt{N_r}}
+ \frac{4L}{\sqrt{N_r}}\Bigl(1+\sqrt{\,2d\log\bigl(6R\sqrt{N_r}\bigr)+2\log\bigl({1}/{\delta}\bigr)\,}\Bigr).
\end{align*}
\end{theorem}

\vspace{-2.5mm}
\section{Prompt-Aware MMD Guidance in Conditional Diffusion Models}\vspace{-2.5mm}


Text-conditioned diffusion models can also be adapted to the distribution of a set of reference (prompt,data) pairs. Here, we extend MMD guidance to this setting by defining a joint divergence over the product space $\mathcal{P} \times \mathcal{Z}$ of prompts and latents. Our approach is to consider a product kernel that decomposes similarity into semantic and visual components:
\begin{equation}
k_{\otimes}\bigl([p,z], [p',z']\bigr) = k_p\bigl(p, p'\bigr) \cdot k_z\bigl(z, z'\bigr),
\label{eq:product_kernel}
\end{equation}
where $k_p: \mathcal{P} \times \mathcal{P} \to \mathbb{R}$ measures semantic similarity between prompt embeddings and $k_z: \mathcal{Z} \times \mathcal{Z} \to \mathbb{R}$ measures visual similarity between latents. As shown in \citep{bamberger2022johnson,wu2025fusing}, this product kernel corresponds to the tensor product feature map $\varphi_{\otimes}(p,z) = \varphi_p(p) \otimes \varphi_z(z)$ in the product RKHS $\mathcal{H}_p \otimes \mathcal{H}_z$. The induced MMD in this space captures distributional differences in both modalities simultaneously.

Given generated pairs $\{(p_i, z_t^{(i)})\}_{i=1}^B$ with distribution $\widehat{P}_t$ and reference pairs $\{(p_j^{(r)}, z_j^{(r)})\}_{j=1}^{N_r}$with distribution $\widehat{Q}$
, 
the gradient with respect to $z_t^{(i)}$ factors through the product structure:
\begin{align}
\nabla_{z_t^{(i)}} &\widehat{\text{MMD}}^2_{\otimes}(\widehat{P}_t,\widehat{Q})  = {\frac{2}{B^2}\sum_{j=1}^B k_p(p_i, p_j) \nabla_{z_t^{(i)}} k_z(z_t^{(i)}, z_t^{(j)})} \nonumber\\ &- {\frac{2}{BN_r}\sum_{j=1}^{N_r} k_p(p_i, p_j^{(r)}) \nabla_{z_t^{(i)}} k_z(z_t^{(i)}, z_j^{(r)})}.\vspace{-1mm}
\label{eq:product_gradient}
\end{align}
The prompt kernel $k_p(p_i, p_j^{(r)})$ acts as an attention weight: reference samples with semantically similar prompts contribute more strongly to the guidance signal. We present the steps of the resulting prompt-aware MMD guidance for prompt-conditioned diffusion models in Algorithm~\ref{alg:mmd_conditional}.

\vspace{-2mm}
\section{Numerical Results}\vspace{-2mm}


We evaluated MMD Guidance as a training-free adaptation method in two scenarios: (1) prompt-free (unconditional) distribution alignment, (2) Prompt-aware adaptation for text-conditioned latent diffusion models. 
In our evaluation, we consider the comparison against these baselines: (1) unguided diffusion sampling (No-Guidance), (2) classifier-free guidance (CFG) (3) classifier guidance (CG), where we utilize a binary classifier trained to distinguish the user's reference dataset from the samples of the original diffusion model. We report training-based baselines in Appendix~\ref{sec:numerical-appendix}.

\textbf{Models and Settings.} We conducted the experiments on the prompt-free latent diffusion models (LDM)~\citep{rombach2022high}, and Stable Diffusion v1.4 \citep{rombach2022high}. For prompt-aware experiments, we used Stable Diffusion XL \citep{podell2023sdxl}, and PixArt \citep{chen2023pixart}. In all the cases, the guidance is applied in the latent space as we discussed in Remark~\ref{remark LDM}.

\textbf{Evaluation.}
We evaluated the generated samples based on fidelity and distributional coverage.
For fidelity, we measure Fréchet distance (FD) \citep{heusel2017gans} and kernel distance (KD) \citep{binkowski2018kid}.
For distributional coverage, we report Coverage/Density \citep{naeem2020reliable}, and RRKE \citep{jalali2023information}.
All metrics are reported as mean over 5 random seeds, and we used DINOv2 \citep{oquab2023dinov2} as the image embedding, following the study by \cite{stein2023exposing}.

\subsection{MMD Guidance for Unconditional Diffusion Models}

\def\headerFont{\normalsize}     


\textbf{Synthetic datasets.}
We tested the MMD guidance and other baselines on synthetic data drawn from a 8-modal Gaussian Mixture Model (GMM) with eight components. Figure~\ref{fig:gmm_baseline} report metrics on a simulated user with 200 data uniformly drawn from the four orange-colored components. 
As reported in Figure~\ref{fig:gmm_baseline}, the MMD Guidance method led to the best FD and KD scores. Additional experimental results on several other synthetic GMMs are presented in Appendix~\ref{sec:numerical-appendix}.

%

\setlength{\textfloatsep}{5pt}
\setlength{\floatsep}{5pt}
\setlength{\intextsep}{5pt}

\begin{figure*}[!t]
\centering
\resizebox{0.8\textwidth}{!}{%
\begin{tikzpicture}
\def\imgwidth{12.4cm}
\def\imgwidthbig{12.0cm}

\node[inner sep=0, outer sep=0] (rowTop) at (9.24, -1.0)
  {\includegraphics[width=\imgwidth]{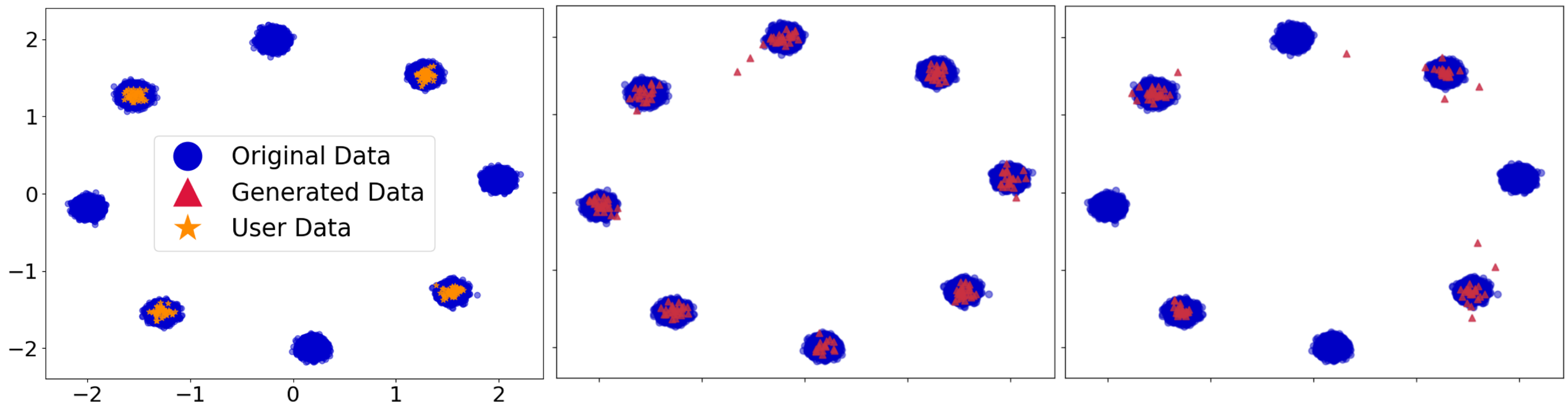}};
\node[inner sep=0, outer sep=0] (rowBot) at (9.45, -5.3)
  {\includegraphics[width=\imgwidthbig]{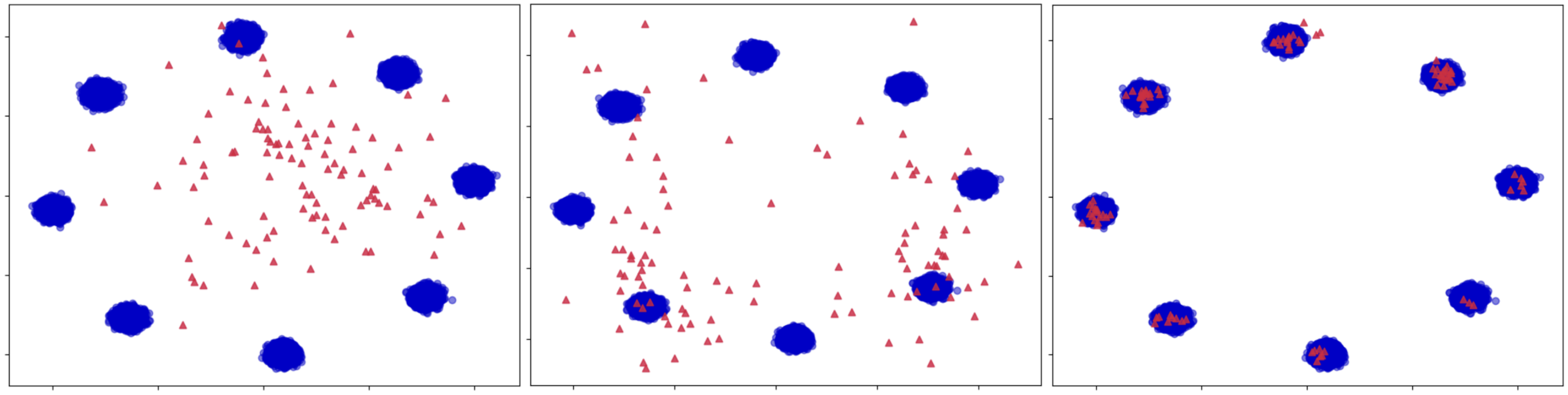}};

\def\colA{0.17}\def\colB{0.50}\def\colC{0.83}

\foreach \x/\dens/\cov/\fid in {
  \colA/{\headerFont\textbf{User Data}}/{0.08\!\pm\!0.04}/{0.06\!\pm\!0.03},
  \colB/{\headerFont\textbf{No-Guidance-DM}}/{0.22\!\pm\!0.04}/{0.44\!\pm\!0.03},
  \colC/{\headerFont\textbf{MMD}}/{0.14\!\pm\!0.04}/{0.07\!\pm\!0.02}
}{
  \node[
    font=\scriptsize, align=center,
    anchor=south,
    fill=white, rounded corners=0.6pt, inner sep=0.9pt
  ] at ($ (rowTop.north west)!\x!(rowTop.north east) + (0,0.08) $)
  { \dens\\[2pt]
   \textbf{FD:} $\cov$\\
   \textbf{\boldmath KD(\tiny{$\times 10$)}:} $\fid$};
}

\foreach \x/\dens/\cov/\fid in {
  \colA/{\normalsize\textbf{User-Trained-DM}}/{2.25\!\pm\!0.15}/{1.21\!\pm\!0.03},
  \colB/{\normalsize\textbf{CFG}}/{2.21\!\pm\!0.11}/{1.15\!\pm\!0.01},
  \colC/{\normalsize\textbf{CG}}/{0.17\!\pm\!0.04}/{0.44\!\pm\!0.02}
}{
  \node[
    font=\scriptsize, align=center,
    anchor=south,
    fill=white, rounded corners=0.6pt, inner sep=0.9pt
  ] at ($ (rowBot.north west)!\x!(rowBot.north east) + (0,0.08) $)
  { \dens\\[2pt]
   \textbf{FD:} $\cov$\\[-2pt]
   \textbf{\boldmath KD(\tiny{$\times 10)$}:} $\fid$};
}
  
\end{tikzpicture}
}
\caption{Comparison of MMD guidance with baselines on 100D Gaussian mixtures, when guiding to a user with 4 Gaussian modes.}
\label{fig:gmm_baseline}
\end{figure*}

\textbf{Changing Mode Proportions in a Mixture of Gaussians.}
To examine whether MMD guidance can correct a mismatch in mixture weights between the training distribution and a target reference distribution, we use an 8-component GMM with uniformly weighted training samples. When sampling from the model without any guidance, they remain close to uniform and therefore do not match an imbalanced target distribution. To define the reference distribution, we randomly select two mixture components and sample their mixture weights from a Dirichlet distribution with parameter \((1, 10)\), yielding a highly imbalanced mixture where one component dominates the other, and then apply MMD guidance using samples from this reference distribution. As shown in Figure~\ref{fig:gmm_weight}, MMD guidance not only steers samples toward the intended high-probability components but also closely reproduces the target mixture proportions of the reference distribution. Numerical evaluations are included in Appendix~\ref{sec:numerical-appendix}

\begin{figure*}[t]
\centering
\resizebox{0.83\textwidth}{!}{%
\begin{tikzpicture}
\def\imgwidth{12.4cm}
\def\imgwidthbig{12.0cm}

\node[inner sep=0, outer sep=0] (rowTop) at (9.24, -1.0)
  {\includegraphics[width=\imgwidth]{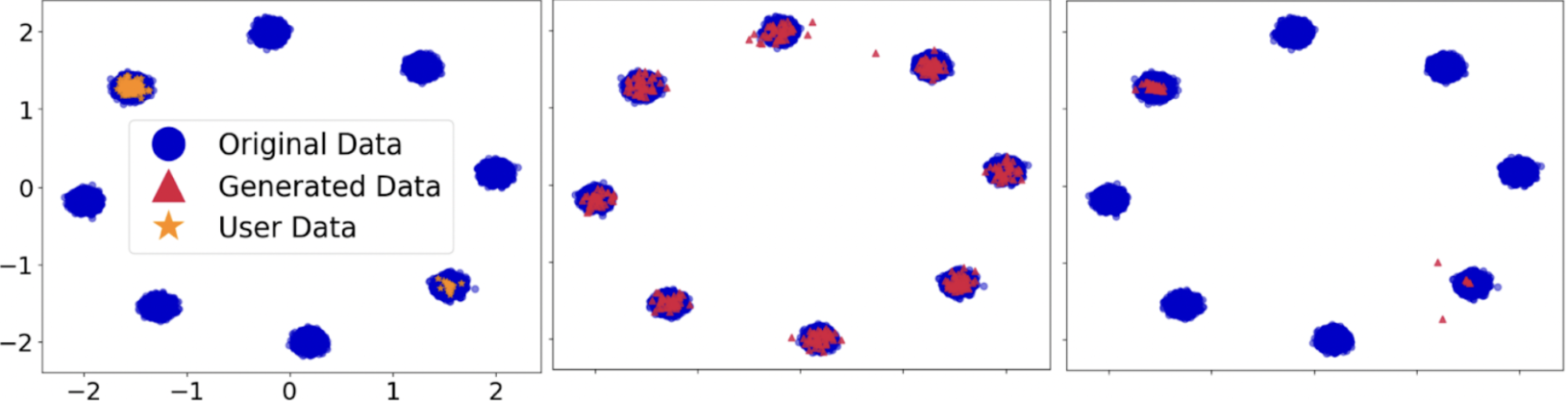}};

\def\colA{0.17}\def\colB{0.50}\def\colC{0.83}

\foreach \x/\dens/\cov/\fid in {
  \colA/{\headerFont\textbf{User Data}}/{0.02\!\pm\!0.00}/{0.03\!\pm\!0.00},
  \colB/{\headerFont\textbf{No-Guidance-DM}}/{4.22\!\pm\!0.01}/{2.96\!\pm\!0.01},
  \colC/{\headerFont\textbf{MMD}}/{0.13\!\pm\!0.00}/{0.24\!\pm\!0.01}
}{
  \node[
    font=\scriptsize, align=center,
    anchor=south,
    fill=white, rounded corners=0.6pt, inner sep=0.9pt
  ] at ($ (rowTop.north west)!\x!(rowTop.north east) + (0,0.08) $)
  { \dens\\[2pt]
   \textbf{FD:} $\cov$\\
   \textbf{\boldmath KD(\tiny{$\times 10$)}:} $\fid$};
}
\end{tikzpicture}
}
\caption{Effect of mode proportions in MMD guidance.}
\label{fig:gmm_weight}
\end{figure*}

\textbf{Real-image datasets.}
We also tested the MMD Guidance on real image dataset benchmarks, including FFHQ ~\citep{karras2019style}, and CelebA-HQ~\citep{karras2017progressive}, using pre-trained LDMs of \citep{rombach2022high}, and a mixture-type image dataset generated with Stable Diffusion v1.4 using the prompts \emph{car} and \emph{bike} under four different style variations (black-and-white, winter scenes, sketch, and cartoon).
We conducted experiments with two simulated users, each with 500 reference samples of the FFHQ dataset of the following specific styles: (User 1) people wearing sunglasses and (User 2) people wearing reading glasses. Table~\ref{ffhq-metric-dinov2} reports the measured scores, where MMD Guidance attained the highest fidelity and distributional coverage scores. The randomly generated data for the qualitative evaluation are shown in  Figure~\ref{ffhq_all}, which support the quantitative comparison. The additional results on FFHQ and other datasets as well as ablation results are in Appendix~\ref{sec:numerical-appendix}.

\begin{figure*}
    \centering
    \includegraphics[width=0.83\linewidth]{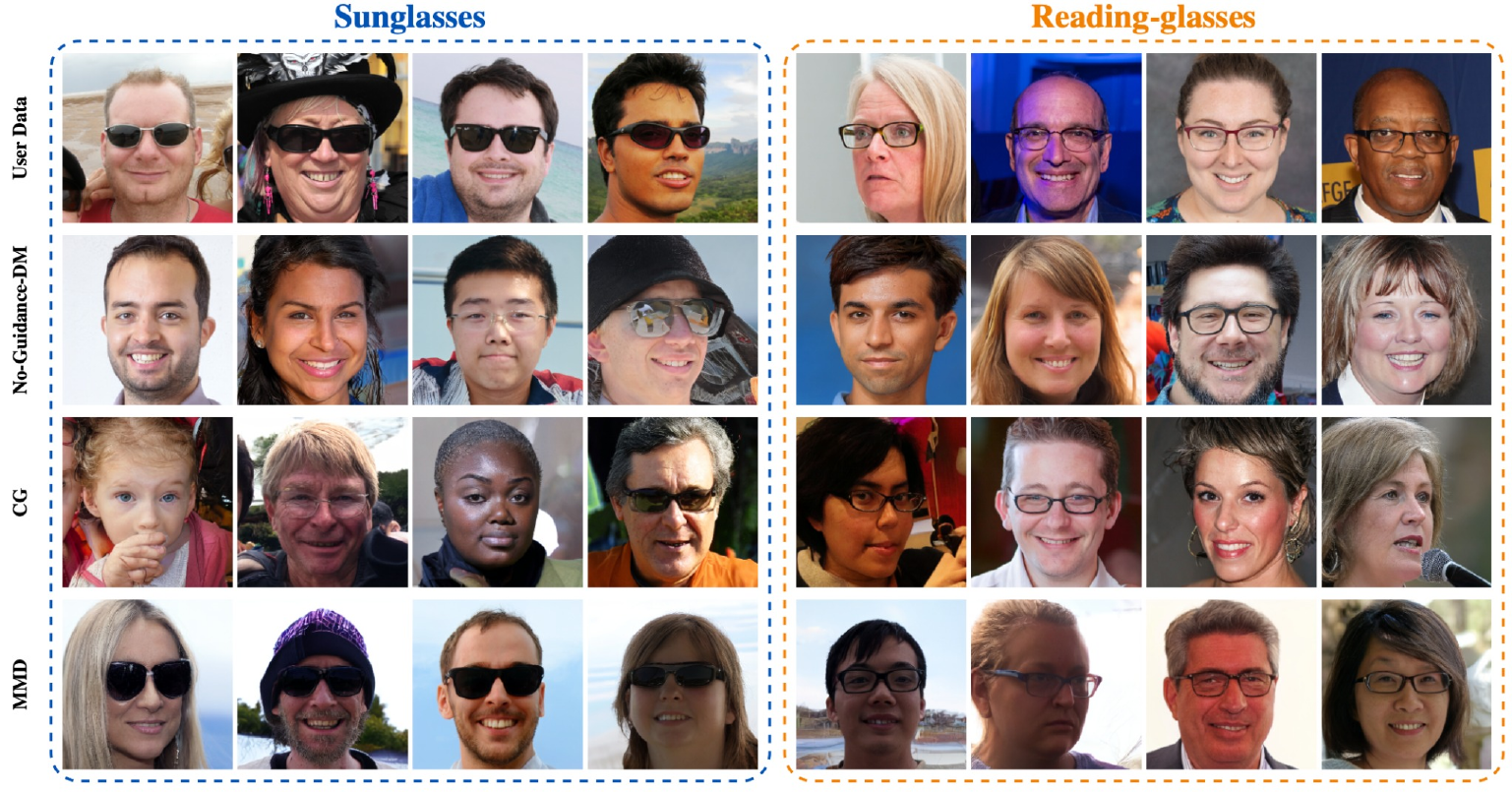}
    \caption{User's samples and generated data by unguided/guided LDMs on the FFHQ dataset.}
    \label{ffhq_all}
\end{figure*}

\begin{figure*}
    \centering
    \includegraphics[width=0.9\linewidth]{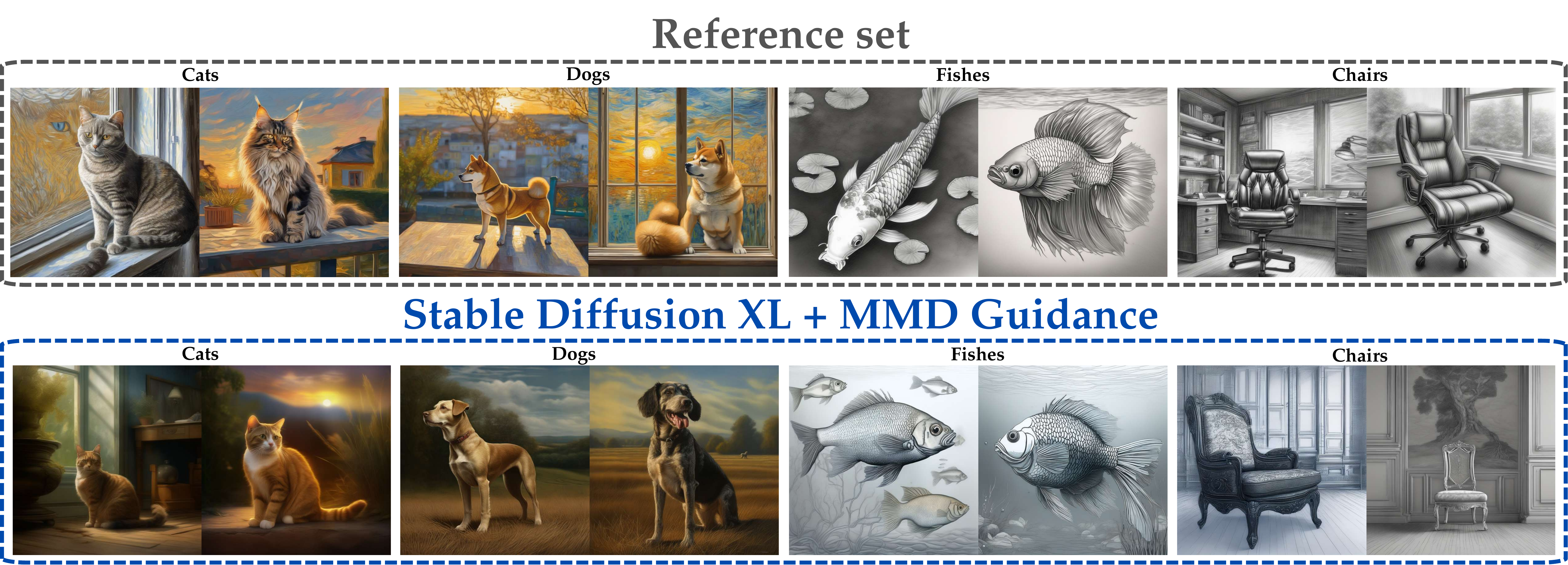}
    \caption{Qualitative comparison of reference set and MMD-guided image generation with SDXL.}\vspace{-2mm}
    \label{fig:sdxl_10k_user}
\end{figure*}

\begin{table}[!t]
    \caption{Evaluation scores for samples generated on FFHQ using MMD guidance vs. baselines.}
    \label{ffhq-metric-dinov2}
    \centering
    \resizebox{\linewidth}{!}{
    \begin{tabular}{ccccccc}
    \toprule
    User & Guidance & FD $\downarrow$ & KD $\downarrow$   & { Density } $\uparrow$ & { Coverage } $\uparrow$ \\ 
    \midrule
    \multirow{6}{*}{Sunglasses} & No-Guidance-DM & 1220.5  & 8.21   & 0.742  & 0.309  \\
    & User-Trained-DM & 1004.7  & 8.50   & 0.434  & 0.579  \\
    & Fine-tuning & 747.7  & 4.15   & 0.716  & 0.712  \\
    & CG & 1195.9  & 8.07   & 0.734  & 0.273  \\
    & Domain Guidance & 710.6  & 4.06   & 0.874  & 0.754  \\
    & MMD (Ours) & 693.0  & 3.25  & 1.131  & 0.791  \\
    \midrule
    \multirow{6}{*}{Reading-glasses} & No-Guidance-DM & 702.9  & 2.22   & 0.688  & 0.809  \\
    & User-Trained-DM & 1105.1  & 7.54   & 0.592  & 0.575  \\
    & Fine-tuning & 732.9  & 2.99   & 0.823  & 0.734  \\
    & CG & 678.5  &  2.16   & 0.701  & 0.779  \\
    & Domain Guidance & 667.6  & 2.08   & 0.757  & 0.817  \\
    & MMD (Ours) & 574.3  & 1.39   & 0.871 & 0.846  \\
    \bottomrule
    \end{tabular}
    }
\end{table}





\subsection{Prompt-Aware MMD Guidance}

\begin{table}[!t]
\caption{Evaluation scores for the dataset in Figure~\ref{fig:figure_one}. Comparison of SDXL and PixArt with No-Guidance-DM vs.\ MMD Guidance.}
\label{tab:sdxl_intro_table}
\centering
\resizebox{\linewidth}{!}{
\begin{tabular}{lccccc}
\toprule
Model & Guidance & FD $\downarrow$ & KD $\downarrow$  & Density ($\times 10 ^ 2$)  $\uparrow$ & Coverage  ($\times 10 ^ 2$) $\uparrow$ \\
\midrule
\multirow{2}{*}{SDXL}
& No-Guidance-DM & 1953.8  & 3.57  & 5.63 & 11.34\\
& MMD (Ours) & 1674.4      & 2.49  & 18.01 & 34.20 \\
\midrule
\multirow{2}{*}{PixArt}
& No-Guidance-DM & 1358.4 & 0.80  & 55.32 & 64.29\\
& MMD (Ours) &  1060.6 & 0.66  & 67.80 & 75.34\\
\bottomrule
\end{tabular}
}
\end{table}

\begin{table}[!t]
\caption{Runtime analysis of MMD guidance vs. no guidance diffusion sample generation.}
\label{tab:runtime_table}
\centering
\resizebox{\linewidth}{!}{
\begin{tabular}{lccccccc}
\toprule
 &  & \multicolumn{6}{c}{\#Samples generated} \\
\cmidrule(lr){3-8}
Model & Guidance & 50 & 100  & 200  & 300  & 400  & 500 \\
\midrule
\multirow{2}{*}{SDXL}
& No-Guidance-DM & 328 & 634  & 1258  & 1914  & 2516  & 3142 \\
& MMD (Ours) & 356 & 720  & 1378  & 2069  & 2749  & 3447 \\
\midrule
\multirow{2}{*}{PixArt}
& No-Guidance-DM & 317 & 632  & 1282  & 1852  & 2435  & 3063 \\
& MMD (Ours)  & 342 & 683  & 1348  & 2024  & 2694  & 3356\\
\bottomrule
\end{tabular}
}
\end{table}

\textbf{Prompt-aware MMD Guidance in LDMs.}
To further assess the effectiveness of MMD Guidance in conditional diffusion models, we constructed a dataset where categories were paired with distinct visual styles: cats with anime, dogs with Van Gogh paintings, cars with Pixar animation, and horses with cowboy movie styles. Prompts were generated using GPT-5, and the corresponding reference dataset was produced with SD-XL. We then evaluated MMD Guidance by sampling from the diffusion model without explicitly specifying the styles. As shown in Figure~\ref{fig:figure_one} and Figure~\ref{fig:pixart_mmd} in Appendix~\ref{sec:numerical-appendix}, the guided model successfully reproduced the visual characteristics of the target distribution, while the unguided LDM produced samples with mismatched styles. Numerical scores in Table~\ref{tab:sdxl_intro_table} confirm this observation: RRKE, FD, and KD decreased, suggesting improved distributional alignment, while Coverage increased, indicating improved diversity. To further evaluate MMD guidance on prompt-aware LDMs, we extended our experiments by generating eight categories of animals and objects, each rendered in four distinct styles (10,000 images in total). For each run, four categories were randomly selected, with one style assigned to each. Qualitative results are in Figure~\ref{fig:sdxl_10k_user}, where MMD-guided SDXL samples closely match the reference styles. More results are in Appendix~\ref{sec:numerical-appendix}.

\textbf{Scalability and time complexity of prompt-aware MMD guidance.}
To evaluate the time complexity of our proposed MMD guidance for LDMs, in Table \ref{tab:runtime_table}, we report the cumulative time for generating a group of $n$ samples with and without MMD guidance for large-scale prompt-aware LDMs using a 50-step diffusion process on an NVIDIA RTX 4090. The time values in the table support the efficiency of MMD guidance in the latent space of LDMs.

\textbf{Effect of the number of reference samples on MMD.}
we have measured the norm difference of the MMD gradient estimated with $N_r$ samples (empirical) and with 10000 samples (population estimate), and report the average $L_2$-norm errors over generating 1000 samples. As shown, the estimate becomes highly accurate, even with 100 data points, which follows the choice of MMD distance with the Gaussian and polynomial degree-3 kernels in our experiments.

\begin{table*}[!t]
\caption{Evaluation metrics as a function of the number of reference samples.}
\label{tab:ref_samples}
\centering
\begin{tabular}{lcccc}
\toprule
\# of reference samples & Guidance & FD $\downarrow$ & KD $\downarrow$ & RRKE $\downarrow$ \\
\midrule
$n = 0$   & No Guidance & 1953.75 & 3.57   & 1.93 \\
$n = 10$  & MMD         & 1805.92 & 2.7657 & 1.89 \\
$n = 50$  & MMD         & 1731.27 & 2.6882 & 1.83 \\
$n = 100$ & MMD         & 1691.42 & 2.5682 & 1.81 \\
$n = 150$ & MMD         & 1686.61 & 2.5301 & 1.80 \\
$n = 200$ & MMD         & 1674.45 & 2.4945 & 1.79 \\
\bottomrule
\end{tabular}
\end{table*}

\begin{table*}[!t]
\caption{
Results under different domain gaps and numbers of reference samples. KD scores are reported
multiplied by 100.
}
\centering
\small
\begin{tabular}{l c c c c c c c c c}
\toprule
\#Ref Sample &
r=2 FD &
r=2 KD&
r=2 RRKE&
r=20 FD&
r=20 KD&
r=20 RRKE&
Shift FD&
Shift KD&
Shift RRKE\\
\midrule
n=0   & 4.64 & 2.92 & 2.50 & 437.03 & 16.88 & 2.61 & 5.52 & 63.01 & 10.12 \\
n=10  & 1.37 & 1.15 & 0.66 & 123.22 & 4.42  & 1.50 & 0.60 & 17.99 & 1.96  \\
n=100 & 1.29 & 0.80 & 0.56 & 108.61 & 3.98  & 1.16 & 0.03 & 1.11  & 0.44  \\
\bottomrule
\end{tabular}

\label{tab:domain_gap}
\end{table*}

\textbf{Comparing MMD with Domain Guidance.} We compare the effectiveness of our method against the Domain Guidance baseline \citep{zhong2025domain} on the FFHQ dataset; here, guidance is applied toward two user-specific subsets, each containing 500 images: one subset consists of images of people wearing sunglasses, and the other wearing reading glasses.
Tables~\ref{ffhq-metric-dinov2} indicate, the MMD guidance outperforms the Domain Guidance baseline across the reported metrics.

\textbf{Effect of Domain Gap on MMD.} To evaluate the robustness of MMD guidance under different domain gaps, we conducted experiments on 8-centered GMMs. We consider both small ($r=2$) and large ($r=20$) separations between cluster centers, as well as a more challenging shifted-center setting where the target distribution lies outside the support of the training distribution. Results in Table~\ref{tab:domain_gap} show that larger domain gaps require stronger guidance and more reference samples to accurately align generated samples with the target distribution. Figure~\ref{fig:domain_gap} in Appendix~\ref{sec:numerical-appendix} visualizes this effect of domain gap.

\textbf{Effect of the intra-batch Diversity Term.} We study the impact of the diversity term through ablation experiments on synthetic GMMs and the FFHQ dataset. The results in Tables~\ref{tab:diversity_gmm},\ref{tab:diversity_ffhq} and Figure~\ref{fig:diversity-gmm} in Appendix~\ref{sec:numerical-appendix} show that removing the diversity term can lead to mode collapse, where the generated samples are concentrated in a compact region instead of covering all modes of the target distribution.

\vspace{-2.5mm}
\section{Conclusion and Limitations}\vspace{-1mm}
We presented MMD Guidance, a training-free adaptation method that enables diffusion models to generate samples matching arbitrary target distributions specified through small reference sets. By augmenting the reverse diffusion process with gradients of the Maximum Mean Discrepancy, our approach achieves distribution alignment without modifying model parameters, a useful capability when computational and data resources for retraining are limited. While MMD Guidance performs well with a limited number of reference samples, extremely sparse scenarios with highly limited data remain challenging due to variance in gradient estimation. The choice of kernel function may impact the performance, and adaptive kernel selection could further improve robustness. Future work could explore combining MMD with other divergences to leverage their complementary strengths, and extend the framework to video generation, where sequential consistency adds complexity.


\clearpage
\clearpage

\section*{Acknowledgments}
This work is partially supported by a grant from the Research Grants Council of the Hong Kong Special Administrative Region, China, Project 14210725, and is partially supported by CUHK Direct Research Grant with CUHK Project No. 4055164. The work is also supported by a grant under 1+1+1 CUHK-CUHK(SZ)-GDSTC Joint Collaboration Fund. Finally, the authors thank the anonymous reviewers and metareviewer for their constructive feedback and suggestions.

\section*{Impact Statement}
This work proposes a training-free, inference-time guidance mechanism for diffusion models that enables distribution matching to a target distribution specified by a limited reference set, reducing the computational and memory costs of adapting pre-trained generative models to new domains or user-specific distributions without retraining or fine-tuning. The method may be useful in settings where only a small number of reference samples are available and rapid deployment is needed, such as customization for specialized visual styles or domain-specific data generation. Since the approach operates entirely at inference time, does not modify model parameters, and does not introduce additional training data, its outputs depend on the user-provided reference samples, and any biases or constraints in those samples may be reflected in the generated results. As with existing diffusion guidance methods, responsible use requires appropriate selection and handling of reference data in accordance with applicable legal and ethical standards.

\bibliography{iclr2026_conference}

@inproceedings{wu2025fusing,
  title   = {When Kernels Multiply, Clusters Unify: Fusing Embeddings with the Kronecker Product},
  author  = {Wu, Youqi and Zhang, Jingwei and Farnia, Farzan},
  booktitle={The Thirty-ninth Annual Conference on Neural Information Processing Systems},
  year    = {2025},
}

@inproceedings{
jalali2024conditional,
title={Conditional Vendi Score: Prompt-Aware Diversity Evaluation for Generative {AI} Models and {LLM}s},
author={Mohammad Jalali and Azim Ospanov and Amin Gohari and Farzan Farnia},
booktitle={The 29th International Conference on Artificial Intelligence and Statistics},
year={2026},
url={https://openreview.net/forum?id=iDrZToIsyd}
}

@article{jalali2025sparke,
  title={SPARKE: Scalable Prompt-Aware Diversity and Novelty Guidance in Diffusion Models via RKE Score},
  author={Jalali, Mohammad and Lei, Haoyu and Gohari, Amin and Farnia, Farzan},
  journal={Advances in Neural Information Processing Systems},
  year={2025},
}

@article{ospanov2025towards,
  title={Towards a scalable reference-free evaluation of generative models},
  author={Ospanov, Azim and Zhang, Jingwei and Jalali, Mohammad and Cao, Xuenan and Bogdanov, Andrej and Farnia, Farzan},
  journal={Advances in Neural Information Processing Systems},
  volume={37},
  pages={120892--120927},
  year={2024}
}

@inproceedings{jalali2025spec,
  title={Towards an Explainable Comparison and Alignment of Feature Embeddings},
  author={Jalali, Mohammad and Nia, Bahar Dibaei and Farnia, Farzan},
  booktitle={International Conference on Machine Learning},
  pages={26757--26796},
  year={2025},
  organization={PMLR}
}

@inproceedings{wu2026maximum_isit,
  title     = {The Maximum von Neumann Entropy Principle: Theory and Applications in Machine Learning},
  author    = {Wu, Youqi and Farnia, Farzan},
  booktitle = {IEEE International Symposium on Information Theory (ISIT)},
  year      = {2026}
}

@inproceedings{
zhang2026exploring,
title={Exploring More to Solve More: Boosting Diversity in Text Diffusion Models via Entropy-Based Guidance},
author={Zhang, Jingwei and
Lei, Haoyu and Feng, Zijin and Sun, Jiacheng and Farnia, Farzan},
booktitle={Forty-third International Conference on Machine Learning},
year={2026},
}

@inproceedings{wu2026koda,
  title={KODA: Contrastive Representation Comparison and Alignment for Vision-Language Foundation Models},
  author={Wu, Youqi and Jalali, Mohammad and Farnia, Farzan},
  booktitle={International Conference on Machine Learning},
  year={2026},
  organization={PMLR}
}

@inproceedings{gong2025kernel,
  title={Kernel-based Unsupervised Embedding Alignment for Enhanced Visual Representation in Vision-language Models},
  author={Gong, Shizhan and Jiang, Yankai and Dou, Qi and Farnia, Farzan},
  booktitle={International Conference on Machine Learning},
  pages={19912--19931},
  year={2025},
  organization={PMLR}
}

@InProceedings{hu2025multiarmedbandit,
  title = 	 {A Multi-Armed Bandit Approach to Online Selection and Evaluation of Generative Models},
  author =       {Hu, Xiaoyan and Leung, Ho-fung and Farnia, Farzan},
  booktitle = 	 {Proceedings of The 28th International Conference on Artificial Intelligence and Statistics},
  pages = 	 {1864--1872},
  year = 	 {2025},
  volume = 	 {258},
  series = 	 {Proceedings of Machine Learning Research},
  publisher =    {PMLR},
}

@inproceedings{hu2025online,
  title={PAK-UCB Contextual Bandit: An Online Learning Approach to Prompt-Aware Selection of Generative Models and LLMs},
  author={Hu, Xiaoyan and Leung, Ho-fung and Farnia, Farzan},
    booktitle={Proceedings of the 26th International Conference on Machine Learning (ICML)},
  year={2025},
}

@inproceedings{
rezaei2024be,
title={Be More Diverse than the Most Diverse: Optimal Mixtures of Generative Models via Mixture-{UCB} Bandit Algorithms},
author={Parham Rezaei and Farzan Farnia and Cheuk Ting Li},
booktitle={The Thirteenth International Conference on Learning Representations},
year={2025},
url={https://openreview.net/forum?id=2Chkk5Ye2s}
}

@inproceedings{wang2023distributedKID,
  title={On the distributed evaluation of generative models},
  author={Wang, Zixiao and Farnia, Farzan and Lin, Zhenghao and Shen, Yunheng and Yu, Bei},
  booktitle={Proceedings of the IEEE/CVF International Conference on Computer Vision Workshops},
  pages={7644--7653},
  year={2025}
}

@inproceedings{ospanov2024statvendi,
  title={Do Vendi Scores Converge with Finite Samples? Truncated Vendi Score for Finite-Sample Convergence Guarantees},
  author={Ospanov, Azim and Farnia, Farzan},
  booktitle={The 41st Conference on Uncertainty in Artificial Intelligence},
  year={2025}
}

@InProceedings{Ospanov_2025_ICCV,
    author    = {Ospanov, Azim and Jalali, Mohammad and Farnia, Farzan},
    title     = {Scendi Score: Prompt-Aware Diversity Evaluation via Schur Complement of CLIP Embeddings},
    booktitle = {Proceedings of the IEEE/CVF International Conference on Computer Vision (ICCV)},
    month     = {October},
    year      = {2025},
    pages     = {16927-16937}
}

@InProceedings{zhang2024interpretable,
  title = 	 {An Interpretable Evaluation of Entropy-based Novelty of Generative Models},
  author =       {Zhang, Jingwei and Li, Cheuk Ting and Farnia, Farzan},
  booktitle = 	 {Proceedings of the 41st International Conference on Machine Learning},
  pages = 	 {59148--59172},
  year = 	 {2024},
  volume = 	 {235},
  series = 	 {Proceedings of Machine Learning Research},
  month = 	 {21--27 Jul},
  publisher =    {PMLR},
}

@inproceedings{zhang2025finc,
  title={Unveiling Differences in Generative Models: A Scalable Differential Clustering Approach},
  author={Zhang, Jingwei and  Jalali, Mohammad and Li, Cheuk Ting and Farnia, Farzan},
  booktitle={{Proceedings of the IEEE/CVF Conference on Computer Vision and Pattern Recognition (CVPR)}},
  year={2025}
}

@article{hu2025promptwise,
  title={PromptWise: Online Learning for Cost-Aware Prompt Assignment in Generative Models},
  author={Hu, Xiaoyan and Pick, Lauren and Leung, Ho-fung and Farnia, Farzan},
  journal={arXiv preprint arXiv:2505.18901},
  year={2025}
}

@article{friedman2023vendiscorediversityevaluation,
  title={The vendi score: A diversity evaluation metric for machine learning},
  author={Dan Friedman, Dan and Dieng, Adji Bousso},
  journal={Transactions on machine learning research},
  year={2023}
}

@article{bamberger2022johnson,
  title={Johnson--Lindenstrauss embeddings with Kronecker structure},
  author={Bamberger, Stefan and Krahmer, Felix and Ward, Rachel},
  journal={SIAM Journal on Matrix Analysis and Applications},
  volume={43},
  number={4},
  pages={1806--1850},
  year={2022},
  publisher={SIAM}
}

@inproceedings{dhariwal2021diffusion,
  title     = {Diffusion Models Beat GANs on Image Synthesis},
  author    = {Dhariwal, Prafulla and Nichol, Alex},
  booktitle = {Advances in Neural Information Processing Systems (NeurIPS)},
  year      = {2021}
}

@article{ho2022classifierfree,
  title   = {Classifier-Free Diffusion Guidance},
  author  = {Ho, Jonathan and Salimans, Tim},
  journal = {arXiv preprint arXiv:2207.12598},
  year    = {2022}
}

@inproceedings{chung2025cfg,
title={{CFG}++: Manifold-constrained Classifier Free Guidance for Diffusion Models},
author={Hyungjin Chung and Jeongsol Kim and Geon Yeong Park and Hyelin Nam and Jong Chul Ye},
booktitle={The Thirteenth International Conference on Learning Representations},
year={2025},
url={https://openreview.net/forum?id=E77uvbOTtp}
}

@misc{galashov2025,
      title={Learn to Guide Your Diffusion Model}, 
      author={Alexandre Galashov and Ashwini Pokle and Arnaud Doucet and Arthur Gretton and Mauricio Delbracio and Valentin De Bortoli},
      year={2025},
      eprint={2510.00815},
      archivePrefix={arXiv},
      primaryClass={cs.LG},
      url={https://arxiv.org/abs/2510.00815}, 
}

@article{malarz2025classifier,
  title={Classifier-free Guidance with Adaptive Scaling},
  author={Malarz, Dawid and Kasymov, Artur and Zieba, Maciej and Tabor, Jacek and Spurek, Przemyslaw},
  journal={arXiv preprint arXiv:2502.10574},
  year={2025}
}

@inproceedings{zhong2025domain,
  title={Domain Guidance: A Simple Transfer Approach for a Pre-trained Diffusion Model},
  author={Jincheng Zhong and XiangCheng Zhang and Jianmin Wang and Mingsheng Long},
  booktitle={The Thirteenth International Conference on Learning Representations},
  year={2025},
  url={https://openreview.net/forum?id=PplM2kDrl3}
}

@article{somepalli2024measuring,
title={Measuring Style Similarity in Diffusion Models},
author={Somepalli, Gowthami and Gupta, Anubhav and Gupta, Kamal and Palta, Shramay and Goldblum, Micah and Geiping, Jonas and Shrivastava, Abhinav and Goldstein, Tom},
journal={arXiv preprint arXiv:2404.01292},
year={2024}
}

@inproceedings{liu2022composablediffusion,
  title     = {Compositional Visual Generation with Composable Diffusion Models},
  author    = {Liu, Nan and Li, Shuang and Du, Yilun and Torralba, Antonio and Tenenbaum, Joshua B.},
  booktitle = {European Conference on Computer Vision (ECCV)},
  series    = {LNCS},
  year      = {2022},
}

@inproceedings{hertz2023prompttoprompt,
  title     = {Prompt-to-Prompt Image Editing with Cross-Attention Control},
  author    = {Hertz, Amir and Mokady, Ron and Tenenbaum, Jay and Aberman, Kfir and Pritch, Yael and Cohen-Or, Daniel},
  booktitle = {International Conference on Learning Representations (ICLR)},
  year      = {2023}
}

@inproceedings{meng2022sdedit,
  title     = {{SDE}dit: Guided Image Synthesis and Editing with Stochastic Differential Equations},
  author    = {Meng, Chenlin and He, Yutong and Song, Yang and Song, Jiaming and Wu, Jiajun and Zhu, Jun-Yan and Ermon, Stefano},
  booktitle = {International Conference on Learning Representations (ICLR)},
  year      = {2022}
}

@article{gal2022textualinversion,
  title   = {An Image is Worth One Word: Personalizing Text-to-Image Generation using Textual Inversion},
  author  = {Gal, Rinon and Alaluf, Yuval and Atzmon, Yuval and Patashnik, Or and Bermano, Amit H. and Chechik, Gal and Cohen-Or, Daniel},
  journal = {arXiv preprint arXiv:2208.01618},
  year    = {2022}
}

@inproceedings{ruiz2023dreambooth,
  title     = {DreamBooth: Fine Tuning Text-to-Image Diffusion Models for Subject-Driven Generation},
  author    = {Ruiz, Nataniel and Li, Yuanzhen and Jampani, Varun and Pritch, Yael and Rubinstein, Michael and Aberman, Kfir},
  booktitle = {IEEE/CVF Conference on Computer Vision and Pattern Recognition (CVPR)},
  year      = {2023}
}

@inproceedings{hu2022lora,
  title     = {{LoRA}: Low-Rank Adaptation of Large Language Models},
  author    = {Hu, Edward J. and Shen, Yelong and Wallis, Phillip and Allen-Zhu, Zeyuan and Li, Yuanzhi and Wang, Shean and Wang, Lu and Chen, Weizhu},
  booktitle = {International Conference on Learning Representations (ICLR)},
  year      = {2022}
}

@inproceedings{zhang2023controlnet,
  title     = {Adding Conditional Control to Text-to-Image Diffusion Models},
  author    = {Zhang, Lvmin and Agrawala, Maneesh},
  booktitle = {IEEE/CVF International Conference on Computer Vision (ICCV)},
  year      = {2023}
}

@inproceedings{mou2024t2iadapter,
  title     = {{T2I-Adapter}: Learning Adapters to Dig Out More Controllable Ability for Text-to-Image Diffusion Models},
  author    = {Mou, Chong and Wang, Xintao and Xie, Liangbin and Wu, Yanze and Zhang, Jian and Qi, Zhongang and Shan, Ying and Qie, Xiaohu},
  booktitle = {Proceedings of the AAAI Conference on Artificial Intelligence (AAAI)},
  volume    = {38},
  pages     = {4296--4304},
  year      = {2024},
}

@inproceedings{song2021denoising,
  title     = {Denoising Diffusion Implicit Models},
  author    = {Song, Jiaming and Meng, Chenlin and Ermon, Stefano},
  booktitle = {International Conference on Learning Representations},
  year      = {2021},
}

@inproceedings{sutherland2018efficient,
  title={Efficient and principled score estimation with nystr{\"o}m kernel exponential families},
  author={Sutherland, Danica J and Strathmann, Heiko and Arbel, Michael and Gretton, Arthur},
  booktitle={International Conference on Artificial Intelligence and Statistics},
  pages={652--660},
  year={2018},
  organization={PMLR}
}

@inproceedings{long2015dan,
  title     = {Learning Transferable Features with Deep Adaptation Networks},
  author    = {Long, Mingsheng and Cao, Yue and Wang, Jianmin and Jordan, Michael I.},
  booktitle = {Proceedings of the 32nd International Conference on Machine Learning (ICML)},
  series    = {Proceedings of Machine Learning Research},
  volume    = {37},
  year      = {2015}
}

@inproceedings{li2015gmmn,
  title     = {Generative Moment Matching Networks},
  author    = {Li, Yujia and Swersky, Kevin and Zemel, Richard S.},
  booktitle = {Proceedings of the 32nd International Conference on Machine Learning (ICML)},
  series    = {Proceedings of Machine Learning Research},
  volume    = {37},
  year      = {2015}
}

@inproceedings{li2017mmdgan,
  title     = {{MMD} {GAN}: Towards Deeper Understanding of Moment Matching Network},
  author    = {Li, Chun-Liang and Chang, Wei-Cheng and Cheng, Yu and Yang, Yiming and P{\'o}czos, Barnab{\'a}s},
  booktitle = {Advances in Neural Information Processing Systems (NeurIPS)},
  year      = {2017}
}

@inproceedings{binkowski2018kid,
  title     = {Demystifying {MMD} {GAN}s},
  author    = {Bi{\'n}kowski, Miko{\l}aj and Sutherland, Dougal J. and Arbel, Michael and Gretton, Arthur},
  booktitle = {International Conference on Learning Representations (ICLR)},
  year      = {2018}
}

@article{bommasani2021opportunities,
  title={On the opportunities and risks of foundation models},
  author={Bommasani, Rishi and Hudson, Drew A and Adeli, Ehsan and Altman, Sylvia and Arora, Simran and von Arx, Sydney and Bernstein, Michael S and Bohg, Jeannette and Bosselut, Antoine and Brunskill, Emma and others},
  journal={arXiv preprint arXiv:2108.07258},
  year={2021}
}

@inproceedings{brown2020language,
  title={Language models are few-shot learners},
  author={Brown, Tom and Mann, Benjamin and Ryder, Nick and Subbiah, Melanie and Kaplan, Jared D and Dhariwal, Prafulla and Neelakantan, Arvind and Shyam, Pranav and Sastry, Girish and Askell, Amanda and others},
  booktitle={Advances in neural information processing systems},
  volume={33},
  pages={1877--1901},
  year={2020}
}

@inproceedings{ho2020denoising,
  title={Denoising diffusion probabilistic models},
  author={Ho, Jonathan and Jain, Ajay and Abbeel, Pieter},
  booktitle={Advances in Neural Information Processing Systems},
  volume={33},
  pages={6840--6851},
  year={2020}
}

@inproceedings{rombach2022high,
  title={High-resolution image synthesis with latent diffusion models},
  author={Rombach, Robin and Blattmann, Andreas and Lorenz, Dominik and Esser, Patrick and Ommer, Bj{\"o}rn},
  booktitle={Proceedings of the IEEE/CVF conference on computer vision and pattern recognition},
  pages={10684--10695},
  year={2022}
}

@inproceedings{poole2022dreamfusion,
  title={Dreamfusion: Text-to-3d using 2d diffusion},
  author={Poole, Ben and Jain, Ajay and Barron, Jonathan T and Mildenhall, Ben},
  booktitle={International Conference on Learning Representations},
  year={2023}
}

@inproceedings{zhang2023adding,
  title={Adding conditional control to text-to-image diffusion models},
  author={Zhang, Lvmin and Rao, Anyi and Agrawala, Maneesh},
  booktitle={Proceedings of the IEEE/CVF international conference on computer vision},
  pages={3836--3847},
  year={2023}
}

@article{gretton2012kernel,
  title={A kernel two-sample test},
  author={Gretton, Arthur and Borgwardt, Karsten M and Rasch, Malte J and Sch{\"o}lkopf, Bernhard and Smola, Alexander},
  journal={The journal of machine learning research},
  volume={13},
  number={1},
  pages={723--773},
  year={2012},
  publisher={JMLR. org}
}

@article{weed2019sharp,
  title={Sharp Asymptotic and Finite-Sample Rates of Convergence of Empirical Measures in Wasserstein Distance},
  author={Weed, Jonathan and Bach, Francis},
  journal={Bernoulli},
  volume={25},
  number={4A},
  pages={2620--2648},
  year={2019}
}

@article{podell2023sdxl,
  title={Sdxl: Improving latent diffusion models for high-resolution image synthesis},
  author={Podell, Dustin and English, Zion and Lacey, Kyle and Blattmann, Andreas and Dockhorn, Tim and M{\"u}ller, Jonas and Penna, Joe and Rombach, Robin},
  journal={arXiv preprint arXiv:2307.01952},
  year={2023}
}

@article{sriperumbudur2012empirical,
  title={On the empirical estimation of integral probability metrics},
  author={Sriperumbudur, Bharath K and Fukumizu, Kenji and Gretton, Arthur and Sch{\"o}lkopf, Bernhard and Lanckriet, Gert RG},
  journal={Electronic Journal of Statistics},
  volume={6},
  pages={1550--1599},
  year={2012}
}

@article{chen2023pixart,
  title={Pixart-$alpha $: Fast training of diffusion transformer for photorealistic text-to-image synthesis},
  author={Chen, Junsong and Yu, Jincheng and Ge, Chongjian and Yao, Lewei and Xie, Enze and Wu, Yue and Wang, Zhongdao and Kwok, James and Luo, Ping and Lu, Huchuan and others},
  journal={arXiv preprint arXiv:2310.00426},
  year={2023}
}

@article{oquab2023dinov2,
  title={Dinov2: Learning robust visual features without supervision},
  author={Oquab, Maxime and Darcet, Timoth{\'e}e and Moutakanni, Th{\'e}o and Vo, Huy and Szafraniec, Marc and Khalidov, Vasil and Fernandez, Pierre and Haziza, Daniel and Massa, Francisco and El-Nouby, Alaaeldin and others},
  journal={arXiv preprint arXiv:2304.07193},
  year={2023}
}

@article{stein2023exposing,
  title={Exposing flaws of generative model evaluation metrics and their unfair treatment of diffusion models},
  author={Stein, George and Cresswell, Jesse and Hosseinzadeh, Rasa and Sui, Yi and Ross, Brendan and Villecroze, Valentin and Liu, Zhaoyan and Caterini, Anthony L and Taylor, Eric and Loaiza-Ganem, Gabriel},
  journal={Advances in Neural Information Processing Systems},
  volume={36},
  pages={3732--3784},
  year={2023}
}

@inproceedings{naeem2020reliable,
  title={Reliable fidelity and diversity metrics for generative models},
  author={Naeem, Muhammad Ferjad and Oh, Seong Joon and Uh, Youngjung and Choi, Yunjey and Yoo, Jaejun},
  booktitle={International conference on machine learning},
  pages={7176--7185},
  year={2020},
  organization={PMLR}
}

@article{jalali2023information,
  title={An information-theoretic evaluation of generative models in learning multi-modal distributions},
  author={Jalali, Mohammad and Li, Cheuk Ting and Farnia, Farzan},
  journal={Advances in Neural Information Processing Systems},
  volume={36},
  pages={9931--9943},
  year={2023}
}

@article{heusel2017gans,
  title={Gans trained by a two time-scale update rule converge to a local nash equilibrium},
  author={Heusel, Martin and Ramsauer, Hubert and Unterthiner, Thomas and Nessler, Bernhard and Hochreiter, Sepp},
  journal={Advances in neural information processing systems},
  volume={30},
  year={2017}
}

@inproceedings{
nia2026mixturegreedy,
title={Mixture-Greedy for Online Generative Model Selection: Is {UCB} Necessary in Diversity-Aware Multi-Armed Bandits?},
author={Bahar Dibaei Nia and Farzan Farnia},
booktitle={ICML 2026 Workshop on Foundations of Deep Generative Models: Understanding Memorization, Generalization, and Reasoning},
year={2026},
url={https://openreview.net/forum?id=8h8nkxR1ac}
}

@inproceedings{
jafari2026dakucb,
title={{DAK}-{UCB}: Diversity-Aware Prompt Routing for {LLM}s and Generative Models},
author={Donya Jafari and Farzan Farnia},
booktitle={The Fourteenth International Conference on Learning Representations},
year={2026},
url={https://openreview.net/forum?id=nnN2TKlS5C}
}

@inproceedings{
lotfian2026promptsplit,
title={PromptSplit: Revealing Prompt-Level Disagreement in Generative Models},
author={Mehdi Lotfian and Mohammad Jalali and Farzan Farnia},
booktitle={ICML 2026 Workshop on Combining Theory and Benchmarks: Towards A Virtuous Cycle to Understand and Guarantee Foundation Model Performance},
year={2026},
url={https://openreview.net/forum?id=sUThOs5ADa}
}

@inproceedings{karras2019style,
  title={A style-based generator architecture for generative adversarial networks},
  author={Karras, Tero and Laine, Samuli and Aila, Timo},
  booktitle={Proceedings of the IEEE/CVF conference on computer vision and pattern recognition},
  pages={4401--4410},
  year={2019}
}

@article{karras2017progressive,
  title={Progressive growing of gans for improved quality, stability, and variation},
  author={Karras, Tero and Aila, Timo and Laine, Samuli and Lehtinen, Jaakko},
  journal={arXiv preprint arXiv:1710.10196},
  year={2017}
}

@inproceedings{kim2022diffusionclip,
  title={Diffusionclip: Text-guided diffusion models for robust image manipulation},
  author={Kim, Gwanghyun and Kwon, Taesung and Ye, Jong Chul},
  booktitle={Proceedings of the IEEE/CVF conference on computer vision and pattern recognition},
  pages={2426--2435},
  year={2022}
}

@inproceedings{nichol2022glide,
  title={GLIDE: Towards Photorealistic Image Generation and Editing with Text-Guided Diffusion Models},
  author={Nichol, Alexander Quinn and Dhariwal, Prafulla and Ramesh, Aditya and Shyam, Pranav and Mishkin, Pamela and Mcgrew, Bob and Sutskever, Ilya and Chen, Mark},
  booktitle={International Conference on Machine Learning},
  pages={16784--16804},
  year={2022},
  organization={PMLR}
}

@inproceedings{liu2023more,
  title={More control for free! image synthesis with semantic diffusion guidance},
  author={Liu, Xihui and Park, Dong Huk and Azadi, Samaneh and Zhang, Gong and Chopikyan, Arman and Hu, Yuxiao and Shi, Humphrey and Rohrbach, Anna and Darrell, Trevor},
  booktitle={Proceedings of the IEEE/CVF Winter Conference on Applications of Computer Vision},
  year={2023}
}

@inproceedings{tevet2023human,
  title={Human Motion Diffusion Model},
  author={Tevet, Guy and Raab, Sigal and Gordon, Brian and Shafir, Yoni and Cohen-or, Daniel and Bermano, Amit Haim},
  booktitle={The Eleventh International Conference on Learning Representations},
  year={2023}
}

@article{zhao2022egsde,
  title={Egsde: Unpaired image-to-image translation via energy-guided stochastic differential equations},
  author={Zhao, Min and Bao, Fan and Li, Chongxuan and Zhu, Jun},
  journal={Advances in Neural Information Processing Systems},
  volume={35},
  pages={3609--3623},
  year={2022}
}

@inproceedings{
  he2024manifold,
  title={Manifold Preserving Guided Diffusion},
  author={Yutong He and Naoki Murata and Chieh-Hsin Lai and Yuhta Takida and Toshimitsu Uesaka and Dongjun Kim and Wei-Hsiang Liao and Yuki Mitsufuji and J Zico Kolter and Ruslan Salakhutdinov and Stefano Ermon},
  booktitle={The Twelfth International Conference on Learning Representations},
  year={2024},
}

@inproceedings{bansal2023universal,
  title={Universal guidance for diffusion models},
  author={Bansal, Arpit and Chu, Hong-Min and Schwarzschild, Avi and Sengupta, Soumyadip and Goldblum, Micah and Geiping, Jonas and Goldstein, Tom},
  booktitle={Proceedings of the IEEE/CVF conference on computer vision and pattern recognition},
  pages={843--852},
  year={2023}
}

@article{ye2024tfg,
  title={Tfg: Unified training-free guidance for diffusion models},
  author={Ye, Haotian and Lin, Haowei and Han, Jiaqi and Xu, Minkai and Liu, Sheng and Liang, Yitao and Ma, Jianzhu and Zou, James Y and Ermon, Stefano},
  journal={Advances in Neural Information Processing Systems},
  volume={37},
  pages={22370--22417},
  year={2024}
}

@inproceedings{yu2023freedom,
  title={Freedom: Training-free energy-guided conditional diffusion model},
  author={Yu, Jiwen and Wang, Yinhuai and Zhao, Chen and Ghanem, Bernard and Zhang, Jian},
  booktitle={Proceedings of the IEEE/CVF International Conference on Computer Vision},
  pages={23174--23184},
  year={2023}
}

@inproceedings{szegedy2016rethinking,
  title={Rethinking the inception architecture for computer vision},
  author={Szegedy, Christian and Vanhoucke, Vincent and Ioffe, Sergey and Shlens, Jon and Wojna, Zbigniew},
  booktitle={Proceedings of the IEEE conference on computer vision and pattern recognition},
  pages={2818--2826},
  year={2016}
}

@article{daras2023consistent,
  title={Consistent diffusion models: Mitigating sampling drift by learning to be consistent},
  author={Daras, Giannis and Dagan, Yuval and Dimakis, Alex and Daskalakis, Constantinos},
  journal={Advances in Neural Information Processing Systems},
  volume={36},
  pages={42038--42063},
  year={2023}
}
\bibliographystyle{icml2026}

\clearpage
\appendix
\onecolumn
\section{Proofs}

\subsection{Proof of Theorem~\ref{thm:gradient_concentration}}

Define the centered random vectors:
\begin{equation}
Y_j = \nabla_{z_0} k(z_0, z_j^{(r)}) - \mathbb{E}_{z' \sim Q}[\nabla_{z_0} k(z_0, z')], \quad j = 1, \ldots, N_r.
\end{equation}

Since the reference samples $z_j^{(r)}$ are drawn i.i.d. from $Q$ independently of how $z_0$ was generated, the $Y_j$ are independent random vectors in $\mathbb{R}^d$ with $\mathbb{E}[Y_j] = 0$ (following its definition) and $\big\|Y_j\big\|_2 \leq 2L$ almost surely, since $\big\|\nabla_{z_0} k(z_0, \cdot)\big\|_2 \leq L$ by the Lipschitz property.

Applying the Hoeffding inequality for random Hilbert-Schmidt operators \citep{sutherland2018efficient}, we have the following to hold with probability at least $1-\delta$:
\begin{equation}
\left\|\frac{1}{N_r}\sum_{j=1}^{N_r} Y_j\right\|_2 \leq \frac{2L}{\sqrt{N_r}}\left(1 + \sqrt{2\log\bigl(\frac{1}{\delta}}\bigr)\right).
\end{equation}

Since $\widehat{g}_{\text{cross}}(z_0) - g^*_{\text{cross}}(z_0) = -\frac{2}{N_r}\sum_{j=1}^{N_r} Y_j$, the following will hold with probability at least $1-\delta$:
\begin{equation}
\big\|\widehat{g}_{\text{cross}}(z_0) - g^*_{\text{cross}}(z_0)\big\|_2 \leq \frac{4L}{\sqrt{N_r}}\left(1 + \sqrt{2\log\bigl({1}/{\delta}\bigr)}\right).
\end{equation}

\subsection{Proof of Corollary~\ref{cor:gaussian_rbf_concise}}

For the Gaussian RBF kernel, the gradient is $\nabla_x k(x,y) = -\frac{1}{\sigma^2}k(x,y)(x-y)$, yielding 
$$\big\|\nabla_x k(x,y)\big\|_2 = \frac{\big\|x-y\big\|_2}{\sigma^2} \exp\bigl(-\frac{\big\|x-y\big\|_2^2}{2\sigma^2}\bigr)$$
Setting $t = \big\|x-y\big\|_2/(\sigma\sqrt{2})$, this becomes $(\sqrt{2}/\sigma) t e^{-t^2}$. Since $\max_{t \geq 0} t e^{-t^2} = e^{-1/2}/\sqrt{2}$ (achieved at $t = 1/\sqrt{2}$), we have the Lipschitz constant $L = 1/(\sigma\sqrt{e})$. Knowing that $\frac{4}{\sqrt{e}}<3$, we can plug this Lipschitz constant into the assumption of Theorem~\ref{thm:gradient_concentration}, which yields the result.

\subsection{Proof of Theorem~\ref{thm:uniform_concentration_ball}}

The theorem follows from applying the following covering-number bound to the result of Theorem~\ref{thm:gradient_concentration}. Consider a positive $\varepsilon>0$, and let $\mathcal{C}_\varepsilon$ be an $\varepsilon$-covering net of $\mathcal{Z}$ in the $\ell_2$ metric. Every $z\in\mathcal{Z}$ has a representative $z'\in\mathcal{C}_\varepsilon$ with $\|z-z'\|_2\le\varepsilon$. For a Euclidean ball of radius $R$, a well-known covering number bound of the $\ell_2$-ball is $|\mathcal{C}_\varepsilon|\le(6R/\varepsilon)^d$.

For a fixed $z\in\mathcal{C}_\varepsilon$, Theorem \ref{thm:gradient_concentration} provides the following pointwise bound to hold with probability at least $1-\eta$:
\[
\|\widehat g_{\mathrm{cross}}(z)-g^*_{\mathrm{cross}}(z)\|_2
\le \frac{4L}{\sqrt{N_r}}\Bigl(1+\sqrt{2\log\tfrac{1}{\eta}}\Bigr)
\]
 Choosing $\eta=\delta/|\mathcal{C}_\varepsilon|$ and applying the union bound across all $z\in\mathcal{C}_\varepsilon$, we conclude that with probability at least $1-\delta$,
\[
\max_{z\in\mathcal{C}_\varepsilon}\|\widehat g_{\mathrm{cross}}(z)-g^*_{\mathrm{cross}}(z)\|_2
\le \frac{4L}{\sqrt{N_r}}\Bigl(1+\sqrt{2\log\tfrac{|\mathcal{C}_\varepsilon|}{\delta}}\Bigr).
\]

To extend this guarantee from the net to the entire ball, consider an arbitrary $z_0\in\mathcal{Z}$ and its nearest covering net point $z'\in\mathcal{C}_\varepsilon$ which by definition satisfies $\|z_0-z'\|_2\le\varepsilon$. The deviation can be decomposed as follows using the triangle inequality:
\begin{align*}
&\bigl\|\widehat g_{\mathrm{cross}}(z_0)-g^*_{\mathrm{cross}}(z_0)\bigr\|_2 \\
\le\; &\bigl\|\widehat g_{\mathrm{cross}}(z_0)-\widehat g_{\mathrm{cross}}(z')\bigr\|_2
+ \bigl\|\widehat g_{\mathrm{cross}}(z')-g^*_{\mathrm{cross}}(z')\bigr\|_2
+ \bigl\|g^*_{\mathrm{cross}}(z')-g^*_{\mathrm{cross}}(z_0)\bigr\|_2.
\end{align*}
The middle term is controlled by the union bound over the net. The first and last terms follows from the Lipschitz constant of $\nabla k$. Thus,
\begin{align*}
\|\widehat g_{\mathrm{cross}}(z_0)-\widehat g_{\mathrm{cross}}(z')\|_2 
=\: &\left\|\frac{2}{N_r}\sum_{j=1}^{N_r}\big(\nabla_{z_0} k(z_0,z_j^{(r)})-\nabla_{z'} k(z',z_j^{(r)})\big)\right\|_2 \\
\le\: &\frac{2}{N_r}\sum_{j=1}^{N_r}L'\|z_0-z'\|_2\\
\le\: & 2L'\varepsilon,
\end{align*}
and similarly $\|g^*_{\mathrm{cross}}(z')-g^*_{\mathrm{cross}}(z_0)\|_2\le 2L'\varepsilon$. Combining these inequalities, we obtain
\[
\bigl\|\widehat g_{\mathrm{cross}}(z_0)-g^*_{\mathrm{cross}}(z_0)\bigr\|_2
\le \frac{4L}{\sqrt{N_r}}\Bigl(1+\sqrt{2\log\tfrac{|\mathcal{C}_\varepsilon|}{\delta}}\Bigr) + 4L'\varepsilon.
\]

Substituting  $|\mathcal{C}_\varepsilon|\le (6R/\varepsilon)^d$ and choosing $\varepsilon=1/\sqrt{N_r}$ completes the bound. The proof is therefore complete.

\subsection{Proof of Corollary~\ref{cor:uniform_gaussian_rbf_ball}}\begin{corollary}\label{cor:uniform_gaussian_rbf_ball}
For the Gaussian RBF kernel $k(x,y)=\exp(-\|x-y\|_2^2/(2\sigma^2))$ on $\mathcal{Z}=\{z:\|z\|_2\le R\}$, with probability at least $1-\delta$,
\[
\sup_{z\in\mathcal{Z}}\|\widehat g_{\mathrm{cross}}(z)-g^*_{\mathrm{cross}}(z)\|_2
\le \frac{24}{\sigma\sqrt{N_r}}
\Bigl(1 + \sqrt{d\log\tfrac{R\sqrt{N_r}}{\sigma}+\log\tfrac{1}{\delta}}\Bigr)
\]
\end{corollary}
\begin{proof}

For Gaussian RBF, $\|\nabla_x k(x,y)\|_2 = \tfrac{\|x-y\|_2}{\sigma^2}k(x,y)\le 1/(\sigma\sqrt{e})$, giving $L=1/(\sigma\sqrt{e})$. The Hessian is
\[
\nabla_x^2 k(x,y) = \Bigl(\tfrac{(x-y)(x-y)^\top}{\sigma^4}-\tfrac{I}{\sigma^2}\Bigr)k(x,y),
\]
whose operator norm is maximized when $\|x-y\|_2^2= \sigma^2$, leading to $\|\nabla_x^2 k(x,y)\|_{\mathrm{op}}\le 2/(\sigma^2\sqrt{e})$. Thus $L'=2/(\sigma^2\sqrt{e})$. Substituting into Theorem \ref{thm:uniform_concentration_ball}'s final inequality gives the bound in the corollary.
\end{proof}
\subsection{Proof of Theorem~\ref{thm:product_concentration}}
\begin{theorem}[Concentration for Weighted Cross Term in Product Kernel]
\label{thm:product_concentration}
Let $\mathcal{P} \subseteq \mathbb{R}^{d_p}$ and $\mathcal{Z} \subseteq \mathbb{R}^{d_z}$ be compact. Let $k_p: \mathcal{P} \times \mathcal{P} \to [-1,1]$ and $k_z: \mathcal{Z} \times \mathcal{Z} \to [-1,1]$ be normalized kernels with $k_p(p,p) = k_z(z,z) = 1$. Assume $k_z$ is $L_z$-Lipschitz continuous in its first argument. Let $Q \times \Pi'$ be the target joint distribution and let $\{(p_j^{(r)}, z_j^{(r)})\}_{j=1}^{N_r} \stackrel{\text{iid}}{\sim} Q \times \Pi'$ be reference samples. For a pair $(p_0, z_0) \in \mathcal{P} \times \mathcal{Z}$, we define
\begin{align*}
g^*_{\text{cross},\otimes}(p_0, z_0) &= -2\mathbb{E}_{(p',z') \sim Q \times \Pi'}[k_p(p_0, p') \nabla_{z_0} k_z(z_0, z')], \\
\widehat{g}_{\text{cross},\otimes}(p_0, z_0) &= -\frac{2}{N_r}\sum_{j=1}^{N_r} k_p(p_0, p_j^{(r)}) \nabla_{z_0} k_z(z_0, z_j^{(r)}).
\end{align*}
Then for any $\delta \in (0,1)$, with probability at least $1-\delta$ over the draw of reference samples:
\begin{equation*}
\big\|\widehat{g}_{\text{cross},\otimes}(p_0, z_0) - g^*_{\text{cross},\otimes}(p_0, z_0)\big\|_2 \leq \frac{4L_z}{\sqrt{N_r}}\left(1 + \sqrt{2\log\bigl({1}/{\delta}\bigr)}\right).
\end{equation*}
\end{theorem}

\begin{proof}

We define the centered random vectors
\begin{equation}
W_j = k_p(p_0, p_j^{(r)}) \nabla_{z_0} k_z(z_0, z_j^{(r)}) - \mathbb{E}_{(p',z') \sim Q \times \Pi'}[k_p(p_0, p') \nabla_{z_0} k_z(z_0, z')], \quad j = 1, \ldots, N_r.
\end{equation}

Since the reference pairs $(p_j^{(r)}, z_j^{(r)})$ are drawn i.i.d. from $Q \times \Pi'$ independently of how $(p_0, z_0)$ was generated, the $W_j$ are independent random vectors in $\mathbb{R}^{d_z}$ with $\mathbb{E}[W_j] = 0$ (by construction)
and $\big\|W_j\big\|_2 \leq 2L_z$ almost surely, since $k_p \in [-1,1]$ and $\big\|\nabla_{z_0} k_z(z_0, \cdot)\big\|_2 \leq L_z$.

Applying the Hoeffding inequality for Hilbert-Schmidth operators \citep{sutherland2018efficient}, we have the following with probability at least $1-\delta$:
\begin{equation}
\left\|\frac{1}{N_r}\sum_{j=1}^{N_r} W_j\right\|_2 \leq \frac{2L_z}{\sqrt{N_r}}\left(1 + \sqrt{2\log\bigl({1}/{\delta}\bigr)}\right).
\end{equation}

Since $\widehat{g}_{\text{cross},\otimes}(p_0, z_0) - g^*_{\text{cross},\otimes}(p_0, z_0) = -\frac{2}{N_r}\sum_{j=1}^{N_r} W_j$, the above inequality will result in the following to hodl with probability at least $1-\delta$:
\begin{equation}
\big\|\widehat{g}_{\text{cross},\otimes}(p_0, z_0) - g^*_{\text{cross},\otimes}(p_0, z_0)\big\|_2 \leq \frac{4L_z}{\sqrt{N_r}}\left(1 + \sqrt{2\log\bigl({1}/{\delta}\bigr)}\right).
\end{equation}
\end{proof}

\begin{corollary}[Product Kernel with Gaussian RBF]
\label{cor:product_gaussian_rbf}
For the product kernel $k_{\otimes}((p,z), (p',z')) = k_p(p, p') \cdot k_z(z, z')$ where $k_z$ is the Gaussian RBF kernel with bandwidth $\sigma$ and $k_p \in [-1,1]$, the weighted cross term satisfies with probability at least $1-\delta$:
\begin{equation}
\big\|\widehat{g}_{\text{cross},\otimes}(p_0, z_0) - g^*_{\text{cross},\otimes}(p_0, z_0)\big\|_2 \leq \frac{3}{\sigma\sqrt{N_r}}\left(1 + \sqrt{2\log\bigl({1}/{\delta}\bigr)}\right).
\end{equation}
\end{corollary}
\begin{proof}

From Corollary~\ref{cor:gaussian_rbf_concise}, the Gaussian RBF kernel has Lipschitz constant $L_z = 1/(\sigma\sqrt{e})$. Applying Theorem~\ref{thm:product_concentration} directly gives the stated bound.
\end{proof}

\section{Algorithmic Descriptions}

Here, we present Algorithm~\ref{alg:mmd_conditional} for prompt-conditioned MMD guidance of diffusion models. We recall that given generated pairs $\{(p_i, z_t^{(i)})\}_{i=1}^B$ with distribution $\widehat{P}_t$ and reference pairs $\{(p_j^{(r)}, z_j^{(r)})\}_{j=1}^{N_r}$with distribution $\widehat{Q}$, the empirical squared MMD in the product space for gradient calculation is:
\begin{align}
\widehat{\text{MMD}}^2_{\otimes}(\widehat{P}_t,\widehat{Q}) &= \frac{1}{B^2}\sum_{i=1}^B\sum_{i'=1}^B k_p(p_i, p_{i'}) k_z(z_t^{(i)}, z_t^{(i')}) + \frac{1}{N_r^2}\sum_{j=1}^{N_r}\sum_{j'=1}^{N_r} k_p(p_j^{(r)}, p_{j'}^{(r)}) k_z(z_j^{(r)}, z_{j'}^{(r)}) \nonumber\\
&\quad - \frac{2}{BN_r}\sum_{i=1}^B\sum_{j=1}^{N_r} k_p(p_i, p_j^{(r)}) k_z(z_t^{(i)}, z_j^{(r)}).
\end{align}

\begin{algorithm}[t]
\caption{Prompt-Aware MMD-Guided Sampling}
\label{alg:mmd_conditional}
\KwIn{Reference pairs $\{(p_j^{(r)}, x_j^{(r)})\}_{j=1}^{N_r}$, prompts $\{\text{prompt}_i\}_{i=1}^B$, guidance schedule $\{\lambda_t\}_{t=1}^T$}
\KwOut{Generated samples $\{x^{(i)}\}_{i=1}^B$ matching prompts and reference style}
\textit{Preprocessing:} \\
\Indp
$z_j^{(r)} \gets \mathcal{E}(x_j^{(r)})$ for all $j \in [N_r]$ \\
$p_i \gets \text{CLIP}_{\text{text}}(\text{prompt}_i)$ for all $i \in [B]$ \\
Compute $K_{ij} = k_p(p_i, p_j^{(r)})$ for all $i \in [B], j \in [N_r]$ \\
\Indm
\textit{Initialization:} $z_T^{(i)} \sim \mathcal{N}(0, I)$ for all $i \in [B]$\;
\For{$t = T$ \KwTo $1$}{
    \For{$i = 1$ \KwTo $B$ \textit{in parallel}}{
        $\widehat{z}_{t-1}^{(i)} \gets \text{sampler}(z_t^{(i)}, t, \epsilon_\theta, p_i)$ \\
        $g^{(i)} \gets \nabla_{z_t^{(i)}} \widehat{\text{MMD}}^2_{\otimes}$ using precomputed $K_{ij}$ \tcp{Using \eqref{eq:product_gradient}}
        $z_{t-1}^{(i)} \gets \widehat{z}_{t-1}^{(i)} - \lambda_t g^{(i)}$\;
    }
}
\Return $\{x^{(i)} = \mathcal{D}(z_0^{(i)})\}_{i=1}^B$
\end{algorithm}

\section{Additional Numerical Results} \label{sec:numerical-appendix}

In this section, we provide additional numerical results for the MMD Guidance method. For all experiments the MMD Guidance scale is chosen via grid search on a held-out validation split using FD.
For all non-text GMM guidance experiments on real images, we used MMD with cubic polynomial kernel.

\begin{figure}
\resizebox{\textwidth}{!}{%
\begin{tikzpicture}
\def\imgwidth{14.3cm}
\def\imgwidthbig{14.0cm}
\node[inner sep=0, outer sep=0] (rowTop) at (9.34, -1.5)
  {\includegraphics[width=\imgwidth]{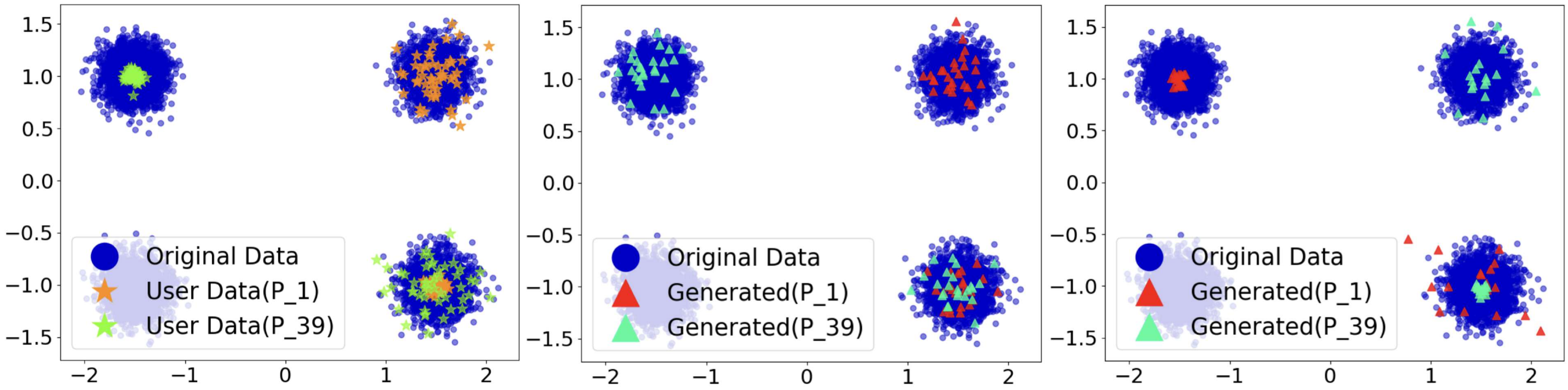}};

\def\colA{0.17}\def\colB{0.50}\def\colC{0.83}

\foreach \x/\dens in {
  \colA/{\headerFont\textbf{User Data}},
  \colB/{\headerFont\textbf{No-Guidance-DM}},
  \colC/{\headerFont\textbf{MMD}}
}{
  \node[
    font=\normalsize, align=center,
    anchor=south,
    fill=white, rounded corners=0.6pt, inner sep=0.9pt
  ] at ($ (rowTop.north west)!\x!(rowTop.north east) + (0,0.18) $)
  { \dens};
}

\end{tikzpicture}
}
\caption{Comparison of MMD Guidance with baselines on 100-D Gaussian distributions, when guiding toward a user with 3 Gaussian components.}
\label{fig:prompt-aware-gmm}
\end{figure}

\textbf{Experiment settings on GMM.} We evaluated the methods listed in Table~\ref{gmm_concat_8_25} in the following setting. We reported results on three users that each has three random components (similar to Figure~\ref{fig:gmmb_3users}). Each component in a user has 50 data points. We applied MMD with an RBF kernel, using bandwidth 1 for the 8-component GMM and bandwidth 4 for the 25-component GMM (since it is more aligned with data from the Gaussian mixtures). The guidance scale is in order of $10^{-1}$.  For the CG baseline, we trained a linear classifier (single FC layer).

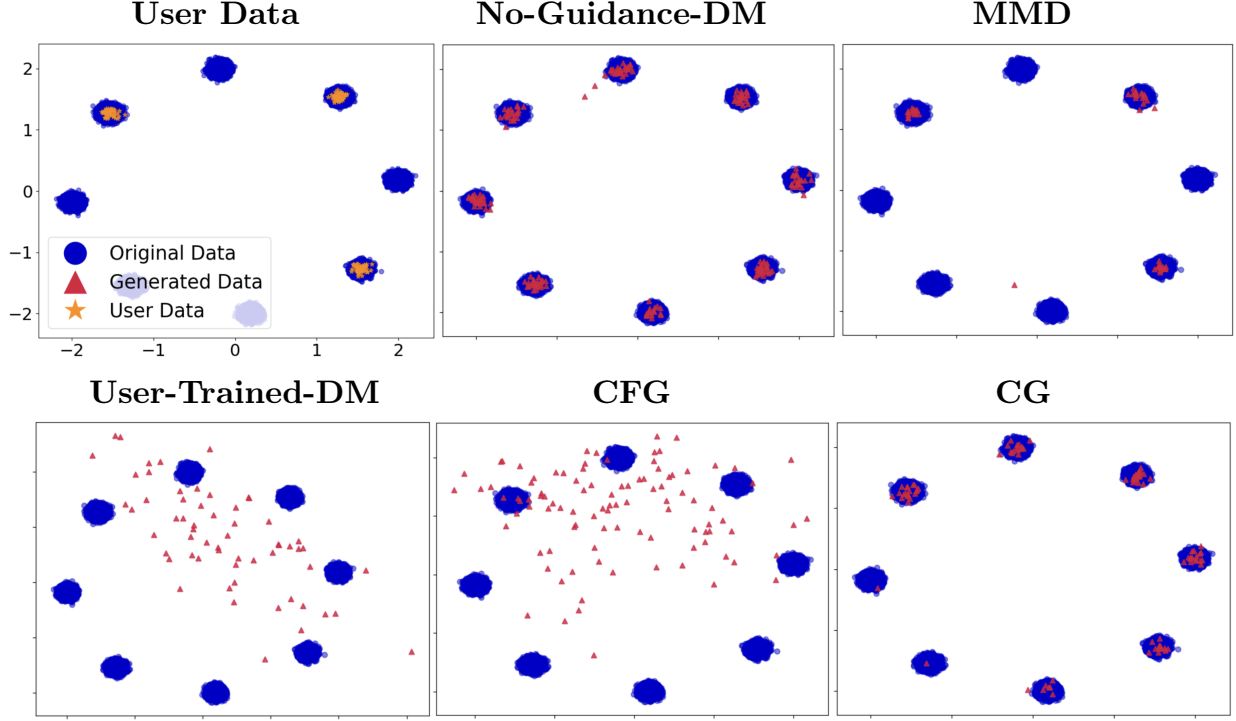
\begin{figure}[!t]
\resizebox{\textwidth}{!}{%
\begin{tikzpicture}
\def\imgwidth{13.3cm}
\def\imgwidthbig{13.0cm}
\node[inner sep=0, outer sep=0] (rowTop) at (9.34, -1.5)
  {\includegraphics[width=\imgwidth]{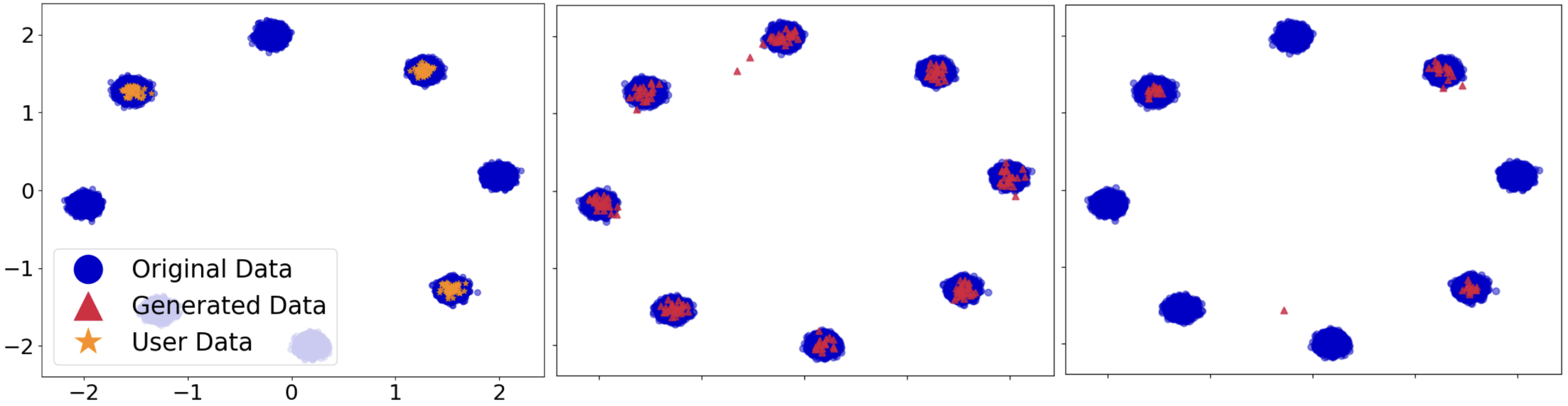}};
\node[inner sep=0, outer sep=0] (rowBot) at (9.45, -5.5)
  {\includegraphics[width=\imgwidthbig]{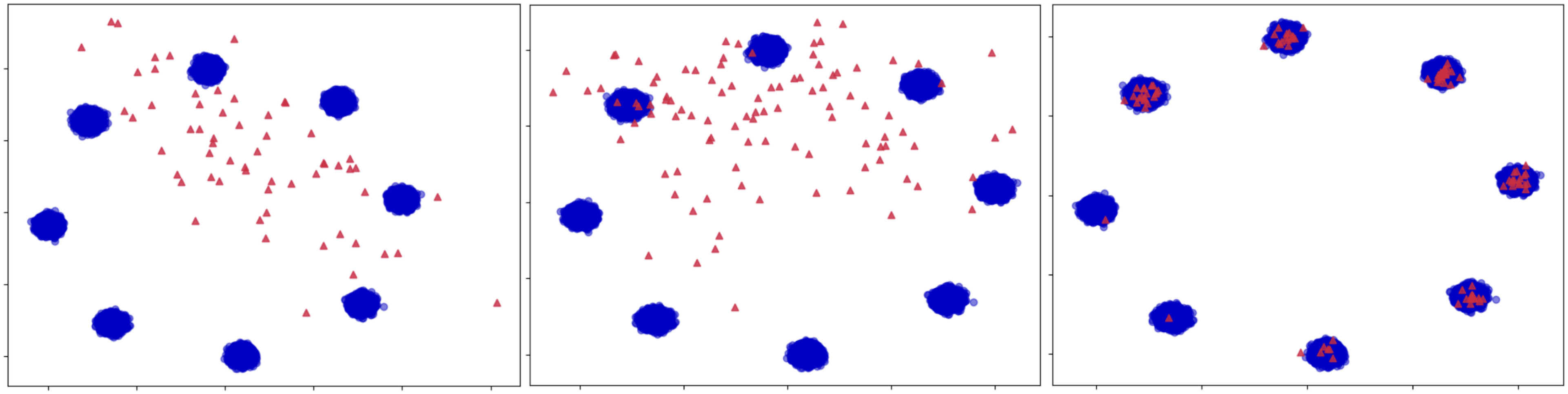}};

\def\colA{0.17}\def\colB{0.50}\def\colC{0.83}

\foreach \x/\dens in {
  \colA/{\headerFont\textbf{User Data}},
  \colB/{\headerFont\textbf{No-Guidance-DM}},
  \colC/{\headerFont\textbf{MMD}}
}{
  \node[
    font=\normalsize, align=center,
    anchor=south,
    fill=white, rounded corners=0pt, inner sep=0pt
  ] at ($ (rowTop.north west)!\x!(rowTop.north east) + (0,0.18) $)
  { \dens};
}

\foreach \x/\dens in {
  \colA/{\headerFont\textbf{User-Trained-DM}},
  \colB/{\headerFont\textbf{CFG}},
  \colC/{\headerFont\textbf{CG}}
}{
  \node[
    font=\normalsize, align=center,
    anchor=south,
    fill=white, rounded corners=0pt, inner sep=0pt
  ] at ($ (rowBot.north west)!\x!(rowBot.north east) + (0,0.18) $)
  { \dens};
}
\end{tikzpicture}
}
\caption{Comparison of MMD Guidance with baselines on 100D Gaussian distributions, when guiding toward a user with 3 Gaussian components.}
\label{fig:gmmb_3users}
\end{figure}

\begin{table}[!t]
\caption{Evaluation metrics for samples generated from three users, using 8, and 25 Gaussian component synthetic data.}
\label{gmm_concat_8_25}
\centering
\setlength{\tabcolsep}{6pt}
\resizebox{\linewidth}{!}{
\begin{tabular}{cccccccc}
\toprule
& & \multicolumn{3}{c}{8 Gaussian Component} & \multicolumn{3}{c}{25 Gaussian Component} \\
\cmidrule(lr){3-5}\cmidrule(lr){6-8}
User & Guidance & FD $\downarrow$ & KD($\times10^{3})\ \downarrow$ & RRKE $\downarrow$
& FD $\downarrow$ & KD($\times10^{3})\ \downarrow$ & RRKE $\downarrow$ \\
\midrule
\multirow{6}{*}{User 1}
  & No-Guidance-DM & $7.04 \pm 0.002$ & $142.42 \pm 0.07$ & $2.314 \pm 0.001$ & $77.70 \pm 0.015$ & $330.91 \pm 0.28$ & $2.073 \pm 0.001$ \\
  & User-Trained-DM & $27.90 \pm 0.009$ & $219.64 \pm 0.14$ & $64.135 \pm 0.103$ & $47.36 \pm 0.026$ & $260.39 \pm 0.35$ & $1.803 \pm 0.000$ \\
  & Fine-tuning & $0.48 \pm 0.001$  & $27.86 \pm 0.10$  & $1.051 \pm 0.001$  & $4.72 \pm 0.004$  & $16.65 \pm 0.11$  & $1.290 \pm 0.000$ \\
  & CG          & $0.94 \pm 0.002$  & $24.13 \pm 0.10$  & $1.219 \pm 0.001$  & $32.68 \pm 0.006$ & $81.31 \pm 0.17$  & $0.398 \pm 0.000$ \\
  & CFG         & $0.97 \pm 0.001$  & $136.06 \pm 0.11$ & $4.833 \pm 0.003$  & $7.79 \pm 0.007$  & $134.12 \pm 0.31$ & $0.270 \pm 0.000$ \\
  & MMD (Ours)  & $0.37 \pm 0.001$  & $17.02 \pm 0.08$  & $1.294 \pm 0.002$  & $3.56 \pm 0.003$  & $9.52 \pm 0.07$   & $0.231 \pm 0.000$ \\
\midrule
\multirow{6}{*}{User 2}
  & No-Guidance-DM & $7.46 \pm 0.004$ & $153.98 \pm 0.16$ & $2.089 \pm 0.003$ & $95.58 \pm 0.034$ & $417.04 \pm 0.65$ & $2.356 \pm 0.000$ \\
  & User-Trained-DM & $32.04 \pm 0.008$ & $217.24 \pm 0.24$ & $76.478 \pm 0.037$ & $88.33 \pm 0.020$ & $260.85 \pm 0.49$ & $2.411 \pm 0.001$ \\
  & Fine-tuning & $1.06 \pm 0.002$  & $40.19 \pm 0.10$  & $1.231 \pm 0.001$  & $33.97 \pm 0.022$ & $135.06 \pm 0.36$ & $1.557 \pm 0.000$ \\
  & CG          & $1.12 \pm 0.001$  & $29.23 \pm 0.09$  & $1.192 \pm 0.002$  & $41.79 \pm 0.015$ & $122.15 \pm 0.38$ & $0.482 \pm 0.000$ \\
  & CFG         & $1.76 \pm 0.002$  & $194.16 \pm 0.24$ & $11.096 \pm 0.009$ & $12.56 \pm 0.015$ & $152.69 \pm 0.43$ & $0.341 \pm 0.000$ \\
  & MMD (Ours)  & $0.33 \pm 0.001$  & $16.54 \pm 0.06$  & $1.245 \pm 0.002$  & $2.84 \pm 0.003$  & $9.00 \pm 0.08$   & $0.246 \pm 0.000$ \\
\midrule
\multirow{6}{*}{User 3}
  & No-Guidance-DM & $8.30 \pm 0.004$ & $167.47 \pm 0.15$ & $2.376 \pm 0.003$ & $87.63 \pm 0.023$ & $378.38 \pm 0.31$ & $2.107 \pm 0.001$ \\
  & User-Trained-DM & $39.00 \pm 0.008$ & $224.98 \pm 0.20$ & $80.293 \pm 0.034$ & $78.56 \pm 0.022$ & $284.06 \pm 0.28$ & $2.362 \pm 0.001$ \\
  & Fine-tuning   & $1.11 \pm 0.002$  & $41.06 \pm 0.07$  & $1.295 \pm 0.002$  & $6.97 \pm 0.008$  & $26.66 \pm 0.14$  & $1.263 \pm 0.000$ \\
  & CG          & $1.24 \pm 0.002$  & $32.59 \pm 0.11$  & $1.110 \pm 0.001$  & $71.95 \pm 0.015$ & $407.94 \pm 0.41$ & $0.531 \pm 0.000$ \\
  & CFG         & $2.20 \pm 0.002$  & $197.11 \pm 0.19$ & $12.831 \pm 0.011$ & $9.49 \pm 0.010$  & $148.67 \pm 0.37$ & $0.257 \pm 0.000$ \\
  & MMD (Ours)  & $0.64 \pm 0.001$  & $14.46 \pm 0.09$  & $1.258 \pm 0.001$  & $1.85 \pm 0.004$  & $7.65 \pm 0.09$   & $0.236 \pm 0.000$ \\
\bottomrule
\end{tabular}
}
\end{table}

\begin{figure}[!t]
\resizebox{\textwidth}{!}{%
\begin{tikzpicture}
\def\imgwidth{13.3cm}
\def\imgwidthbig{13.0cm}
\node[inner sep=0, outer sep=0] (rowTop) at (9.34, -1.5)
  {\includegraphics[width=\imgwidth]{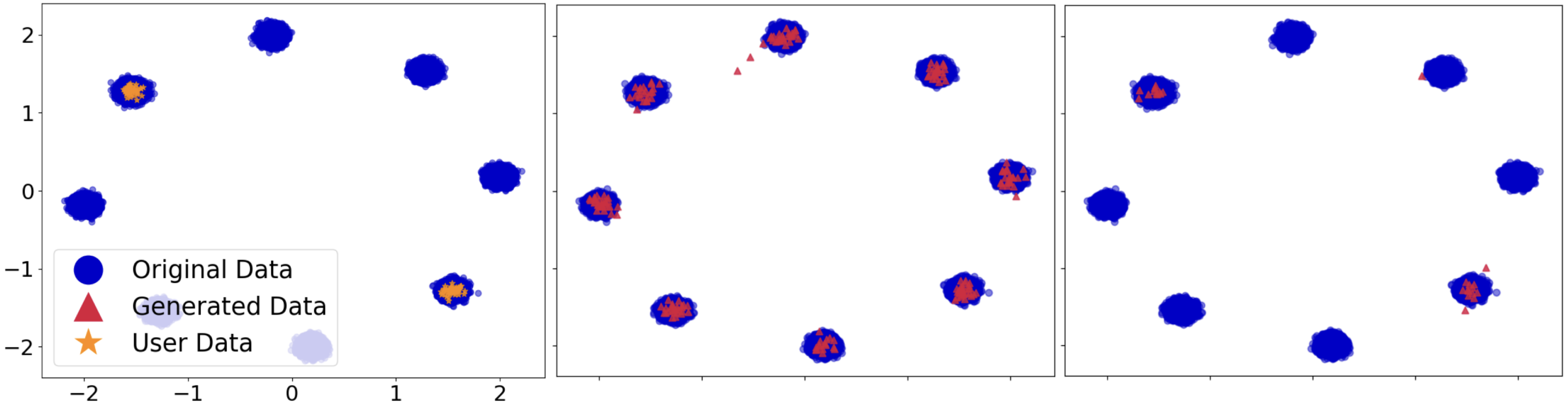}};
\node[inner sep=0, outer sep=0] (rowBot) at (9.45, -5.5)
  {\includegraphics[width=\imgwidthbig]{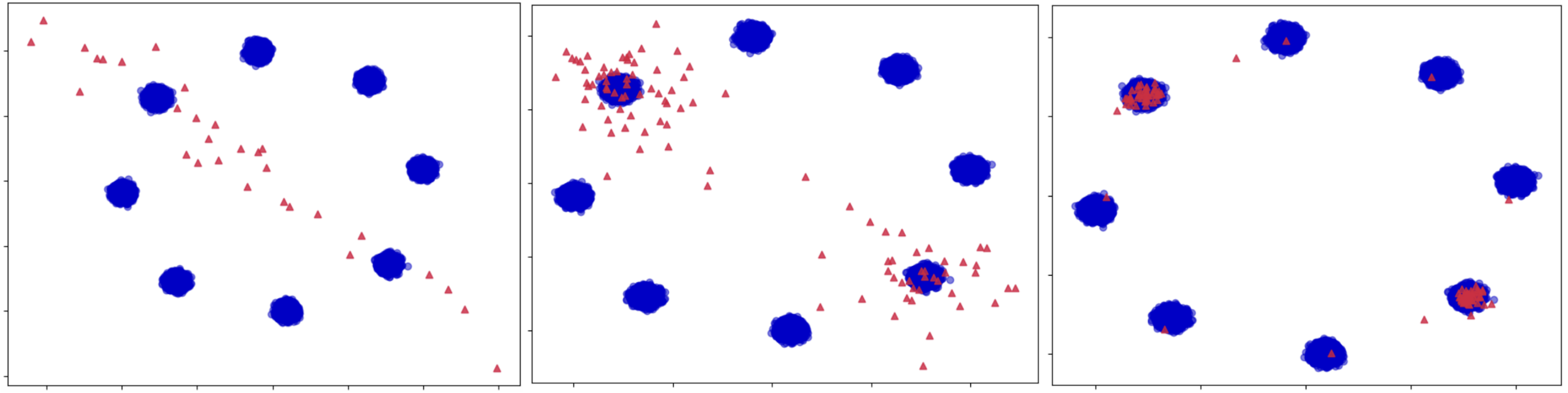}};

\def\colA{0.17}\def\colB{0.50}\def\colC{0.83}

\foreach \x/\dens in {
  \colA/{\headerFont\textbf{User Data}},
  \colB/{\headerFont\textbf{No-Guidance-DM}},
  \colC/{\headerFont\textbf{MMD}}
}{
  \node[
    font=\normalsize, align=center,
    anchor=south,
    fill=white, rounded corners=0.6pt, inner sep=0.9pt
  ] at ($ (rowTop.north west)!\x!(rowTop.north east) + (0,0.18) $)
  { \dens};
}

\foreach \x/\dens in {
  \colA/{\headerFont\textbf{User-Trained-DM}},
  \colB/{\headerFont\textbf{CFG}},
  \colC/{\headerFont\textbf{CG}}
}{
  \node[
    font=\normalsize, align=center,
    anchor=south,
    fill=white, rounded corners=0.6pt, inner sep=0.9pt
  ] at ($ (rowBot.north west)!\x!(rowBot.north east) + (0,0.18) $)
  { \dens};
}
\end{tikzpicture}
}
\caption{Comparison of MMD Guidance with baselines on 100D Gaussian distributions, when guiding toward a user with 2 Gaussian components.}
\label{fig:gmmb_2users}
\end{figure}

\textbf{Additional results on synthetic data.} Similar to Figure~\ref{fig:gmm_baseline} we compare our method to the baselines with different number of components for users, in Figures~\ref{fig:gmmb_3users}, and~\ref{fig:gmmb_2users}. Additionally, we evaluated metrics on three users under 8-component and 25-component GMM settings, computing all metrics on 1,000 generated samples per method. As shown in Table~\ref{gmm_concat_8_25}, across both mixture complexities, the MMD Guidance consistently achieves the lowest FD, KD, and RRKE in most cases. Figure~\ref{gmm_clusters} illustrates the effectiveness of our method in GMM settings of 8 and 25 components.

\textbf{Prompt-aware GMM.} \textbf{Prompt-aware synthetic GMM.}
To simulate the text–style experiment in a controlled setting with known ground truth, we considered a prompt-conditioned GMM, in which the prompt corresponds to the number of the Gaussian component in the GMM.
In this setup, a user provide few-shot references for selected prompts; for each prompt, the user samples share the same component mean but have different variances, simulating style shifts.
As shown in Figure~\ref{fig:prompt-aware-gmm}, the guided sampler allocates mass to the correct components for the queried prompts and respects their local geometry, component variance, and also preserves the target mixture ratios across components.
In comparison, the unguided model only generates samples with respect to the means, ignores variance differences, and does not follow the intended proportions.

Also, we generated the results in Figure~\ref{fig:prompt-aware-gmm} under the following setup. We construct a synthetic user with data–prompt pairs drawn from a four-component GMM. The user distribution includes prompts "one" and "thirty nine", but with a user-specific target variance within the corresponding mixture component. As shown in Figure~\ref{fig:prompt-aware-gmm}, MMD Guidance reproduces both the correct components and the target variance (style), where the No-Guidance-DM model selects the right component but does not match the variance.

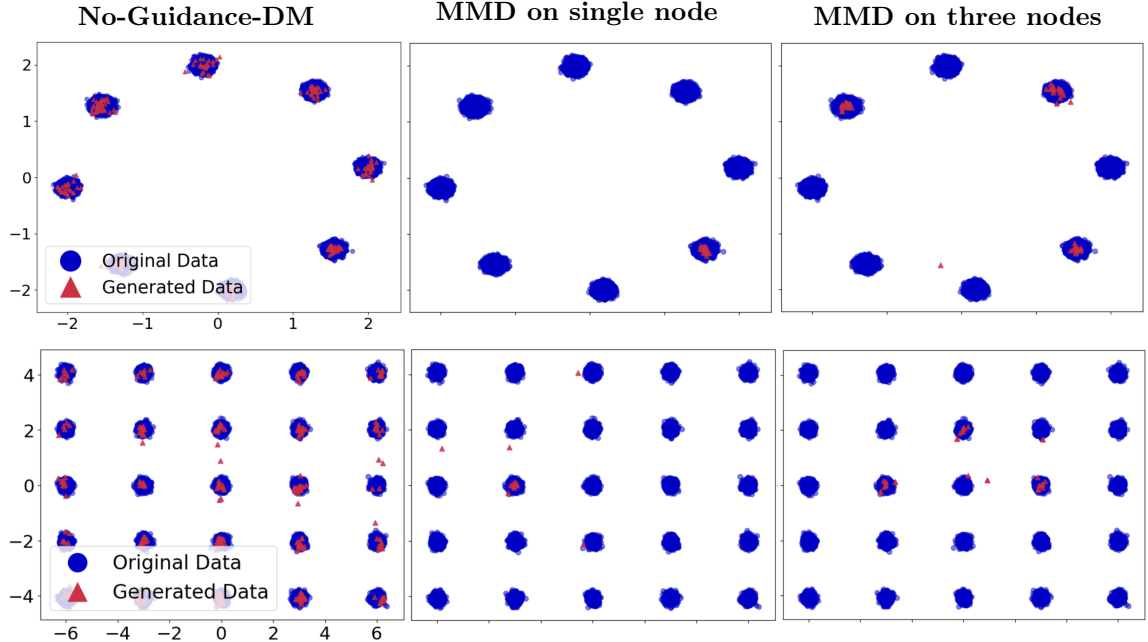
\begin{figure}[!t]
\centering
\begin{tikzpicture}[every node/.style={inner sep=0, outer sep=0}]
\def\imgwidth{1\linewidth}   
\def\rowgap{2mm}                 
\def\labelraise{2mm}             

\node (rowTop) {\includegraphics[width=\imgwidth]{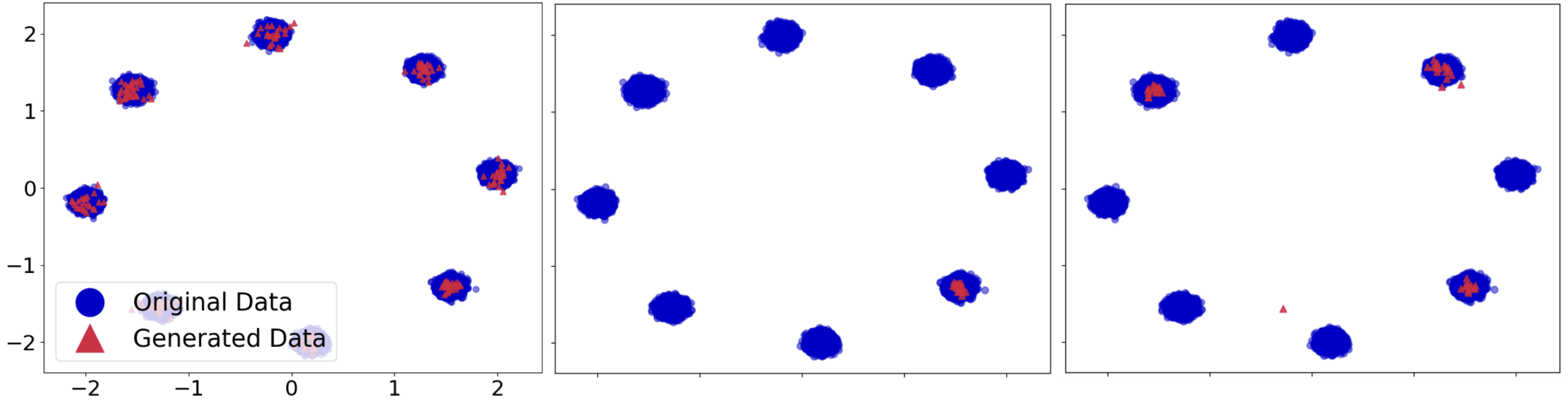}};

\node[below=\rowgap of rowTop] (rowBot)
  {\includegraphics[width=\imgwidth]{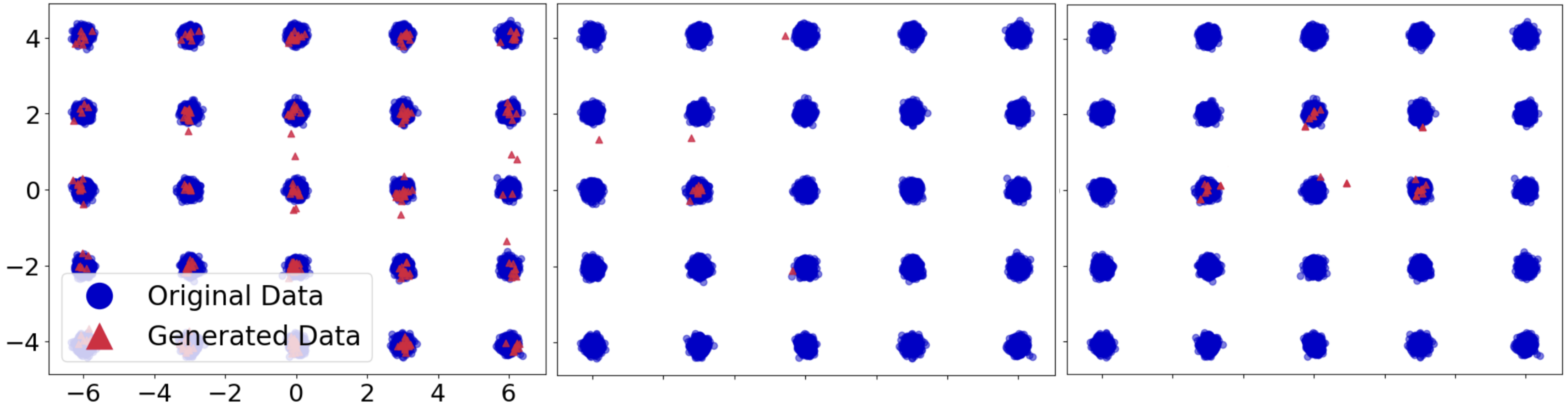}};

\node[anchor=south, yshift=\labelraise]
  at ($(rowTop.north west)!0.1667!(rowTop.north east)$) {\textbf{No-Guidance-DM}};
\node[anchor=south, yshift=\labelraise]
  at ($(rowTop.north west)!0.5!(rowTop.north east)$)    {\textbf{MMD on single node}};
\node[anchor=south, yshift=\labelraise]
  at ($(rowTop.north west)!0.8333!(rowTop.north east)$) {\textbf{MMD on three nodes}};
\end{tikzpicture}

\caption{Effect of MMD on Gaussian mixture models with 8 and 25 components.
The first row shows results for 8 components; the second row shows 25 components.
Column labels denote unguided diffusion and two MMD guidance targets.}
\label{gmm_clusters}
\end{figure}







\textbf{Experiment settings on FFHQ.}
We evaluated the methods listed in Table~\ref{ffhq-metric-dinov2} using the following setup: Each user contains 500 images from the FFHQ dataset.
(i) a user who has images of people wearing sunglasses;
(ii) a user whose images predominantly feature people wearing reading-glasses. The user images were passed through the LDM AutoEncoder to construct the features used for the MMD Guidance in latent space.  In FFHQ experiments, we generated samples using pretrained LDM weights and a DDPM sampler. The MMD Guidance scale is in order of $10^{-4}$, and all metrics are computed on 500 generated samples per method. For the CG baseline, we trained a classifier with an Inception v3\cite{szegedy2016rethinking} backbone and a linear classification head.

\textbf{Additional results on FFHQ.} We also compare our method to baselines that require additional user-specific training (User-Trained-DM and fine-tuning). Table~\ref{ffhq-localft} shows that, across both users, the MMD Guidance outperforms these baselines. Beyond superior results, MMD Guidance is substantially more efficient in both time and compute, as it does not require any per-user training. In practice, the additional cost is negligible compared to unguided sampling (more details in Runtime Analysis). Furthermore, since these baselines are trained on a very small per-user subset (500 images), the diffusion model is prone to overfitting; in particular, the User-Trained-DM baseline shows clear memorization of the training data, as illustrated in Figure~\ref{ffhq_appendix}.

\begin{table}[!t]
    \caption{Evaluation metrics for samples generated on FFHQ for local training and fine-tuning.}
    \label{ffhq-localft}
    \centering
    \resizebox{\linewidth}{!}{
    \begin{tabular}{ccccccc}
    \toprule
    User & Guidance & FD $\downarrow$ & KD $\downarrow$ & RRKE $\downarrow$ & Density ($\times 10 ^ 2$) $\uparrow$ & Coverage ($\times 10 ^ 2$) $\uparrow$ \\ 
    \midrule
    \multirow{3}{*}{Sunglasses} & User-Trained-DM & $1004.71 \pm 51.23$ & $8.50 \pm 0.45$ & $1.52 \pm 0.03$ & $43.43 \pm 5.70$ & $57.96 \pm 1.60$ \\
    & Fine-tuning & $747.69 \pm 59.18$ & $4.15 \pm 0.52$ & $1.40 \pm 0.02$ & $71.64 \pm 4.66$ & $71.12 \pm 2.85$ \\ 
    & MMD (Ours) & $692.87 \pm 30.43$ & $3.25 \pm 0.18$ & $1.39 \pm 0.04$ & $113.13 \pm 9.36$ & $79.08 \pm 1.54$ \\
    \midrule
    \multirow{3}{*}{Reading-glasses} & User-Trained-DM & $1105.12 \pm 53.03$ & $7.54 \pm 0.29$ & $1.47 \pm 0.01$ & $59.14 \pm 4.18$ & $57.52 \pm 2.91$ \\
    & Fine-tuning & $732.91 \pm 43.02$ & $2.99 \pm 0.20$ & $1.35 \pm 0.01$ & $82.32 \pm 4.47$ & $73.40 \pm 1.89$ \\
    & MMD (Ours) & $574.29 \pm 17.57$ & $1.39 \pm 0.05$ & $1.30 \pm 0.01$ & $87.10 \pm 9.69$ & $84.60 \pm 2.54$ \\
    \bottomrule
    \end{tabular}
    }
\end{table}

\begin{figure*}
    \centering
    \includegraphics[width=\linewidth]{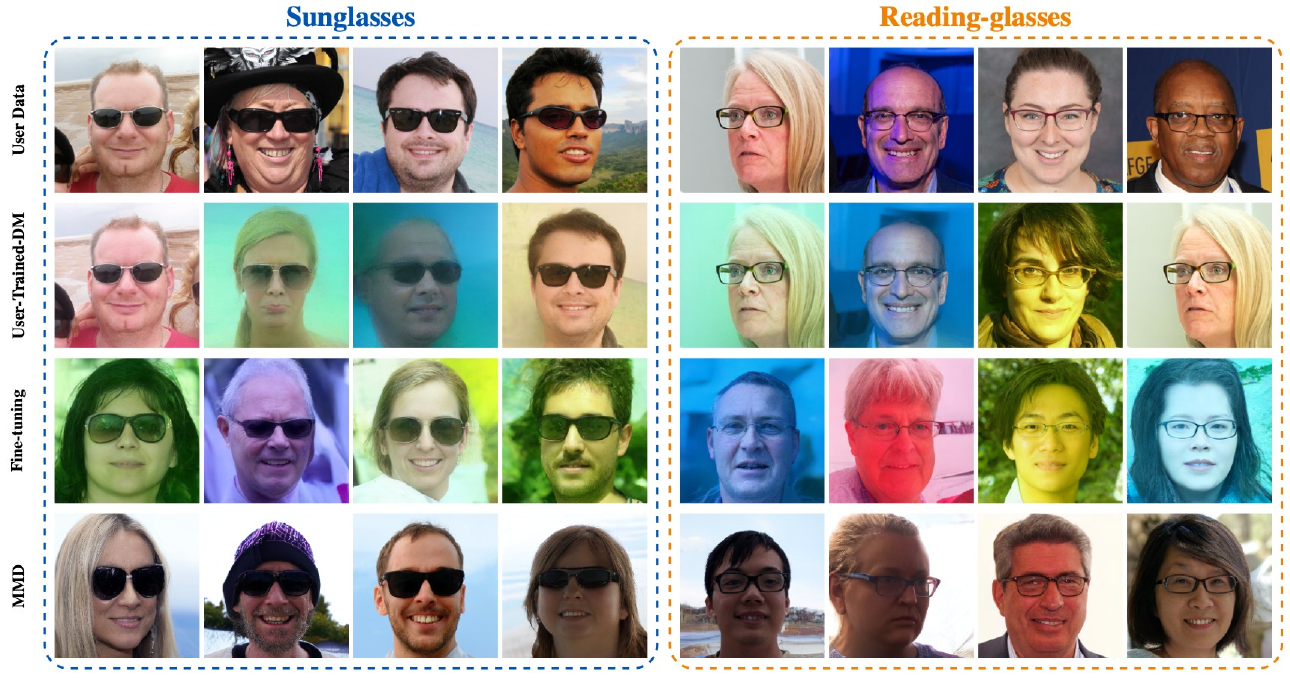}
    \caption{Comparison of MMD Guidance with training baselines on FFHQ dataset.}
    \label{ffhq_appendix}
\end{figure*}

\textbf{Experiment settings on CelebA-HQ.} We conducted experiments on two users derived from a CelebA-HQ subset to capture specific styles:
one with 500 images of people with blond hair and one with 500 images of people with black hair. The user images were passed through the LDM AutoEncoder to construct the features used for the MMD Guidance on latent space. For CelebA-HQ experiments, we generated samples using pretrained LDM weights and a
DDPM sampler. The MMD Guidance scale is in order of $10^{-5}$, and all metrics are computed on 200 generated samples per method. For the CG baseline, we trained a classifier with an Inception v3~\cite{szegedy2016rethinking} backbone and a linear classification head.

\textbf{CelebA-HQ results.} We computed the evaluation metrics on CelebA-HQ and compared them with No-Guidance-DM, CG, fine-tuning, and User-Trained-DM in Table~\ref{tab:celeba-metrics}.
Across both user groups, the MMD Guidance achieves the strongest overall performance. 
Similarly to FFHQ, the limited per-user data (500 images) leads the User-Trained-DM baseline to overfit (memorize user data). Qualitative comparisons are also provided in Figure~\ref{celeba_all}.

\begin{table}[!t]
    \caption{Comparison metrics for samples generated on CelebA-HQ using MMD versus the baselines.}
    \label{tab:celeba-metrics}
    \centering
    \resizebox{\linewidth}{!}{
    \begin{tabular}{c c c c c c c}
    \toprule
    User & Guidance & FD $\downarrow$ & KD $\downarrow$ & RRKE $\downarrow$ & Density ($\times 10^{2}$) $\uparrow$ & Coverage ($\times 10^{2}$) $\uparrow$ \\
    \midrule
    \multirow{5}{*}{Black-haired}
      & No-Guidance-DM & $665.99 \pm 22.26$ & $2.61 \pm 0.38$ & $1.32 \pm 0.01$ & $94.68 \pm 4.55$ & $83.62 \pm 6.96$ \\
      & User-Trained-DM  & $921.55 \pm 18.76$ & $5.70 \pm 0.39$ & $1.41 \pm 0.01$ & $86.16 \pm 7.95$ & $70.09 \pm 3.31$ \\
      & Fine-tuning   & $817.17 \pm 14.97$ & $4.86 \pm 0.40$ & $1.37 \pm 0.01$ & $89.38 \pm 1.26$ & $75.26 \pm 4.40$ \\
      & CG            & $628.59 \pm 14.52$ & $2.18 \pm 0.30$ & $1.30 \pm 0.01$ & $108.02 \pm 5.76$ & $91.14 \pm 2.42$ \\
      & MMD (Ours)    & $627.57 \pm 18.20$ & $2.06 \pm 0.28$ & $1.30 \pm 0.01$ & $104.64 \pm 4.54$ & $90.75 \pm 2.95$ \\
    \midrule
    \multirow{5}{*}{Blond-haired}
      & No-Guidance-DM & $668.17 \pm 24.25$ & $2.99 \pm 0.23$ & $1.44 \pm 0.02$ & $66.82 \pm 9.52$ & $82.69 \pm 4.42$ \\
      & User-Trained-DM         & $779.93 \pm 14.70$ & $4.75 \pm 0.21$ & $1.49 \pm 0.01$ & $60.30 \pm 8.27$ & $66.39 \pm 3.40$ \\
      & Fine-tuning   & $640.63 \pm 13.07$ & $4.04 \pm 0.07$ & $1.40 \pm 0.01$ & $89.50 \pm 11.32$ & $69.56 \pm 4.77$ \\
      & CG            & $663.36 \pm 24.96$ & $2.88 \pm 0.23$ & $1.44 \pm 0.02$ & $67.58 \pm 10.14$ & $85.05 \pm 5.16$ \\
      & MMD (Ours)    & $637.31 \pm 27.12$ & $2.28 \pm 0.20$ & $1.41 \pm 0.02$ & $80.98 \pm 11.56$ & $89.72 \pm 3.50$ \\
    \bottomrule
    \end{tabular}}
\end{table}

\begin{figure*}[t]
    \centering
    \includegraphics[width=\linewidth]{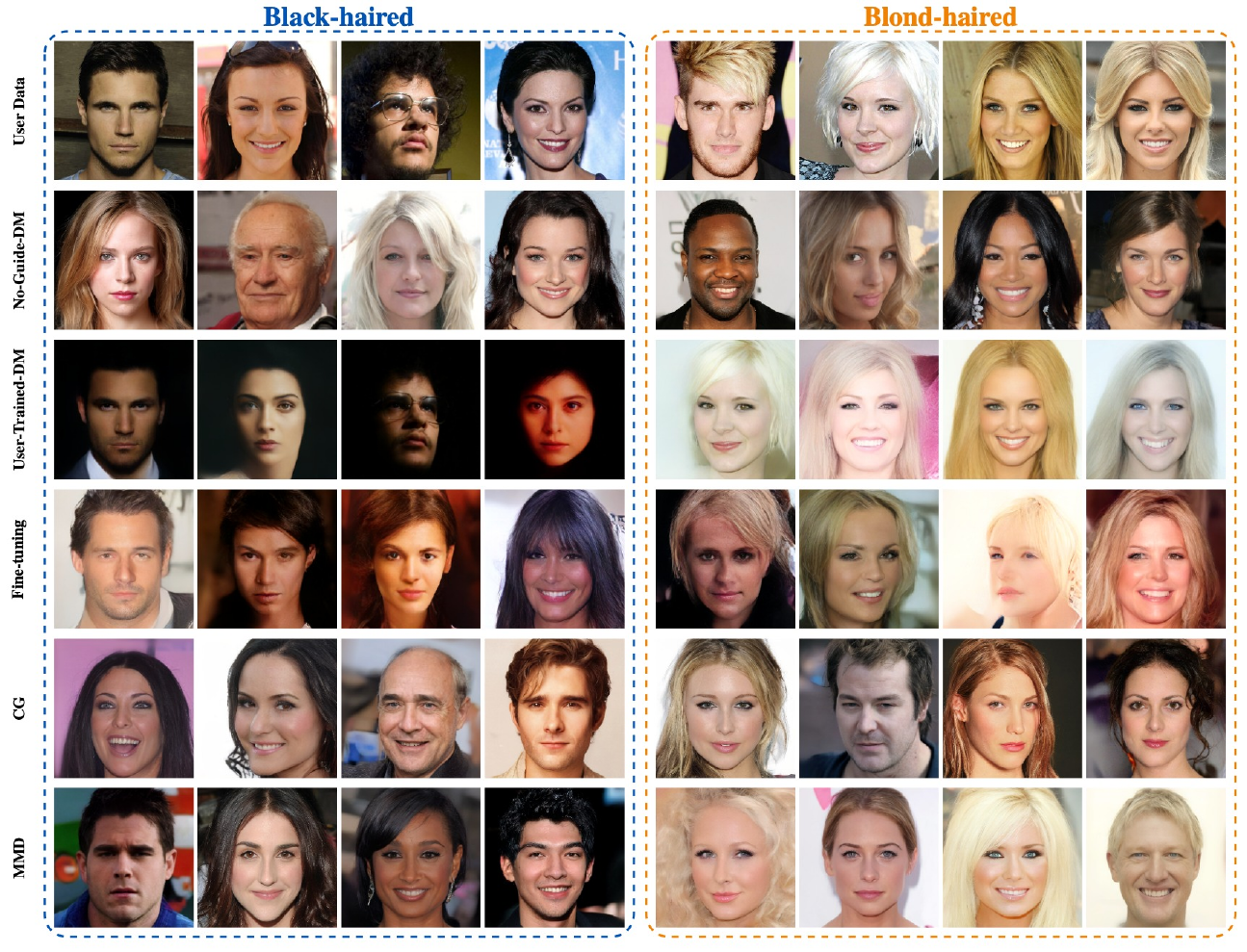}
    \caption{Comparison of MMD Guidance and baselines on CelebA-HQ dataset.}
    \label{celeba_all}
\end{figure*}

\textbf{Experiment settings on Stable Diffusion v1.4.}
We evaluate the methods in Table~\ref{bike_metrics} under the following setting. 
We use Stable Diffusion v1.4 at $512{\times}512$ resolution with 50 inference steps. 
For each image style (winter, black-and-white, sketch, cartoon), we first generate 200 images using the prompts 
"bicycle in winter time", "black and white photo of a bicycle", "sketch of a bicycle", and "cartoonish bicycle", with a fixed classifier-free guidance (CFG) scale of 5. 
During MMD Guidance, we use the prompt "bicycle". 
For SD v1.4 experiments, the MMD Guidance scale is on the order of $10^{-2}$, and all metrics are computed on 500 generated samples per method. 
The same setting is used for the results in Table~\ref{car_metrics}, except that all prompts (for both the user construction and guided generation) use "car" instead of "bicycle".
\begin{table}[!t]
    \caption{Evaluated metrics for \textbf{Bike}. Each block shows a user (Cartoon, Winter, B/W, and Sketch) with No-Guidance-DM vs.\ MMD Guidance.}
    \label{bike_metrics}
    \centering
    \resizebox{\linewidth}{!}{
    \begin{tabular}{cccccccc}
    \toprule
    User & Guidance
         & FD $\downarrow$
         & KD $\downarrow$
         & RRKE $\downarrow$
         & Density ($\times10^{2})$ $\uparrow$
         & Coverage ($\times10^{2})$ $\uparrow$ \\
    \midrule
    \multirow{2}{*}{B/W}
      & No-Guidance-DM &  $354.30 \pm 8.76$ &  $3.83 \pm 0.19$ & $1.43 \pm 0.02$ & $43.56 \pm 3.57$ & $69.18 \pm 2.42$ \\
      & MMD (Ours)     &  $241.40 \pm 9.30$ &  $1.96 \pm 0.13$ & $1.32 \pm 0.01$ & $63.62 \pm 4.04$ & $82.22 \pm 2.98$ \\
    \midrule
    \multirow{2}{*}{Winter}
      & No-Guidance-DM & $1147.48 \pm 33.88$ & $12.60 \pm 0.38$ & $1.84 \pm 0.03$ & $9.28 \pm 4.50$ &  $6.45 \pm 0.63$ \\
      & MMD (Ours)     & $1140.07 \pm 33.84$ & $12.99 \pm 0.37$ & $1.82 \pm 0.02$ & $12.32 \pm 6.31$ &  $6.73 \pm 0.88$ \\
    \midrule
    \multirow{2}{*}{Cartoon}
      & No-Guidance-DM & $1583.38 \pm 9.88$  & $22.03 \pm 0.27$ & $2.56 \pm 0.04$ & $53.63 \pm 9.26$ & $6.06 \pm 0.67$ \\
      & MMD (Ours)     & $1516.45 \pm 11.97$ & $21.43 \pm 0.27$ & $2.54 \pm 0.04$ & $46.99 \pm 10.93$ & $8.08 \pm 0.65$ \\
    \midrule
    \multirow{2}{*}{Sketch}
      & No-Guidance-DM &  $949.10 \pm 7.68$  &  $17.40 \pm 0.30$ & $2.53 \pm 0.06$ & $0.02 \pm 0.04$ & $0.12 \pm 0.18$ \\
      & MMD (Ours)     &  $703.12 \pm 7.35$  &  $12.42 \pm 0.24$ & $1.99 \pm 0.03$ & $2.62 \pm 0.71$ & $9.57 \pm 2.32$ \\
    \bottomrule
    \end{tabular}
    }
\end{table}


\begin{table}[!t]
    \caption{Evaluated metrics for \textbf{Car}. Each block shows a user (Cartoon, Winter, B/W, and Sketch) with No-Guidance-DM vs.\ MMD Guidance.}
    \label{car_metrics}
    \centering
    \resizebox{\linewidth}{!}{
    \begin{tabular}{ccccccc}
    \toprule
    User & Guidance
         & FD $\downarrow$
         & KD $\downarrow$
         & RRKE $\downarrow$
         & Density ($\times10^{2})\uparrow$
         & Coverage ($\times10^{2})\uparrow$ \\
    \midrule
    \multirow{2}{*}{B/W}
      & No-Guidance-DM & $1487.97 \pm 47.90$ & $9.07 \pm 0.60$ & $1.74 \pm 0.05$ & $59.41 \pm 4.56$ & $50.02 \pm 3.42$ \\
      & MMD (Ours)     & $1242.00 \pm 46.46$ & $6.94 \pm 0.56$ & $1.62 \pm 0.04$ & $55.76 \pm 3.98$ & $55.62 \pm 2.27$ \\
    \midrule
    \multirow{2}{*}{Winter}
      & No-Guidance-DM & $1394.50 \pm 33.74$ & $8.41 \pm 0.39$ & $1.78 \pm 0.03$ & $11.04 \pm 2.74$ & $11.82 \pm 2.01$ \\
      & MMD (Ours)     & $1127.11 \pm 28.07$ & $6.82 \pm 0.37$ & $1.62 \pm 0.03$ & $15.92 \pm 1.52$ & $20.88 \pm 1.30$ \\
    \midrule
    \multirow{2}{*}{Cartoon}
      & No-Guidance-DM & $2222.78 \pm 13.54$ & $16.80 \pm 0.15$ & $3.35 \pm 0.05$ & $1.72 \pm 1.09$ & $1.36 \pm 0.56$ \\
      & MMD (Ours)     & $1826.11 \pm 14.44$ & $14.00 \pm 0.18$ & $2.74 \pm 0.04$ & $2.75 \pm 1.15$ & $4.06 \pm 0.86$ \\
    \midrule
    \multirow{2}{*}{Sketch}
      & No-Guidance-DM & $1268.91 \pm 23.54$ & $8.77 \pm 0.25$ & $1.75 \pm 0.03$ & $4.88 \pm 1.30$ & $5.80 \pm 0.861$ \\
      & MMD (Ours)     & $1207.64 \pm 26.41$ & $8.29 \pm 0.24$ & $1.71 \pm 0.03$ & $7.72 \pm 2.12$ & $9.10 \pm 1.42$ \\
    \bottomrule
    \end{tabular}
    }
\end{table}

\textbf{Ablation study on MMD Guidance scale and MMD kernel.}
Figures~\ref{fig:poly_ablation}, \ref{fig:rbf_ablation}, and \ref{fig:rbf_vs_poly_ablation} analyze the effect of the kernel choice and the MMD Guidance scale $\alpha$ on FD and KD. The comparisons use black-and-white images from the bike-user experiments with Stable Diffusion v1.4. It is evident from these results that polynomial kernels ($d \in \{2,3,4\}$), and RBF kernels ($\sigma \in \{1.25,1.5,2\}$) follow very similar FD/KD curves, and the optimal $\alpha$ appears in the same range. At the best $\alpha$, polynomial $d{=}3$ is marginally better on metrics than the RBF choices, but the gap is very small. 
Moreover, across all settings, performance improves as $\alpha$ increases from $0$ up to a small range and then starts to drop. However, the overall results shows that the MMD Guidance is not highly dependent on the choice of guidance parameters or kernel.

\begin{figure}[t]
    \centering    \includegraphics[width=0.94\linewidth]{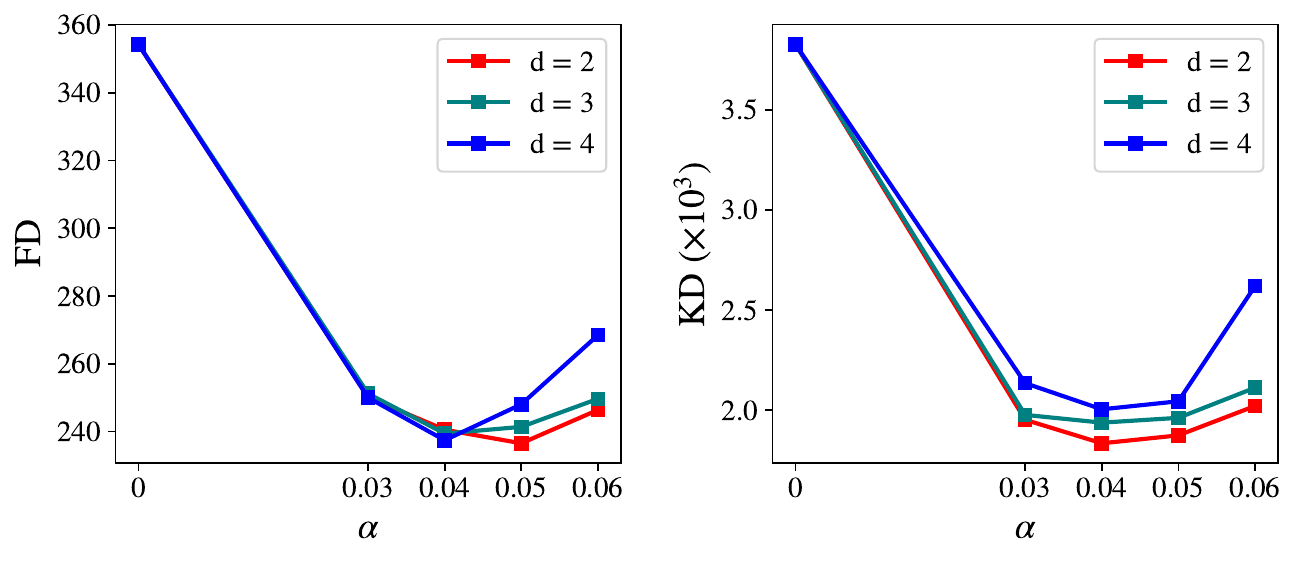}
    \caption{Effect of MMD Guidance scale $\alpha$ with polynomial kernels of degree $d \in \{2,3,4\}$ on FD and KD metrics.}
    \label{fig:poly_ablation}
\end{figure}

\begin{figure}[t]
    \centering    \includegraphics[width=0.94\linewidth]{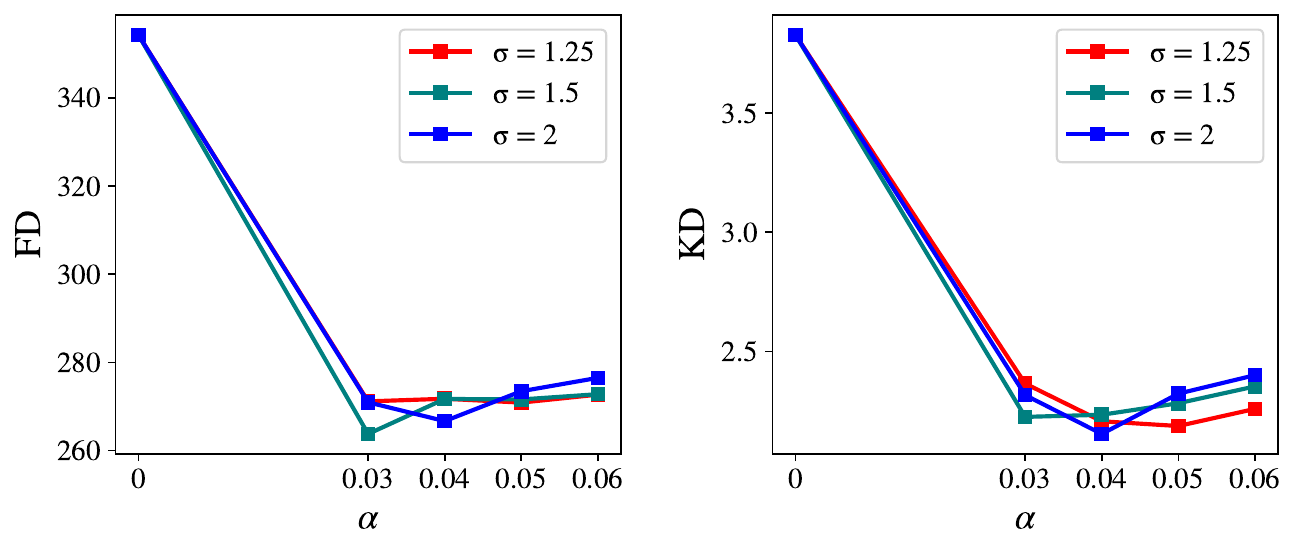}
    \caption{Effect of MMD Guidance scale $\alpha$ with RBF kernels of $\sigma \in \{1.25, 1.5, 2\}$ on FD and KD metrics.}
    \label{fig:rbf_ablation}
\end{figure}

\begin{figure}[t]
    \centering    \includegraphics[width=0.94\linewidth]{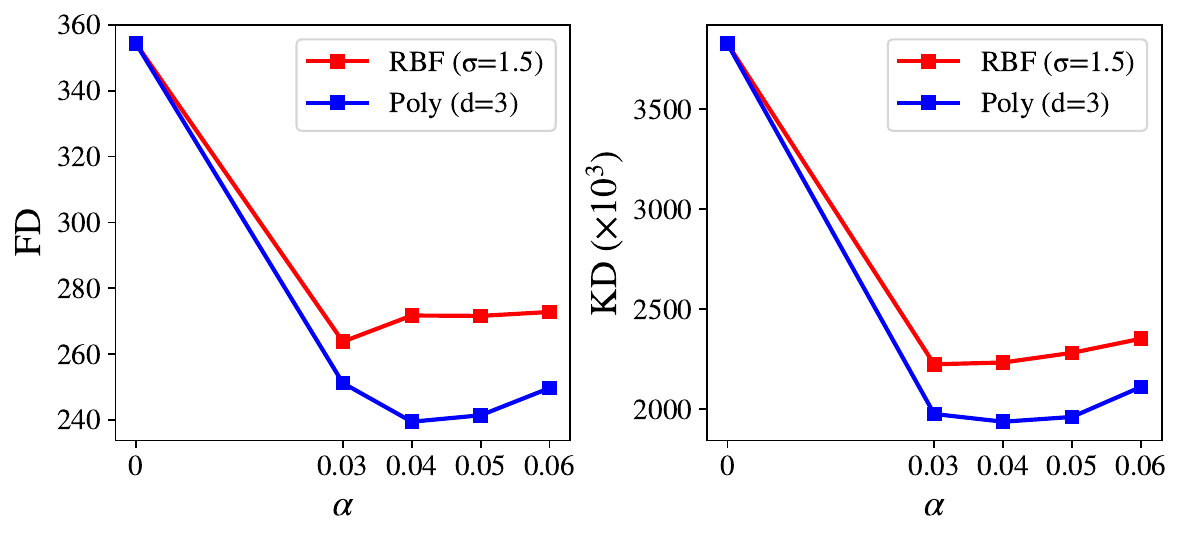}
    \caption{Comparison of RBF ($\sigma=1.5$) vs.\ Polynomial (degree $d=3$) under varying MMD Guidance scale $\alpha$ on FD and KD metrics.}
    \label{fig:rbf_vs_poly_ablation}
\end{figure}


\textbf{Ambient space MMD Guidance.} The previous experiments apply MMD Guidance in the latent space, we additionally applied MMD Guidance on ambient space using features, extracted from the Inception v3~\citep{szegedy2016rethinking} and computing MMD in that feature space. For these experiments we used AFHQ datasets.

\textbf{Experiment settings on AFHQ.} We evaluated the metrics listed in Table~\ref{afhq-prdc}. We conducted experiments on three users derived from AFHQ classes, each with 500 images of dogs, cats, and wild.
The user images were passed through the Inception v3 to construct the features used for the MMD Guidance on ambient space. Figures~\ref{afhq-qual} depict extra qualitative results for the MMD Guidance for these users. For AFHQ experiments, we generated samples using weights from~\citet{daras2023consistent}. The MMD Guidance scale is in order of $10^{-1}$, and all metrics are computed on 500 generated samples per method. Table~\ref{afhq-prdc} compares the evaluation metrics with the No-Guidance-DM baseline to assess the effectiveness of our method.

\textbf{Effect of the number of reference samples on the MMD Guidance.}
To assess how the number of available reference sets influences the effectiveness of MMD, we evaluated our method under varying sample sizes for FFHQ and GMM datasets.

FFHQ setting: We subsample the sunglasses reference set to ${10, 50, 150, 250, 400, 500}$
examples and compute all evaluation metrics. As expected, table~\ref{tab:sunglasses_ref_samples} depicts that the performance improves significantly in compare to no guidance after only 50 reference samples and it converges shortly after.

\begin{table}[!t]
\caption{Evaluation metrics as a function of the number of reference samples for the Sunglasses user of FFHQ dataset.}
\label{tab:sunglasses_ref_samples}
\centering
\begin{tabular}{lcccc}
\toprule
\# of ref. samples & Guidance & FD $\downarrow$ & KD $\downarrow$ & RRKE $\downarrow$ \\
\midrule
$n=0$   & No-Guidance-DM & 1004.71 & 8.21 & 1.73 \\
\midrule
$n=10$  & MMD & 886.41 & 5.67 & 1.55 \\
$n=50$  & MMD & 736.69 & 3.78 & 1.44 \\
$n=150$ & MMD & 719.37 & 3.59 & 1.41 \\
$n=250$ & MMD & 710.48 & 3.33 & 1.41 \\
$n=400$ & MMD & 696.07 & 3.29 & 1.38 \\
$n=500$ & MMD & 692.87 & 3.25 & 1.39 \\
\bottomrule
\end{tabular}
\end{table}

GMM setting: We study the effect of the number of reference samples using a synthetic GMM dataset with both 8, and 25 components. The reference set is drawn from 3 randomly selected components and further subsampled to obtain reference sets of sizes 
${5,25,50,100,150}$. We compute all evaluation metrics for each setting. As shown in Table~\ref{tab:gmm_ref_samples}, the performance converges rapidly, with little to no improvement beyond 50 reference samples.

\begin{table}[!t]
\caption{Evaluation metrics as a function of the number of reference samples for 8- and 25-component Gaussian mixtures.}
\label{tab:gmm_ref_samples}
\centering
\setlength{\tabcolsep}{8pt}
\resizebox{\linewidth}{!}{
\begin{tabular}{lccccccc}
\toprule
& & \multicolumn{3}{c}{8 Gaussian Component} & \multicolumn{3}{c}{25 Gaussian Component} \\
\cmidrule(lr){3-5}\cmidrule(lr){6-8}
\# Ref. Samples & Guidance 
& FD $\downarrow$ & KD($\times10^{3}$)$\downarrow$ & RRKE $\downarrow$
& FD $\downarrow$ & KD($\times10^{3}$)$\downarrow$ & RRKE $\downarrow$ \\
\midrule
$n=0$   & No-Guidance-DM 
    & 7.04  & 142.42 & 2.31 
    & 77.70 & 330.91 & 2.073 \\
\midrule
$n = 5$   & MMD 
    & 4.62 & 65.17 & 1.45
    & 29.63 & 79.72 & 0.81 \\
$n=25$  & MMD 
    & 0.49 & 29.72 & 1.33
    & 4.93 & 22.66 & 0.33 \\
$n=50$  & MMD 
    & 0.38 & 24.51 & 1.30
    & 3.52 & 9.49  & 0.23 \\
$n=100$ & MMD 
    & 0.34 & 23.33 & 1.22
    & 3.69 & 11.25 & 0.25 \\
$n = 150$ & MMD 
    & 0.37 & 17.02 & 1.29
    & 3.56 & 9.52  & 0.23 \\
\bottomrule
\end{tabular}}
\end{table}

\begin{table}[!t]
\caption{Evaluation metrics as a function of the number of reference samples.}
\label{tab:ref_samples_m}
\centering
\begin{tabular}{lccc}
\toprule
\# of ref. samples & FD $\downarrow$ & KD $\downarrow$ & RRKE $\downarrow$ \\
\midrule
$n = 10$  & 1715.86 & 1.5944 & 1.98 \\
$n = 50$  & 1709.59 & 1.5498 & 1.92 \\
$n = 100$ & 1693.35 & 1.5082 & 1.81 \\
$n = 150$ & 1681.33 & 1.4887 & 1.79 \\
$n = 200$ & 1679.69 & 1.4534 & 1.76 \\
\bottomrule
\end{tabular}
\end{table}

\textbf{Applying MMD guidance and denoising process separately.}
We conducted an experiment in which MMD guidance was applied separately from the denoising process. We performed five MMD-guidance steps every ten diffusion timesteps. As shown in Table~\ref{tab:sdxl_mmd_variants}, this modification yielded improved MMD-guidance performance compared with the original simultaneous-guidance approach.

\begin{table}[!t]
\caption{Comparison of SDXL variants with and without MMD guidance.}
\label{tab:sdxl_mmd_variants}
\centering
\begin{tabular}{lccc}
\toprule
Model & FD $\downarrow$ & KD $\downarrow$ & RRKE $\downarrow$ \\
\midrule
SDXL (No-Guidance-DM)                & 1734.25 & 1.5536 & 1.84 \\
SDXL + MMD Guidance (clean reference)   & 1678.24 & 1.4429 & 1.76 \\
SDXL + MMD Guidance (separate steps)   & 1648.95 & 1.4405 & 1.62 \\
\bottomrule
\end{tabular}
\end{table}

\textbf{Computing MMD guidance with noisy reference sample.}
Prompt-aware results: We compared computing MMD guidance with clean reference samples and with the latent vector of the reference sample at time step $t$. Suggested by the results in Table~\ref{tab:sdxl_mmd_noisy_clean}, we observe that using the latent vector of the noisy reference sample at time-step $t$, achieves comparable results to using clean references.

\begin{table}[!t]
\caption{Comparison of SDXL variants with and without MMD guidance.}
\label{tab:sdxl_mmd_noisy_clean}
\centering
\begin{tabular}{lccc}
\toprule
Model & FD $\downarrow$ & KD $\downarrow$ & RRKE $\downarrow$ \\
\midrule
SDXL (No-Guidance-DM)                              & 1734.25 & 1.5536 & 1.84 \\
SDXL + MMD Guidance (noisy reference at timestep $t$) & 1725.56 & 1.5487 & 1.79 \\
SDXL + MMD Guidance (clean reference)           & 1678.24 & 1.4429 & 1.76 \\
\bottomrule
\end{tabular}
\end{table}

\textbf{MMD Guidance for Adapting Highly Different Domains.}
To test MMD guidance in highly different domains, /following \cite{somepalli2024measuring}, we use two distinct painting styles, Klimt and Franz Marc, as the reference set and apply MMD guidance to adapt the distribution of samples generated by SDXL. To quantitatively measure the domain gap between SDXL outputs and the reference set, we compute Recall and Coverage, obtaining values of 0.34 and 0.12, respectively. Table~\ref{tab:highly_different_domains} reports the quantitative results with and without MMD guidance.

\begin{table}[t]
\centering
\caption{Quantitative comparison of SDXL with and without MMD guidance for adapting highly different artistic domains.}
\label{tab:highly_different_domains}
\begin{tabular}{lccccccc}
\hline
Model & FD $\downarrow$ & KD $\downarrow$ & RRKE $\downarrow$ & Density ($\times 10^2$) $\uparrow$ & Coverage ($\times 10^2$) $\uparrow$ & Precision $\uparrow$ & Recall $\uparrow$ \\
\hline
SDXL (No Guidance) & 1988.4 & 4.38 & 1.98 & 9.83 & 12.76 & 25.93 & 18.39 \\
SDXL + MMD (Ours) & 1564.3 & 2.28 & 1.74 & 24.34 & 52.74 & 74.21 & 34.84 \\
\hline
\end{tabular}
\end{table}

\textbf{Effect of Reference Mixture Weights on Mode Proportions.}
To address whether guidance can adjust mixture proportions, an additional experiment was conducted on the synthetic mixture-of-Gaussians setup. A diffusion model was first trained on data where all Gaussian components had equal mixture weights, yielding a uniform distribution over modes. As a baseline, samples generated without any guidance (second panel) reproduce this uniform mixture and do not reflect any target reweighting.

\begin{table}[t]
\caption{Effect of mode proportions in MMD guidance on the number of samples in each GMM component.}
\label{tab:mmd-weights}
\centering
\begin{tabular}{lcccccccc}
\toprule
Guidance & 1 & 2 & 3 & 4 & 5 & 6 & 7 & 8 \\
\midrule
Reference sample & 90\% & 10\% & 0\% & 0\% & 0\% & 0\% & 0\% & 0\% \\
No-Guidance-DM         & 12\% & 12\% & 11\% & 18\% & 12\% & 10\% & 13\% & 12\% \\
MMD              & 88\% & 12\% & 0\% & 0\% & 0\% & 0\% & 0\% & 0\% \\
\bottomrule
\end{tabular}
\end{table}

To simulate a non-uniform target distribution, reference samples were drawn from a mixture whose component weights were sampled from a Dirichlet distribution with parameter (1,10) over the depicted components, thereby concentrating mass on a subset of modes. As depicted in Figure~\ref{fig:gmm_weight} after applying MMD-Guidance at sampling time, the guided diffusion model successfully shifted its outputs to match these reweighted proportions: the generated samples predominantly occupy the emphasized component, and their relative frequencies closely track those in the reference set. This demonstrates that the proposed guidance mechanism can meaningfully steer mode proportions using only a small, reweighted reference sample set. Table~\ref{tab:mmd-weights} provides numerical proofs for this observation. After generation, each sample is assigned to the closest GMM component using a \(k\)-nearest-neighbor classifier.



\textbf{Additional qualitative results for Prompt-based diffusion models.}
We conducted experiments on Pixar animated cars and Cowboy animated horses using Pixart-$\Sigma$ diffusion models, similar to the settings shown in Figure~\ref{fig:figure_one} as illustrated in Figure~\ref{fig:pixart_mmd}.
This comparison evaluates image generation via a text-conditioned latent diffusion model (LDM) with and without guidance. The LDM (Pixart-$\Sigma$) employing our proposed MMD guidance, based on 100 reference samples of "car" and "horse" images, effectively captures the visual format of the target distribution. In contrast, the unguided LDM outputs exhibit stylistic differences from the target model.

\textbf{Visualization of the generation process and the effect of late-step MMD guidance.}
We further study the generation trajectory under MMD guidance by visualizing intermediate denoising results. In particular, we examine whether the desired samples can still be obtained when MMD guidance is applied only during a restricted time interval of the reverse diffusion process. The results show that applying MMD guidance only in the final 30 denoising steps is already sufficient to steer the samples toward the target distribution. Compared with applying guidance throughout the entire reverse process, this restricted setting produces visually similar generation trajectories while incurring lower computational cost. These findings indicate that the later denoising stages are especially important for distribution alignment.

To further investigate the effect of restricting guidance to the late denoising phase, we apply MMD guidance to SDXL only during the last 30 iterations. As shown in Table~\ref{tab:sdxl_last30}, this setting yields results that remain close to full-step MMD guidance and significantly outperform the unguided baseline across all metrics. These results suggest that the final denoising stages play a particularly important role in aligning the generated samples with the target distribution.

\textbf{Effect of the intra-batch Diversity Term.}
To assess the effect of the intra-batch diversity term, we conducted an ablation study where we first evaluated on two settings of 8-center and 25-center GMMs. For each setting, we sample 3 random user distributions with 150 points each and run the method with and without the diversity term. Table~\ref{tab:diversity_gmm} reports the quantitative results, and Figure~\ref{fig:diversity-gmm} shows a representative failure case without the diversity term.

\begin{figure}
\resizebox{\textwidth}{!}{%
\begin{tikzpicture}
\def\imgwidth{14.3cm}
\def\imgwidthbig{14.0cm}
\node[inner sep=0, outer sep=0] (rowTop) at (9.34, -1.5)
  {\includegraphics[width=\imgwidth]{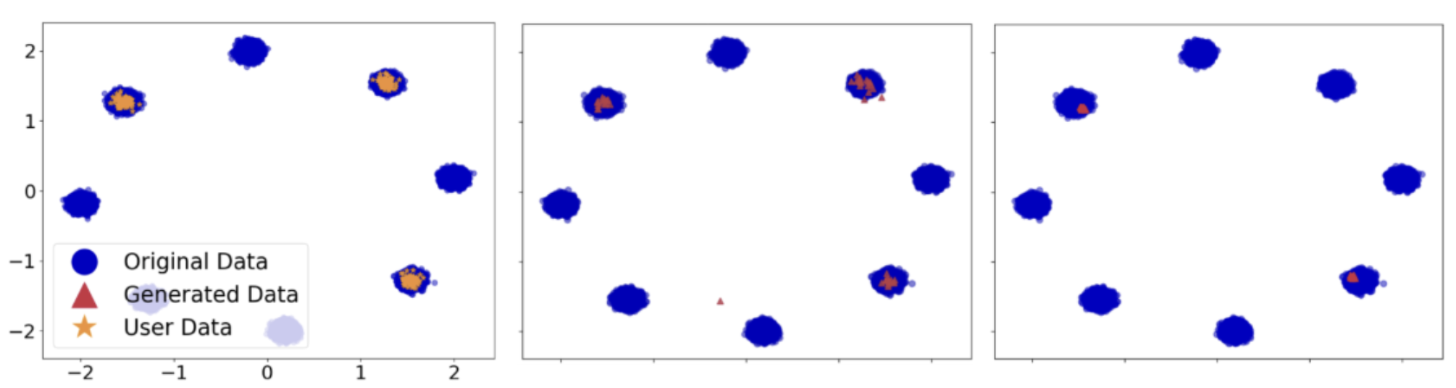}};

\def\colA{0.17}\def\colB{0.50}\def\colC{0.83}

\foreach \x/\dens in {
  \colA/{\headerFont\textbf{User Data}},
  \colB/{\headerFont\textbf{Diversity Term}},
  \colC/{\headerFont\textbf{No Diversity Term}}
}{
  \node[
    font=\normalsize, align=center,
    anchor=south,
    fill=white, rounded corners=0.6pt, inner sep=0.9pt
  ] at ($ (rowTop.north west)!\x!(rowTop.north east) + (0,0.18) $)
  { \dens};
}

\end{tikzpicture}
}
\caption{Effect of the intra-batch Diversity Term in 8-centered GMM.}
\label{fig:diversity-gmm}
\end{figure}

\begin{table*}[t]
\caption{Effect of the intra-batch diversity term on synthetic GMM experiments.}
\label{tab:diversity_gmm}
\centering
\setlength{\tabcolsep}{6pt}
\resizebox{\linewidth}{!}{
\begin{tabular}{cccccccc}
\toprule
& & \multicolumn{3}{c}{8 Gaussian Component} & \multicolumn{3}{c}{25 Gaussian Component} \\
\cmidrule(lr){3-5}\cmidrule(lr){6-8}
User & Guidance &
FD $\downarrow$ & KD($\times10^{2}$)$\downarrow$ & RRKE $\downarrow$ &
FD $\downarrow$ & KD($\times10^{2}$)$\downarrow$ & RRKE $\downarrow$ \\
\midrule

\multirow{2}{*}{User1}
& Diversity Term    & 0.37 & 1.70 & 1.294 & 3.56 & 0.95 & 0.231 \\
& No Diversity Term & 4.14 & 12.90 & 2.00 & 11.87 & 13.49 & 0.131 \\
\midrule

\multirow{2}{*}{User2}
& Diversity Term    & 0.33 & 1.65 & 1.245 & 2.84 & 0.90 & 0.246 \\
& No Diversity Term & 0.55 & 4.34 & 1.36 & 31.20 & 21.92 & 0.019 \\
\midrule

\multirow{2}{*}{User3}
& Diversity Term    & 0.64 & 1.45 & 1.258 & 1.85 & 0.77 & 0.236 \\
& No Diversity Term & 3.70 & 12.75 & 1.42 & 6.65 & 6.23 & 0.29 \\
\bottomrule

\end{tabular}
}
\end{table*}

Using the FFHQ dataset, we observed that in Table~\ref{tab:diversity_ffhq}, similar to GMMs, incorporating the diversity term and combining it with the cross-term leads to noticeable improvements across the metrics. 

\begin{table*}[t]
\caption{
Effect of the intra-batch diversity term on FFHQ dataset.
}
\label{tab:diversity_ffhq}
\centering
\setlength{\tabcolsep}{6pt}
\resizebox{\linewidth}{!}{
\begin{tabular}{l l c c c c c}
\toprule
User &
Guidance &
FD$\downarrow$ &
KD$\downarrow$ &
RRKE$\downarrow$ &
Density (x100)$\uparrow$ &
Coverage (x100)$\uparrow$ \\
\midrule
Sunglasses      & Cross Term Only   & 941.49 & 5.63 & 1.51 & 136.00 & 73.24 \\
Sunglasses      & Diversity + Cross & 692.87 & 3.25 & 1.39 & 113.13 & 79.08 \\
\midrule
Reading-glasses & Cross Term Only   & 954.83 & 6.64 & 1.42 & 76.30 & 60.32 \\
Reading-glasses & Diversity + Cross & 574.29 & 1.39 & 1.30 & 87.10 & 84.60 \\
\bottomrule
\end{tabular}
}
\end{table*}

\begin{table*}[t!]
\caption{Comparison between fine-tuning and MMD guidance on synthetic GMM under large number of data reference sample.}
\label{tab:gmm_ft_vs_mmd}
\centering
\setlength{\tabcolsep}{6pt}
\resizebox{\linewidth}{!}{
\begin{tabular}{cccccccc}
\toprule
& & \multicolumn{3}{c}{8 Gaussian Component} & \multicolumn{3}{c}{25 Gaussian Component} \\
\cmidrule(lr){3-5}\cmidrule(lr){6-8}
User & Guidance &
FD $\downarrow$ & KD(100)$\downarrow$ & RRKE $\downarrow$ &
FD $\downarrow$ & KD(100)$\downarrow$ & RRKE $\downarrow$ \\
\midrule

\multirow{3}{*}{User1}
& No Guidance & 7.042 & 14.242 & 2.314 & 77.702 & 33.091 & 2.073 \\
& Fine Tuning & 0.136 & 0.172 & 0.718 & 0.266 & 0.212 & 0.270 \\
& MMD         & 0.144 & 0.201 & 0.695 & 0.271 & 0.203 & 0.207 \\
\midrule

\multirow{3}{*}{User2}
& No Guidance & 7.464 & 15.398 & 2.089 & 95.583 & 41.704 & 2.256 \\
& Fine Tuning & 0.065 & 0.159 & 1.446 & 0.620 & 0.291 & 0.251 \\
& MMD         & 0.063 & 0.168 & 0.933 & 0.558 & 0.196 & 0.220 \\
\midrule

\multirow{3}{*}{User3}
& No Guidance & 8.303 & 16.747 & 2.376 & 87.633 & 37.838 & 2.107 \\
& Fine Tuning & 0.256 & 0.201 & 0.290 & 0.732 & 0.246 & 0.267 \\
& MMD         & 0.182 & 0.148 & 0.220 & 0.699 & 0.220 & 0.260 \\
\bottomrule

\end{tabular}
}
\end{table*}

\textbf{Comparison with Fine-Tuning under Larger Reference Sets.}
We further compare MMD Guidance with fine-tuning in a setting with larger numbers of reference samples per user. We evaluate on Gaussian mixture models in two settings of 8-center and 25-center GMMs. For each setting, we sample three random user distributions containing 3000 data points. As shown in Table~\ref{tab:gmm_ft_vs_mmd}, the performance of fine-tuning improves significantly as more reference samples become available. Nevertheless, MMD Guidance continues to achieve performance comparable to fine-tuning while avoiding additional training, demonstrating its effectiveness even in training-sufficient regimes.

\begin{figure*}[t]
\centering
\begin{tikzpicture}

\node[anchor=south west, inner sep=0] (fig)
  at (0,0) {\includegraphics[width=\textwidth]{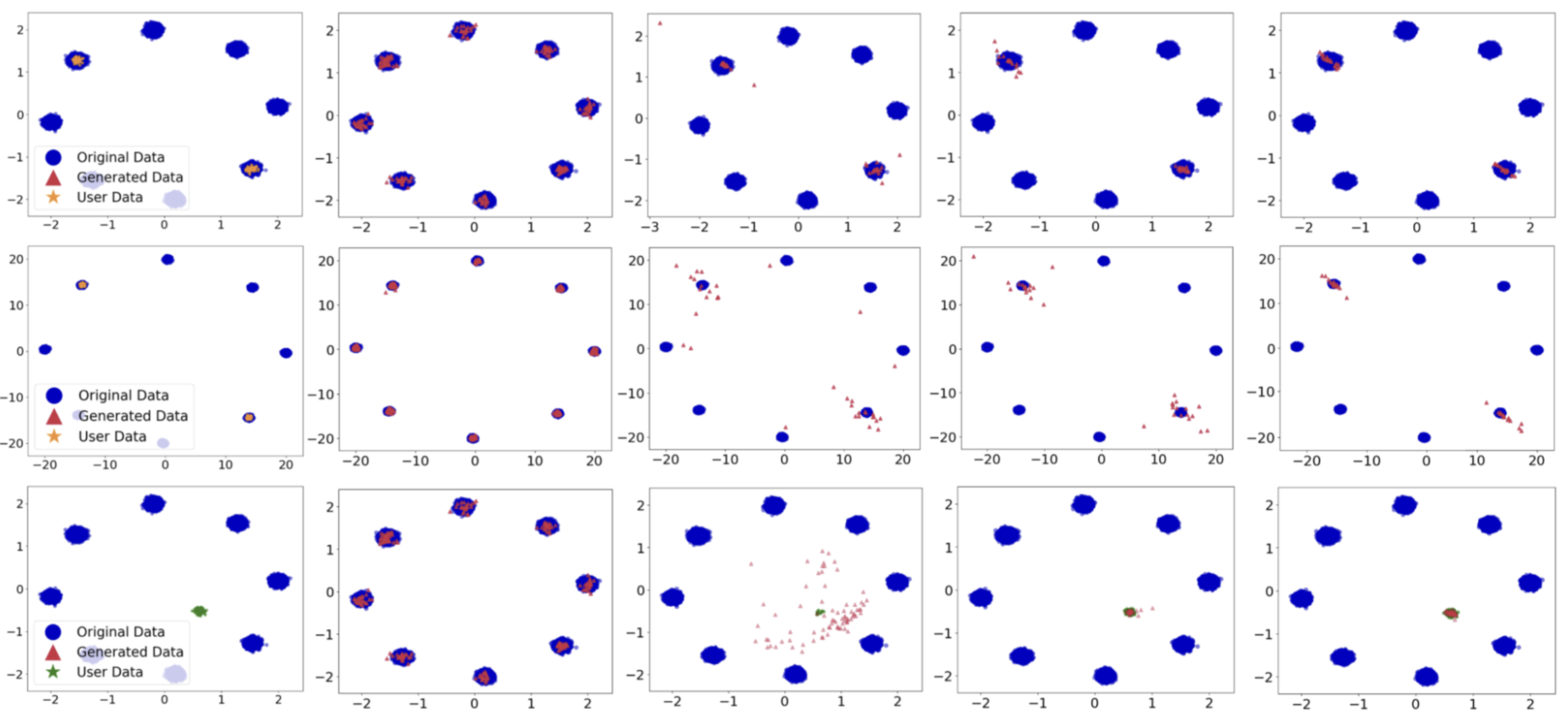}};

\node[anchor=south, yshift=2pt]
at ($(fig.north west)!0.10!(fig.north east)$)
{\textbf{User}};

\node[anchor=south, yshift=2pt]
at ($(fig.north west)!0.30!(fig.north east)$)
{\textbf{\#Samples = 0}};

\node[anchor=south, yshift=2pt]
at ($(fig.north west)!0.50!(fig.north east)$)
{\textbf{\#Samples = 10}};

\node[anchor=south, yshift=2pt]
at ($(fig.north west)!0.70!(fig.north east)$)
{\textbf{\#Samples = 100}};

\node[anchor=south, yshift=2pt]
at ($(fig.north west)!0.90!(fig.north east)$)
{\textbf{\#Samples = 200}};

\end{tikzpicture}

\caption{Effect of the number of reference samples on domain gap for 8-centered GMM using MMD guidance. Effect of domain gap on guided sampling in Gaussian mixture models. Larger gaps require stronger guidance and more reference samples. The same trend holds when the target lies outside the training}
\label{fig:domain_gap}
\end{figure*}

\begin{figure}[t]
    \centering
    \vspace{-3mm}
    \includegraphics[width=0.94\linewidth]{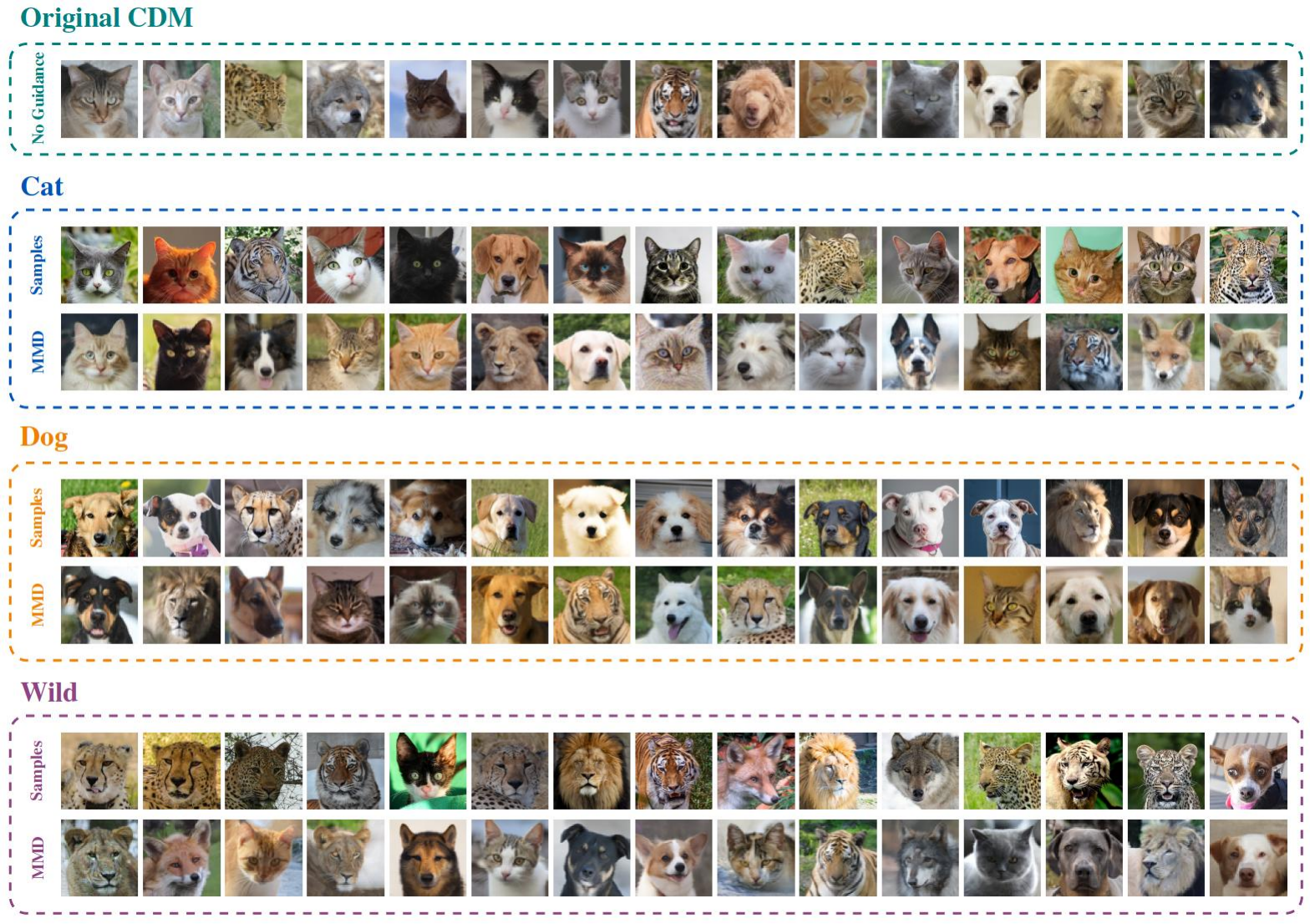}
    \caption{MMD Guidance on AFHQ dataset.}
    \label{afhq-qual}
  
\end{figure}

\begin{figure}
    \centering
    \includegraphics[width=0.72\linewidth]{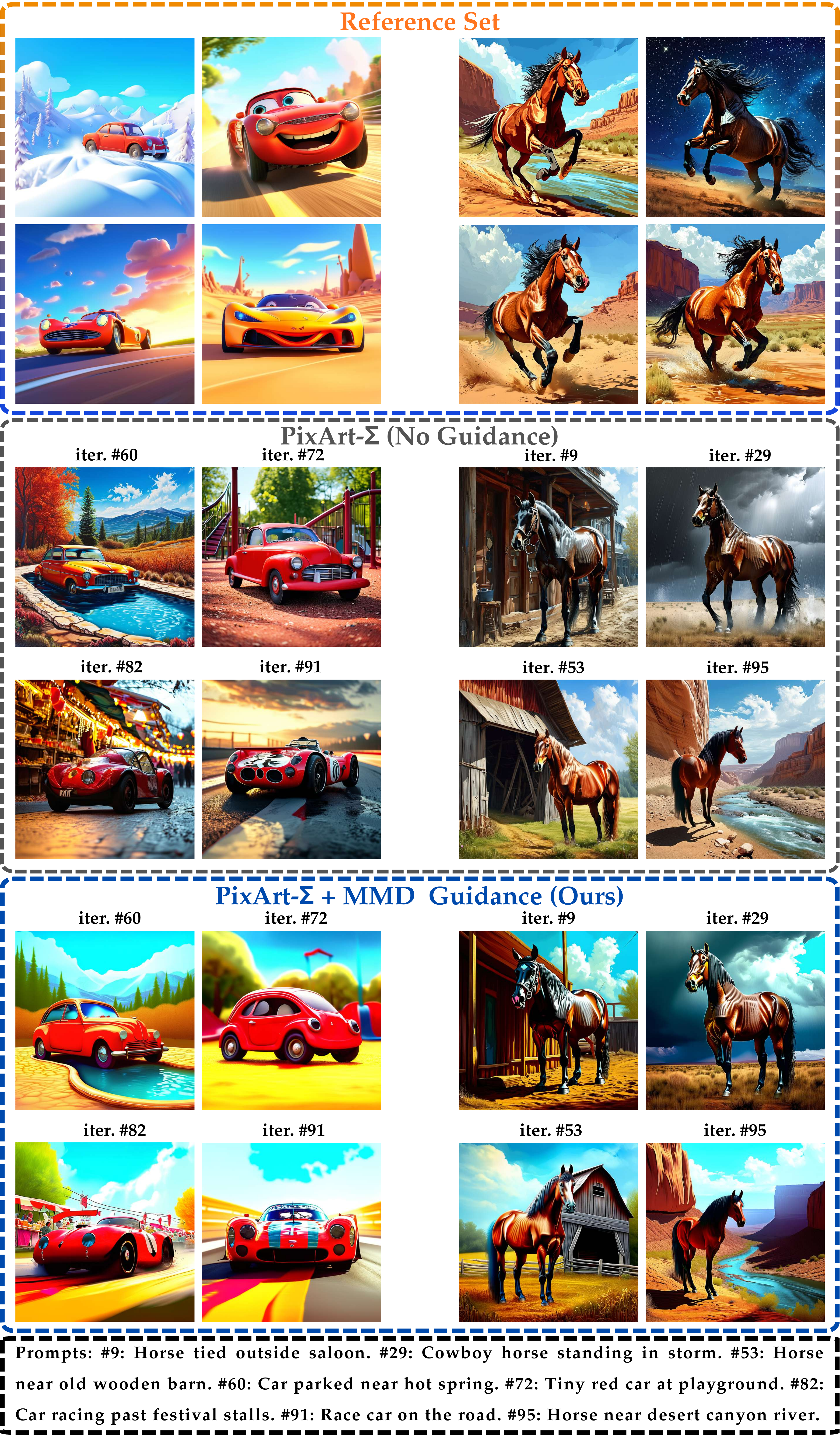}
    \caption{Comparison of Pixart-$\Sigma$ image generation with and without MMD guidance, showing style differences in unguided LDM outputs from the target distribution of "car" and "horse" images.}
    \label{fig:pixart_mmd}
\end{figure}

\begin{figure}
    \centering
    \includegraphics[width=\linewidth]{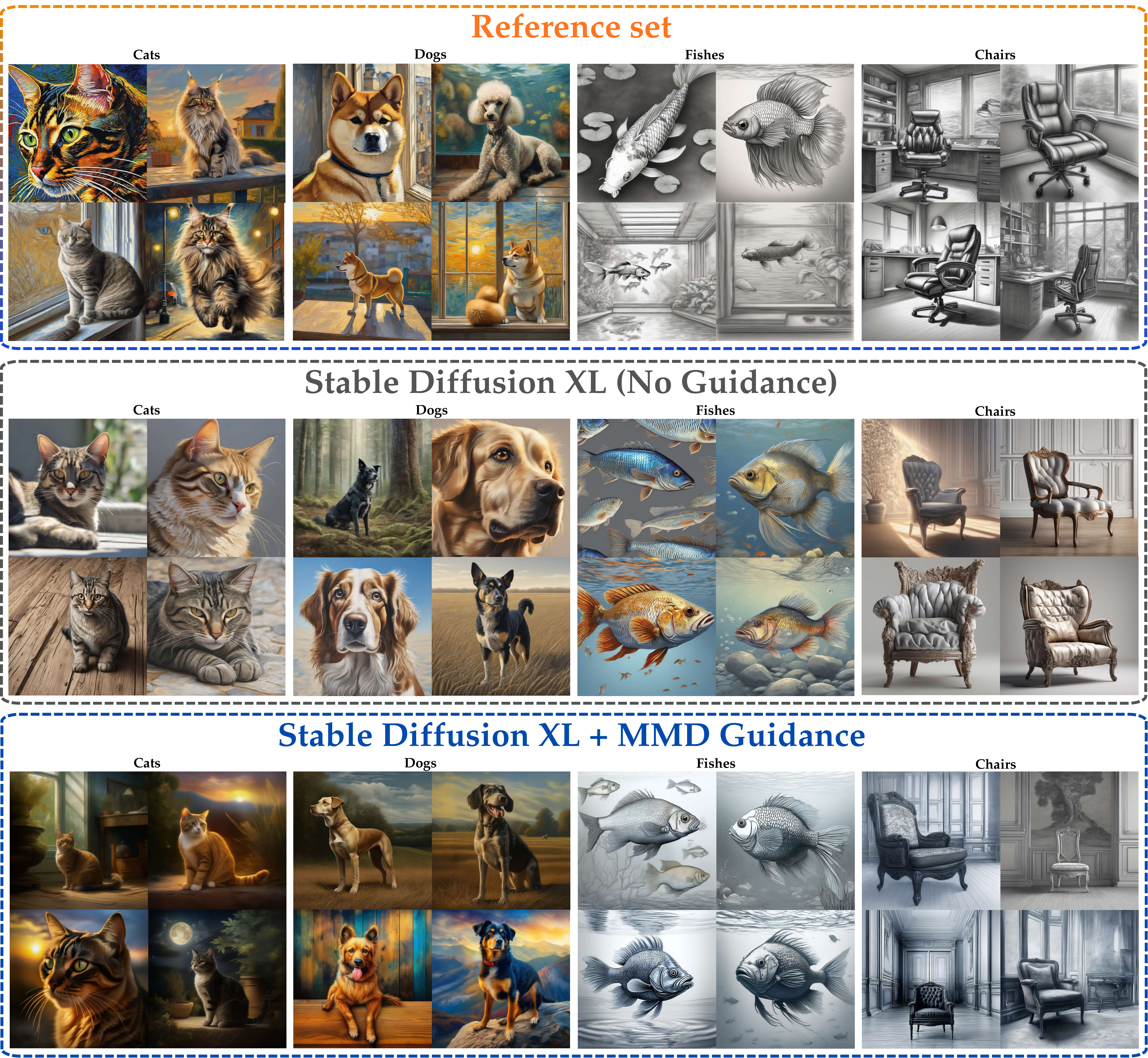}
    \caption{Qualitative comparison of reference set and MMD-guided image generation with SDXL.}
    \label{fig:sdxl_user_1}
\end{figure}

\begin{figure}
    \centering
    \includegraphics[width=\linewidth]{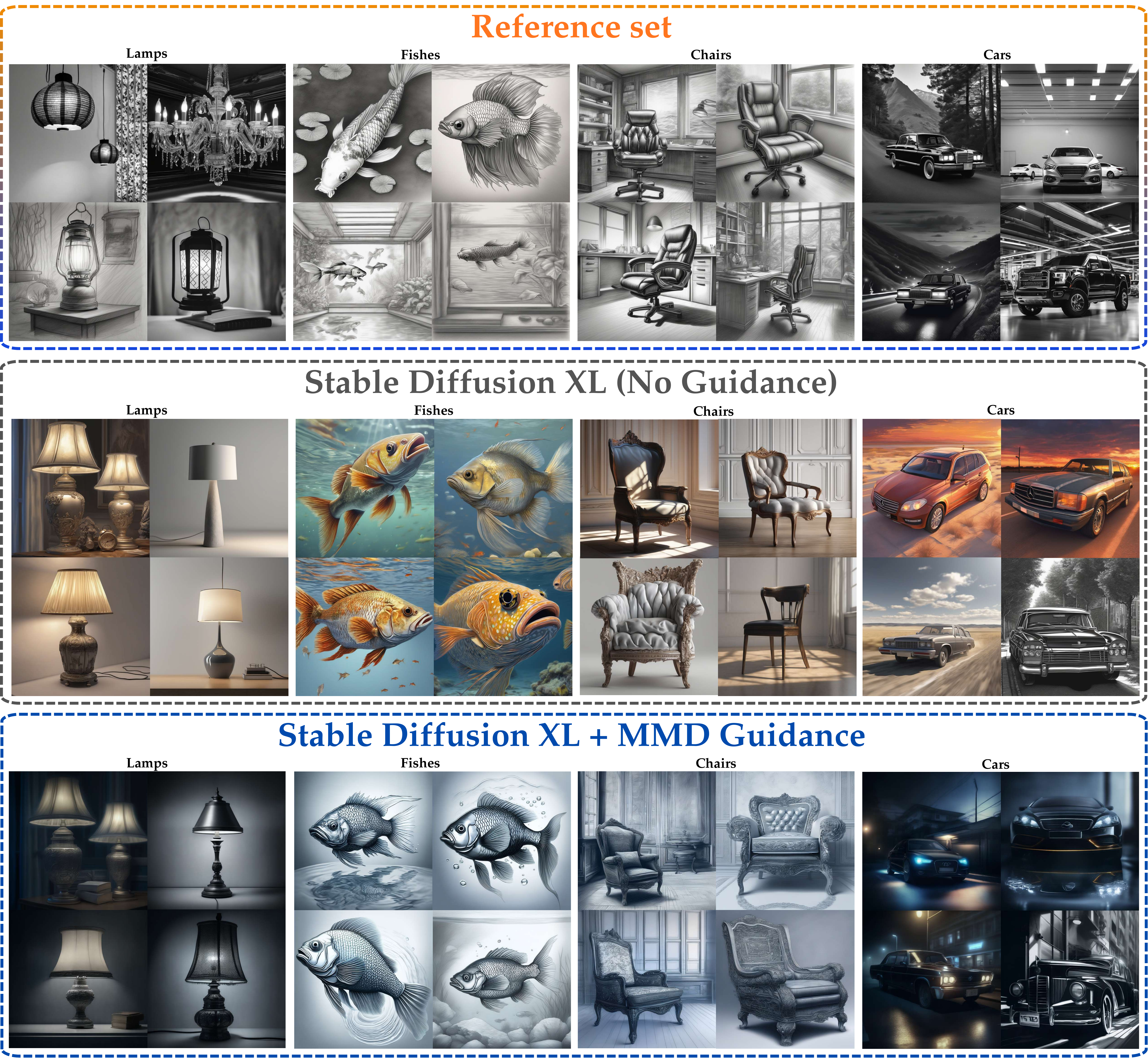}
    \caption{Qualitative comparison of reference set and MMD-guided image generation with SDXL.}
    \label{fig:sdxl_user_2}
\end{figure}

\begin{table}[!t]
    \caption{Evaluation metrics for AFHQ with No-Guidance-DM vs. MMD Guidance.}
    \label{afhq-prdc}
    \centering
    \resizebox{\linewidth}{!}{
    \begin{tabular}{ccccccc}
    \toprule
    User & Guidance & FD $\downarrow$ & KD $\downarrow$ & RRKE $\downarrow$ & Density ($\times 10^{2}$) $\uparrow$
    & Coverage ($\times 10^{2}$) $\uparrow$ \\
    \midrule
    \multirow{2}{*}{Cat} & No-Guidance-DM & $1154.87 \pm 231.71$ & $3.16 \pm 1.17$ & $1.37 \pm 0.10$ & $44.70 \pm 20.08$ & $38.16 \pm 2.99$ \\ 
    & MMD (Ours) & $1030.14 \pm 208.31$ & $2.85 \pm 1.02$ & $1.33 \pm 0.08$ & $42.81 \pm 18.25$ & $43.08 \pm 2.37$  \\
    \midrule
    \multirow{2}{*}{Dog} & No-Guidance-DM & $1309.05 \pm 271.71$ & $2.55 \pm 0.56$ & $1.40 \pm 0.08$ & $42.21 \pm 10.34$ & $60.96 \pm 7.94$ \\ 
    & MMD (Ours) & $1147.64 \pm 203.07$ & $1.91 \pm 0.36$ & $1.36 \pm 0.06$ & $42.50 \pm 7.47$ & $63.60 \pm 5.38$ \\
    \midrule
    \multirow{2}{*}{Wild} & No-Guidance-DM & $1301.12 \pm 176.96$ & $4.69 \pm 1.26$ & $1.39 \pm 0.06$ & $41.87 \pm 16.39$ & $21.64 \pm 6.46$ \\ 
    & MMD (Ours) & $1052.44 \pm 115.79$ & $3.45 \pm 0.90$ & $1.32 \pm 0.04$ & $35.47 \pm 14.14$ & $24.96 \pm 6.21$ \\
    \bottomrule
    \end{tabular}
    }
\end{table}

\begin{figure*}
    \centering
    \includegraphics[width=\linewidth]{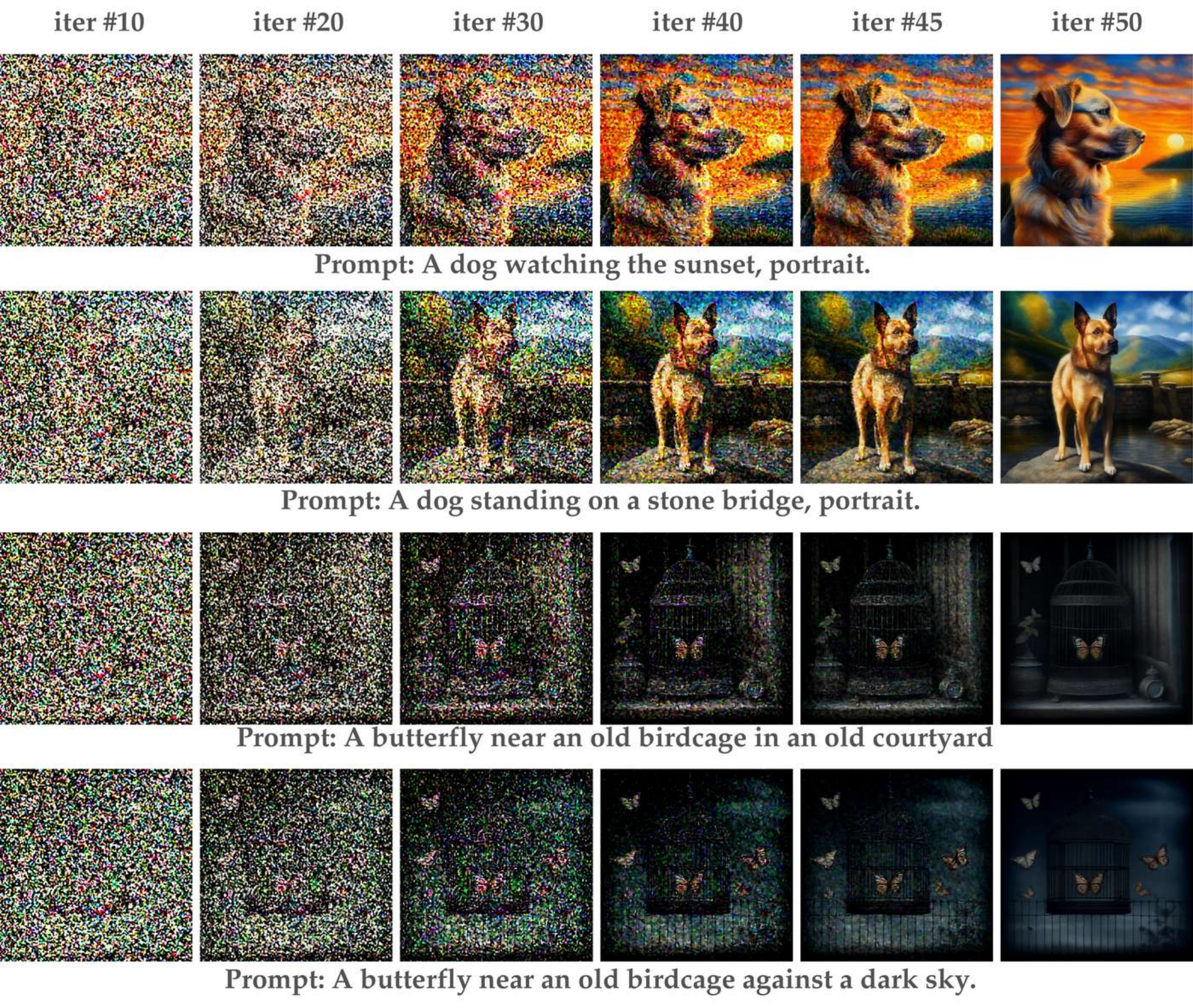}
    \caption{Visualization of the SDXL generation process under MMD guidance.}
    \label{fig:visualization_sdxl}
\end{figure*}

\begin{figure*}
    \centering
    \includegraphics[width=\linewidth]{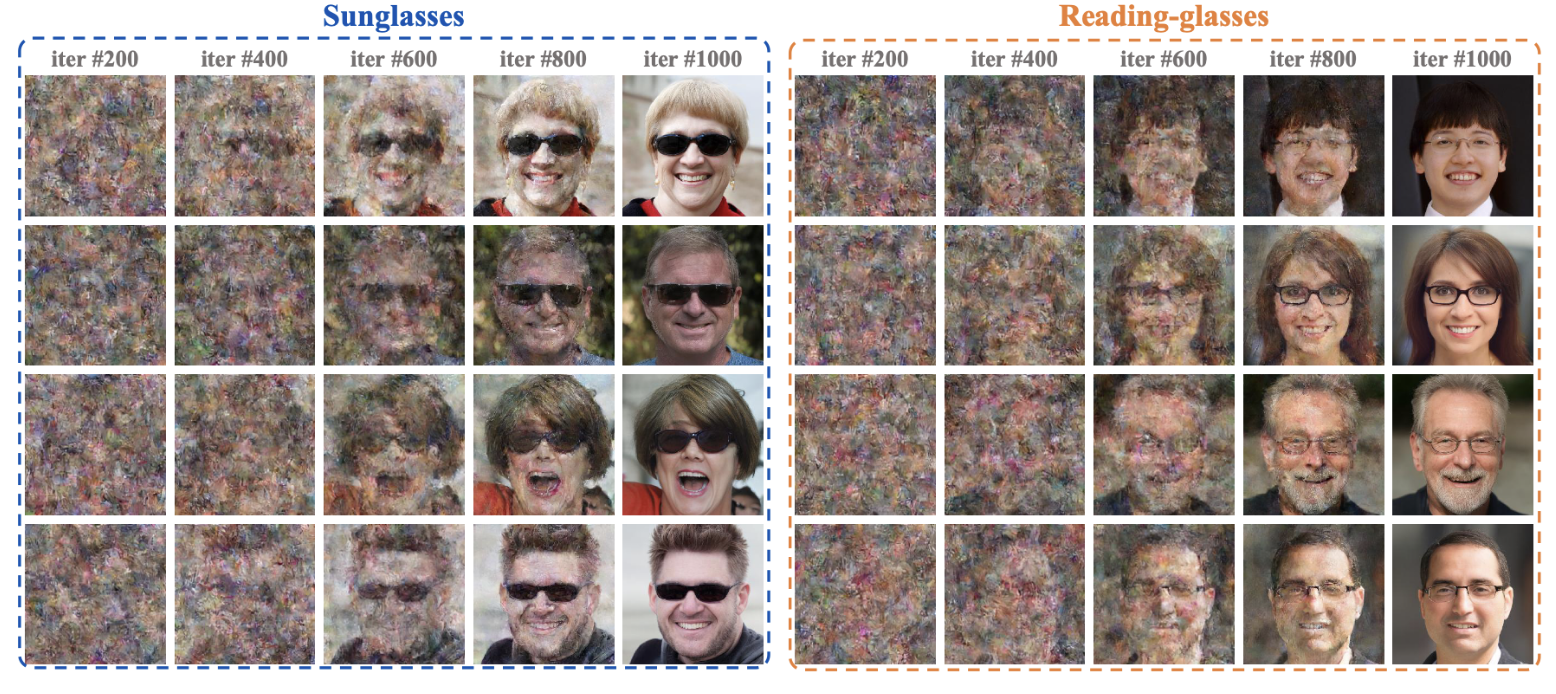}
    \caption{Visualization of the LDM generation process for the FFHQ dataset under MMD guidance.}
    \label{fig:visualization_ffhq}
\end{figure*}

\begin{table}[t]
\centering
\caption{Quantitative comparison on SDXL with full-step guidance and late-step guidance.}
\label{tab:sdxl_last30}
\begin{tabular}{lccccc}
\hline
Model & Guidance & FD $\downarrow$ & KD $\downarrow$ & Density ($\times 10^2$) $\uparrow$ & Coverage ($\times 10^2$) $\uparrow$ \\
\hline
SDXL & No-Guidance-DM & 1953.8 & 3.57 & 5.63 & 11.34 \\
SDXL & MMD (Ours) & 1674.4 & 2.49 & 18.01 & 34.20 \\
SDXL (last 30 steps) & MMD (Ours) & 1732.23 & 2.53 & 16.34 & 31.92 \\
\hline
\end{tabular}
\end{table}

\end{document}